\documentclass[opre]{informs3}



\usepackage{amsmath,amssymb,mathrsfs,enumerate,bm,xcolor,multirow,pbox}
\usepackage{graphicx,color,framed,tikz,caption,subcaption}
\usepackage{enumitem}
\usepackage{algorithm}
\usepackage[noend]{algpseudocode}
\usepackage{natbib}
\usepackage[colorlinks=true,linkcolor=blue,citecolor=blue,urlcolor=blue]{hyperref}
\usepackage{booktabs}
\usepackage{tabularray}
\usetikzlibrary{arrows.meta,decorations.pathreplacing,positioning,shapes.geometric,calc,fit,backgrounds}
\usepackage{makecell}
\usepackage[table]{xcolor}
\usepackage{pgfplots}
\usepgfplotslibrary{fillbetween}

\definecolor{SparseColor}{RGB}{245,249,255}   
\definecolor{LinearColor}{RGB}{248,245,255}   
\definecolor{HolderColor}{RGB}{245,252,245}   
\definecolor{IIDColor}{RGB}{238,234,248}

\newcolumntype{L}[1]{>{\raggedright\arraybackslash}m{#1}}
\newcolumntype{C}[1]{>{\centering\arraybackslash}m{#1}}
\setlist{leftmargin=9mm}

\bibpunct[, ]{(}{)}{,}{a}{}{,}%
\def\bibfont{\small}%

\TheoremsNumberedThrough
\ECRepeatTheorems
\EquationsNumberedThrough
\MANUSCRIPTNO{}

\allowdisplaybreaks[4]

\newcommand{\R}{\mathbb{R}}

\newcommand{\pnorm}[2]{\lVert #1\rVert_{#2}}

\newcommand{\abs}[1]{\lvert#1\rvert}

\newcommand{\biggabs}[1]{\bigg\lvert#1\bigg\rvert}

\renewcommand{\epsilon}{\varepsilon}

\renewcommand{\d}[1]{\mathrm{d}#1}
\newcommand{\dist}[2]{\mathrm{dist}\left(#1,#2\right)}
\newcommand{\floor}[1]{\left\lfloor #1 \right\rfloor}

\newcommand{\bco}{\mathsf{bco}}
\renewcommand{\tilde}{\widetilde}

\DeclareMathOperator{\E}{\mathbb{E}}
\DeclareMathOperator{\Prob}{\mathbb{P}}

\DeclareMathOperator{\proj}{\mathsf{P}}

\DeclareMathOperator{\Ber}{Ber}
\DeclareMathOperator{\KL}{KL}
\DeclareMathOperator{\kl}{kl}

\renewcommand{\mc}{\mathsf{c}}

\newcommand{\Regret}{\mathrm{Regret}}

\newcommand{\cO}{\mathcal{O}}

\newcommand{\cA}{\mathcal{A}}
\newcommand{\cB}{\mathcal{B}}
\newcommand{\cG}{\mathcal{G}}
\newcommand{\cU}{\mathcal{U}}
\newcommand{\cF}{\mathcal{F}}
\newcommand{\cH}{\mathcal{H}}

\newcommand{\cE}{\mathcal{E}}
\newcommand{\cD}{\mathcal{D}}
\newcommand{\cT}{\mathcal{T}}
\newcommand{\cQ}{\mathcal{Q}}
\newcommand{\cR}{\mathcal{R}}
\newcommand{\cC}{\mathcal{C}}

\newcommand{\bA}{\bm{A}}
\newcommand{\bI}{\bm{I}}

\newcommand{\frakq}{\mathfrak{q}}
\newcommand{\refine}{\mathsf{refine}}

\newcommand{\coarse}{\mathsf{coarse}}
\newcommand{\ORBIT}{\mathsf{ORBIT}}
\newcommand{\cV}{\mathcal{V}}

\AtBeginDocument{%
	\def\MR#1{}
}
\renewcommand*{\theassumption}{\Alph{assumption}}
\renewcommand{\qquad}{\quad}
\renewcommand{\subseteq}{\subset}

\RUNTITLE{Oracle Price Map Learning}
\TITLE{
Harnessing Unimodality in Semiparametric Contextual Pricing via Oracle Price Map Learning\thanks{The current version is an extension of an earlier draft available at \url{https://papers.ssrn.com/sol3/papers.cfm?abstract_id=6572802}, with
the same framework extended to general context distributions and broader utility
classes.}}

\RUNAUTHOR{Fan, Han, Lv, Xu, and Zhou}

\ARTICLEAUTHORS{
\AUTHOR{Yingying Fan}
\AFF{Data Sciences and Operations Department, University of Southern California, Los Angeles, California 90089, USA
\EMAIL{fanyingy@marshall.usc.edu}}
\AUTHOR{Yuxuan Han}
\AFF{Stern School of Business, New York University, New York, NY 10003, USA\\
\EMAIL{yh6061@stern.nyu.edu}}
\AUTHOR{Jinchi Lv}
\AFF{Data Sciences and Operations Department, University of Southern California, Los Angeles, California 90089, USA
\EMAIL{jinchilv@marshall.usc.edu}}
\AUTHOR{Xiaocong Xu}
\AFF{Data Sciences and Operations Department, University of Southern California, Los Angeles, California 90089, USA
\EMAIL{xuxiaoco@marshall.usc.edu}}
\AUTHOR{Zhengyuan Zhou}
\AFF{Stern School of Business, New York University, New York, NY 10003, USA\\
\EMAIL{zz26@stern.nyu.edu}}
}

\ABSTRACT{We study contextual dynamic pricing in a semiparametric scalar-index valuation model where the latent value is $v_t=\mu_\ast(\mc_t)+\xi_t$, with an unknown utility map $\mu_\ast$ and an unknown additive noise distribution. The key decision object is the one-dimensional oracle price map $u\mapsto p^\ast(u)$ induced by the scalar index $u=\mu_\ast(\mc)$ and the noise tail. Under the $\beta$-H\"older smoothness of the tail function for $\beta\geq 2$ and a revenue-geometry condition that gives a unique, stable, interior maximizer, this oracle map is itself $(\beta-1)$-smooth. We exploit such structure through $\ORBIT$, a modular coarse-to-fine policy that takes a scalar pilot index as input, localizes a benchmark price in each active bin, and learns a local polynomial approximation of the oracle map inside a trust region via bandit convex optimization. For the baseline linear utility model $\mu_\ast(\mc)=\mc^\top\theta_\ast$, an adaptive elliptical exploration scheme constructs the required scalar pilot online without distributional assumptions on the contexts. The resulting policy achieves regret $\widetilde{\cO}\big(T^{\frac{2\beta-1}{4\beta-3}}+\sqrt{dT}\big)$. For fixed $d$, we establish a matching lower bound in the horizon dependence, unveiling that the nonparametric oracle-map learning term is minimax sharp. The same scalar-pilot interface also yields extensions to sparse high-dimensional linear utility and nonparametric H\"older utility.}


\begin{document}
\maketitle


\sloppy

\section{Introduction}\label{sec:intro}

Dynamic pricing with demand learning is a central problem in revenue management. In many digital and service markets, the seller observes customer- or product-side covariates before quoting a price, but the demand system is revealed only through realized purchases. Prices thus shape both current revenue and the information available for future decisions, creating the classical exploration-exploitation tradeoff \citep{kleinberg2003value,denboer2015dynamic,broder2012dynamic,lobel2020algorithmic,chenhu2023datadriven}.

In this paper, we study contextual dynamic pricing with binary feedback. In period $t$, the arriving customer has latent valuation
\begin{align}\label{intro-eq: valuation}
    v_t = u_t + \xi_t, \qquad u_t = \mu_\ast(\mc_t),
\end{align}
where $\mc_t \in \mathbb{R}^d$ is an observed covariate vector, $\mu_\ast(\cdot):\mathbb{R}^d \mapsto \mathbb{R}$ is an \textit{unknown} utility function, $u_t$ is the realized utility, and $\xi_t$ is an unobserved noise term with fixed but \textit{unknown} distribution. After observing $\mc_t$, the seller posts a price $p_t$ and observes only the binary purchase indicator
\begin{align*}
    y_t = \mathbf{1}\{v_t \geq p_t\}.
\end{align*}
Such specification allows flexible modeling of contextual heterogeneity through $\mu_\ast(\cdot)$ while leaving the noise law fully nonparametric. It thus lies between known-link contextual demand models, where the demand link is specified in advance, and fully nonparametric contextual demand models, where learning must proceed over the full covariate space \citep{javanmard2019dynamic,cohen2020feature,ban2021personalized,chen2021nonparametric}. When $\mu_\ast(\cdot)$ is specialized to the linear utility, this framework recovers the \textit{semiparametric} demand model in \citet{luo2022contextual,fan2024policy,tullii2024improved,wang2025tight,han2026general}; when $\mu_\ast(\cdot)$ is specialized to the nonparametric utility, the framework recovers the \textit{doubly nonparametric} demand model in \citet{chen2024dynamic}.

A central object for policy optimization is the utility-associated revenue function
\begin{align}\label{intro-eq: revenue}
    r(u,p) := p\, g(p-u), \quad u,\, p \; \in \mathbb{R},
\end{align}
where $g$ denotes the unknown tail function of $\xi_t$. The corresponding utility-based oracle price is
\begin{align*}
    p^\ast(u) \in \argmax_{p \geq 0} r(u,p).
\end{align*}
If both $\mu_\ast(\cdot)$ and $p^\ast(\cdot)$ were known, the seller would simply post $p^\ast(\mu_\ast(\mc_t))$ in each period. This observation suggests that the optimal pricing policy is governed primarily by the oracle price rule $u \mapsto p^\ast(u)$. The central idea of our paper is to make such observation operational: instead of recovering the full demand system, we show that under the structural conditions studied here, it suffices to learn the oracle price map itself. This reduction is, however, not automatic. Since $p^\ast(u)$ is defined only \textit{implicitly} through an optimization problem, it does not need to be unique, stable, or sufficiently regular to support direct estimation without additional structure.

In this work, we show that under the standard $\beta$-H\"older smoothness and unimodality conditions commonly imposed in contextual pricing models \citep{javanmard2019dynamic,chen2021nonparametric,fan2024policy,chen2024dynamic,wang2025tight,han2026general}, the oracle price map $p^\ast(\cdot)$ enjoys favorable regularity. Strong unimodality does more than simplify a one-dimensional optimization: it guarantees a unique interior maximizer, converts local price error into quadratic revenue loss, and along with smoothness of the unknown tail, endows $p^\ast(\cdot)$ with usable regularity. These properties allow the pricing problem to be organized around direct learning of the one-dimensional map $u \mapsto p^\ast(u)$, as opposed to the global recovery of the unknown tail function $g$. Such perspective is the key conceptual novelty of our work, and  differs from the CDF-estimation-based approaches adopted in prior studies \citep{xu2022towards,luo2022contextual,luo2024distribution,chen2024dynamic,tullii2024improved}.

The main technical challenges appear when the H\"older smoothness exponent $\beta$ exceeds strictly $2$. 
Strong unimodality controls the curvature in the price variable, while the $\beta$-H\"older condition is imposed on the unknown tail function.
For $\beta = 2$, the interaction is essentially second-order. For $\beta > 2$, the additional regularity is hidden inside the implicit maximization that defines $p^\ast(u)$. To benefit from such extra smoothness, one needs to first show that the oracle price map is itself $(\beta-1)$-smooth, and then exploit that regularity without leaving the neighborhood in which the curvature remains reliable. In other words, higher-order smoothness becomes useful only after the problem has been localized sharply enough.

\subsection{Our contributions}\label{sec:contrib}

The major contributions of our paper are summarized as follows.

\vspace{0.1cm}\noindent\textbf{Utility-agnostic oracle price map learning algorithm.} Building on the oracle-price-map perspective, we propose a modular policy, $\ORBIT$ (\emph{oracle-price learning with binning and trust-region refinement}), that exploits only the one-dimensional revenue landscape in \eqref{intro-eq: revenue}. $\ORBIT$ is agnostic to the specific utility class and the context distribution once it is supplied with a sequence of pilot utility estimates $\{\tilde u_s\}$ satisfying an appropriate approximation guarantee for the latent utilities $u_s = \mu_\ast(\mc_s)$. Such pilot sequences can be obtained by existing methods for various utility classes \citep{fan2024policy,tullii2024improved,chen2024dynamic,gong2025minimax}.

\begin{figure}[t]
\centering
\begin{tikzpicture}
\begin{axis}[
    width=10.8cm,
    height=7.8cm,
    domain=2:12,
    samples=500,
    xmin=2, xmax=12,
    ymin=0.497, ymax=0.69,
    axis x line*=bottom,
    axis y line*=left,
    axis line style={-{Latex[length=3mm,width=2mm]}, line width=1pt, black},
    xlabel={Smoothness parameter $\beta$},
    xlabel style={at={(axis description cs:0.55,-0.10)}, anchor=north},
    ylabel={},
    xtick={2,4,6,12},
    xticklabels={$2$,$4$,$6$,$+\infty$},
    ytick={0.505,0.6},
    yticklabels={$\tfrac{1}{2}$,$\tfrac{3}{5}$},
    tick label style={font=\footnotesize},
    tick style={black, line width=0.8pt},
    xmajorgrids=false,
    ymajorgrids=false,
    clip=false,
    enlargelimits=false,
    legend style={
        draw=none,
        fill=white,
        fill opacity=0.9,
        text opacity=1,
        font=\footnotesize,
        at={(0.98,0.96)},
        anchor=north east,
        row sep=1pt
    },
    legend cell align=left,
]

\node[anchor=south, align=center, font=\small] at (axis description cs:-0.07,1.02) {Exponent\\in $T$};

\addplot[
    only marks,
    mark=*,
    mark size=2.5pt,
    color=black
] coordinates {(2,3/5)};
\addlegendentry{\citet{wang2025tight}: $\tilde{\Theta}(T^{3/5})$}

\addplot[
    only marks,
    mark=o,
    mark size=2.5pt,
    thick,
    color=black,
    fill=white
] coordinates {(12,0.505)};
\addlegendentry{\citet{javanmard2019dynamic}: $\tilde{\Theta}(\sqrt{T})$}


\addplot[black, densely dotted, line width=0.9pt] coordinates {(2,0.505) (12,0.505)};
\addplot[gray!70, densely dotted, line width=0.9pt] coordinates {(2,0.6) (12,0.6)};

\addplot[name path=ourpath, draw=none] {(2*x-1)/(4*x-3)};
\addplot[red!10, forget plot] fill between[of=ourpath and ourpath, soft clip={domain=2:12}];

\addplot[blue!80!black, thick] {(2*x+1)/(4*x-1)};

\addplot[green!55!black, thick] {(x+1)/(2*x+1)};

\addplot[
    red!80!black,
    very thick,
    preaction={draw=red!15, line width=7pt, opacity=0.9}
] {(2*x-1)/(4*x-3)};

\node[
    anchor=west,
    rotate=-12,
    text=blue!80!black,
    font=\fontsize{7.5pt}{8.5pt}\selectfont,
    inner sep=1pt
] at (axis cs:5.6,0.581)
{\textbf{Fan et al. (2024): $\tilde{\cO}(\bm{T}^\frac{\bm{2\beta+1}}{\bm{4\beta-1}})$}};

\node[
    anchor=west,
    rotate=-15,
    text=green!45!black,
    font=\fontsize{7.5pt}{8.5pt}\selectfont,
    inner sep=1pt
] at (axis cs:3.2,0.5745)
{\textbf{Han et al. (2026)}: $\tilde{\cO}(\bm{T}^\frac{\bm{\beta+1}}{\bm{2\beta+1}})$};

\node[
    anchor=west,
    rotate=0,
    text=red!80!black,
    font=\fontsize{7.5pt}{8.5pt}\selectfont,
    inner sep=1pt
] at (axis cs:2.5,0.515)
{\textbf{This work:} $\tilde{\Theta}(\bm{T}^\frac{\bm{2\beta-1}}{\bm{4\beta-3}})$};

\end{axis}
\end{tikzpicture}
\caption{Comparison of regret exponents in $T$ as a function of smoothness parameter $\beta$.}
\label{fig: beta-curve}
\end{figure}
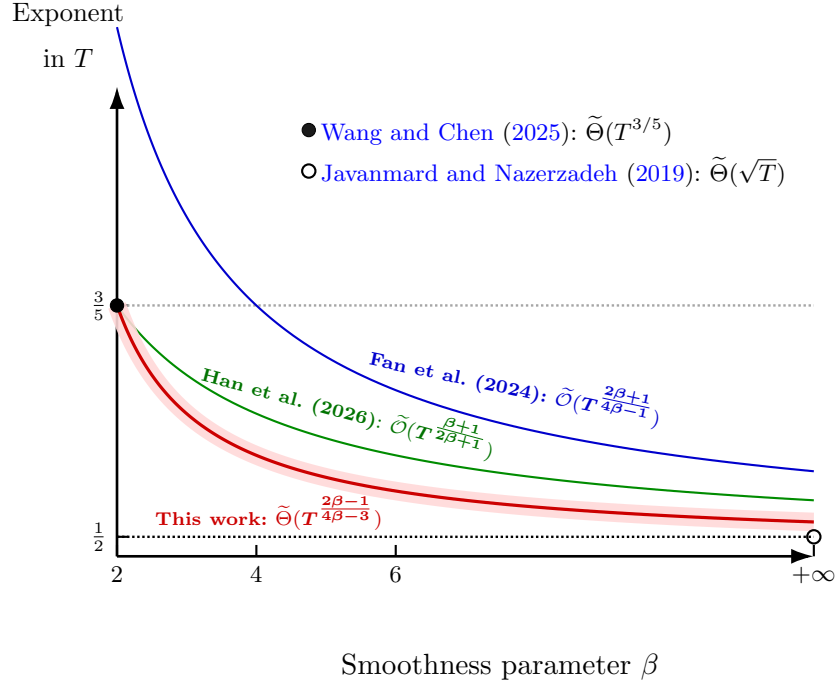

Given the pilot utility sequence, $\ORBIT$ follows a coarse-to-fine design with two interacting components. The coarse stage creates the region where the curvature is informative and the refinement stage exploits higher-order smoothness through the local polynomial approximation. Concretely, $\ORBIT$ first partitions the one-dimensional surrogate utility space into active bins. Within each active bin, a logarithmic-length coarse stage identifies a candidate price region in which the revenue is locally concave as a function of price. The subsequent refinement stage then approximates the oracle price map on this bin by a low-degree polynomial, and takes the polynomial coefficients as the new optimization variable. The remaining online problem then reduces to a \textit{bandit convex optimization} (BCO) problem over polynomial coefficients, instead of a globally nonconvex problem over arbitrary decision rules. Such reduction allows us to leverage recent developments from the BCO literature \citep{hazan2014bandit,fokkema2024online,lattimore2021improved,lattimore2024bandit} to obtain improved regret guarantees while relaxing the context-distribution requirements imposed by prior works \citep{fan2024policy,chen2024dynamic,wang2025tight,han2026general}.

\vspace{0.1cm}\noindent\textbf{Optimal rates under linear utility.} In the widely used semiparametric demand model with linear utility $\mu_\ast(\mc)=\mc^\top\theta_\ast$, we show that under the $\beta$-H\"older smoothness of the tail function with $\beta\geq 2$, combining $\ORBIT$ with the adaptive exploration scheme of \citet{tullii2024improved} achieves regret rate
\begin{align}\label{intro-eq:regret}
    \widetilde{\cO}\Bigg(\underbrace{T^{\frac{2\beta-1}{4\beta-3}}}_{\substack{\text{nonparametric}\\ \text{oracle-map learning}}}+ \underbrace{\sqrt{dT}}_{\substack{\text{parametric}\\ \text{utility-index estimation}}}\Bigg).
\end{align}
Compared to previous results for the same semiparametric model \citep{fan2024policy,wang2025tight,han2026general}, the bound in \eqref{intro-eq:regret} improves the state of the art in three important aspects.

First, the horizon exponent is \textit{sharper}. When $\beta=2$, \eqref{intro-eq:regret} recovers the optimal $\widetilde{\cO}(T^{3/5})$ rate of \citet{wang2025tight}. When $\beta>2$, the oracle-map term $\widetilde{\cO}(T^{\frac{2\beta-1}{4\beta-3}})$ improves over the best previously known rate $\widetilde{\cO}(T^{\frac{\beta+1}{2\beta+1}})$ of \citet{han2026general}; see Figure~\ref{fig: beta-curve} for an illustration. This also shows that the rate conjectured by \citet{wang2025tight}, $\widetilde{\cO}(T^{ \frac{\beta+1}{2\beta+1}})$, is not minimax optimal under the strong-unimodality structure considered here. That conjecture was natural since the same exponent is optimal in many online learning problems with $\beta$-smoothness but without unimodality \citep{hu2020smooth,wang2021multimodal, liu2021smooth,gur2022smoothness}. Our results reveal that strong unimodality changes the smoothness--regret tradeoff: it permits a faster interpolation from the nonparametric $\widetilde{\cO}(T^{3/5})$ rate at $\beta=2$ toward the parametric $\widetilde{\cO}(\sqrt T)$ benchmark \citep{javanmard2019dynamic}. In Section~\ref{sec:lower}, we establish a matching lower bound for fixed $d$, so the exponent $(2\beta-1)/(4\beta-3)$ is sharp for the linear model considered here. The hard instances exploit finite-support context distributions, admissible under our distribution-free setting. The upper and lower bounds together provide a complete characterization of the effect of strong unimodality on the regret rate for general $\beta$.

Second, the dependence on the context dimensionality $d$ is \textit{separated} from the nonparametric learning term. The rate in \eqref{intro-eq:regret} decomposes the semiparametric pricing problem into a $\beta$-dependent nonparametric oracle-map learning term and a $d$-dependent parametric utility-index estimation term. Such separation is \textit{not} automatic. The model contains an unknown infinite-dimensional tail function, and the seller observes only binary purchase feedback under adaptively chosen prices, so utility-index estimation and price-response learning are statistically intertwined. Once the latent utility index $u=\mc^\top\theta_\ast$ is known, the nonparametric object relevant for pricing is, however, not a $d$-dimensional demand surface, but the one-dimensional oracle price map $u\mapsto p^\ast(u)$. Thus, the nonparametric difficulty should be governed by the smoothness and local curvature of this scalar map, while the ambient dimensionality should enter through parametric utility-index estimation. Existing analyses do \textit{not} fully separate these two sources of complexity: their leading nonparametric terms carry polynomial factors in $d$; see Table~\ref{tab:results-summary} for a summary. The additive form in \eqref{intro-eq:regret} realizes the desired separation: all polynomial dependence on $d$ is confined to the parametric term $\sqrt{dT}$, while the nonparametric oracle-map term $T^{\frac{2\beta-1}{4\beta-3}}$ is dimension-free. In this sense, $\ORBIT$ \textit{decouples} the parametric and nonparametric rates in semiparametric contextual pricing.

Third, the regret guarantee in~\eqref{intro-eq:regret} does \textit{not} require any context-distribution assumption. Existing analyses of the same semiparametric model estimate typically the utility index from fixed or epoch-wise exploration samples. Their guarantees thus require the context distribution to cover all directions well enough, so that a fixed exploration sample is informative in all directions. For example, \citet{fan2024policy} assumed a covariance eigenvalue lower bound, while \citet{wang2025tight,han2026general} assumed independent and identically distributed (i.i.d.) contexts together with the same type of covariance nondegeneracy; see Table~\ref{tab:results-summary} for more details. We avoid such requirement by combining $\ORBIT$ with the adaptive exploration design of \citet{tullii2024improved}. Such design explores only when the current context is not yet well covered by previous exploration samples, and sends a context to $\ORBIT$ only after its utility index is estimated accurately enough. This argument is pathwise over the realized context sequence, and therefore applies to \textit{discrete, continuous, and mixed} context distributions. Since the refinement stage is analyzed as an adversarial BCO problem, the resulting scalar sequence does not need to be i.i.d. or homogeneous. Consequently, the regret bound in \eqref{intro-eq:regret} avoids the non-degeneracy, homogeneity, and i.i.d. context assumptions required in previous smooth-regime analyses.

\vspace{0.1cm}
\noindent\textbf{General utilities through an offline pilot interface.}
$\ORBIT$ also extends beyond linear utilities to broader utility classes. In Section~\ref{sec:extensions}, we replace the linear-utility assumption with the \emph{realizable} setting: the true utility $\mu_\ast$ is assumed to belong to a known function class $\cF$. To handle this case, we equip $\ORBIT$ with an offline pilot interface: given $n$ randomized exploration observations, an estimation oracle for $\cF$ returns an estimate $\widehat\mu_n$ that, with high probability, satisfies
\begin{equation}\label{intro-eq:oracle-condition}
    \|\widehat\mu_n - \mu_\ast\|_\infty \lesssim \cV_T(\cF)\, n^{-\alpha},
\end{equation}
for some $\alpha > 0$. Here $\cV_T(\cF)$ measures the complexity of the class $\cF$; its dependence on $T$ enters only through the failure probability.
Compared with the adaptive pilot used in the linear case, this offline interface trades context-distribution flexibility for function-class generality: it can be combined with any utility class admitting such an estimation oracle, in line with the offline-oracle viewpoint in general contextual decision making \citep{foster2018practical,simchi2022bypassing,gong2025minimax}, while the
required context-distribution assumptions are absorbed implicitly into the
oracle condition~\eqref{intro-eq:oracle-condition}.

With a suitable choice of exploration length, the resulting explore-then-ORBIT
policy achieves regret rate
\begin{equation}\label{intro-eq:general-valuation-regret}
    \widetilde{\cO}\Big(
        T^{\frac{2\beta-1}{4\beta-3}}
        + T^{\frac{1}{1+2\alpha}}\, \cV_T(\cF)^{\frac{2}{1+2\alpha}}
    \Big).
\end{equation}
This is the smooth regime $(\beta \geq 2)$ counterpart of the Lipschitz
setting result of \citet{gong2025minimax}, which in matched notation reads
$\widetilde{\cO}\big( ( T\cV_T(\cF))^{\frac{2}{3}}+( T\cV_T(\cF))^{\frac{1}{1+\alpha}} \big)$.
As in the linear case, the bound \textit{decouples} cleanly the two sources of
complexity: the first term depends only on the smoothness and
strong-unimodality structure governing the oracle-map learning, while the second
one depends only on the hardness of pilot estimation over $\cF$.

\begin{table}[t]
\setlength{\extrarowheight}{8pt}
\resizebox{\textwidth}{!}{%
\begin{tabular}{llll}
\toprule
\textbf{Utility} & \textbf{Work} & \textbf{Regret} & \textbf{Context Assumption} \\
\midrule

\rowcolor{LinearColor}
  & \citet{fan2024policy} 
  & $\tilde{\cO}((dT)^{\frac{2\beta+1}{4\beta-1}})$ 
  & $\E[\mc_t\mc_t^\top ] \succ 0$ \\

\rowcolor{LinearColor}
  & \citet{wang2025tight} ($\beta = 2$) 
  & $\tilde{\cO}(d^{3/2}T^{\frac{3}{5}})$ 
  & \cellcolor{IIDColor} \\

\rowcolor{LinearColor}
  & \citet{han2026general} 
  & $\tilde{\cO}(d^4T^{\frac{\beta+1}{2\beta+1}} + \text{poly}(d^\beta))$ 
  & \cellcolor{IIDColor}
    \multirow{-2}{*}{\makecell[l]{i.i.d. context, $\E[\mc_t\mc_t^\top] \succ 0$}} \\

\rowcolor{LinearColor}
\multirow{-4}{*}{\cellcolor{LinearColor}\textbf{Linear}}
  & \textbf{This work, Corollary~\ref{cor:upper}}
  & $\tilde{\cO}(T^{\frac{2\beta-1}{4\beta-3}} + \sqrt{dT})$ 
  & -- \\

\addlinespace[3pt]

\rowcolor{SparseColor}
  & \citet{javanmard2019dynamic}$^*$ 
  & $\tilde\cO(s\sqrt{T})$ 
  & \cellcolor{SparseColor} \\

\rowcolor{SparseColor}
\multirow{-2}{*}{\cellcolor{SparseColor}\textbf{Sparse}}
  & \textbf{This work, Corollary~\ref{thm:regret-sparse}}  
  & $\tilde{\cO}(T^{\frac{2\beta-1}{4\beta-3}} + s\sqrt{T})$ 
  & \cellcolor{SparseColor}
    \multirow{-2}{*}{\makecell[l]{Compatibility condition (Assumption~\ref{assumption:compatibility})}} \\

\addlinespace[3pt]

\rowcolor{HolderColor}
  & \citet{chen2024dynamic} ($\gamma = 2,4$)$^\dagger$ 
  & $\tilde{\cO}(dT^{\frac{2\beta+1}{4\beta-1}} + T^{\frac{d+2\gamma}{d+4\gamma}})$ 
  & \cellcolor{HolderColor} \\

\rowcolor{HolderColor}
\multirow{-2}{*}{\cellcolor{HolderColor}\textbf{$\gamma$-H\"{o}lder}}
  & \textbf{This work, Corollary~\ref{thm:holder-regret}} ($\gamma > 0$) 
  & $\tilde{\cO}(T^{\frac{2\beta-1}{4\beta-3}} + T^{\frac{d+2\gamma}{d+4\gamma}})$ 
  & \cellcolor{HolderColor}
    \multirow{-2}{*}{\makecell[l]{Density condition (Assumption~\ref{assumption:density})}} \\
\bottomrule
\end{tabular}}
\caption{Summary of regret results under different utility structures and assumptions. All mentioned works have additional assumptions on the shape of $F_\Xi$ that are stronger than or equal to Assumption~\ref{assump:rho}.}
{\footnotesize $^*$ \citet{javanmard2019dynamic} assumed a parametric form of $F_\Xi$.\\
$^\dagger$ When $\gamma=4,$ \citet{chen2024dynamic} had additional additive $T^{7/13}$ term.\\
}
\label{tab:results-summary}
\end{table}

We instantiate this result for two popular utility classes outside the linear case analysis, in each case under context assumptions that are no
stronger than those used in the closest prior work. For $s$-sparse
high-dimensional linear utilities, a Lasso oracle under the standard
compatibility condition gives that 
$$
    \widetilde{\cO}\Big( T^{\frac{2\beta-1}{4\beta-3}} + s\sqrt{T} \Big),
$$
extending the linear-utility analysis of \citet{javanmard2019dynamic} to
nonparametric valuation noise. For $\gamma$-H\"older nonparametric utilities,
a local polynomial oracle under a bounded-density condition yields that 
$$
\widetilde{\cO}\Big(
        T^{\frac{2\beta-1}{4\beta-3}}
        + T^{\frac{d+2\gamma}{d+4\gamma}}
    \Big),
$$
sharpening the tail-smoothness-dependent term in \citet{chen2024dynamic} and
extending their analysis from $\gamma \in \{2,4\}$ to arbitrary $\gamma > 0$.
These consequences are summarized in Table~\ref{tab:results-summary}.

\subsection{Related works}\label{sec:literature}

The literature is best viewed along a spectrum. At one end are known-link contextual pricing models, where the demand curve is specified up to a finite-dimensional parameter; see, e.g., \cite{javanmard2019dynamic,cohen2020feature,ban2021personalized,xu2021logarithmic}, among others. Once the link is fixed, the remaining learning problem is largely parametric. At the other end are fully nonparametric demand models; see, e.g., \cite{chen2021nonparametric,tullii2024improved}. Those formulations offer greater modeling freedom, but learning must then proceed in the full contextual space or over a much broader valuation class. Our model sits \textit{between} these two extremes: the latent index compresses contextual heterogeneity, while the unknown noise law keeps the pricing problem genuinely nonparametric.

This flexible modeling has been extensively studied in \cite{xu2022towards,fan2024policy,luo2022contextual,luo2024distribution,chen2024dynamic,bracale2025dynamic,gong2025minimax,wang2025tight,bracale2025revenue,han2026general}, as well as the linear-valuation branch of \cite{tullii2024improved}; in a different semiparametric pricing model, \cite{shah2019semiparametric} earlier obtained an $\widetilde{\cO}(\sqrt{T})$ regret guarantee. These works show that semiparametric structure can improve substantially over fully nonparametric contextual learning, but also make clear how difficult it is to combine online exploration with an unknown link function.

Among these works, \cite{wang2025tight} and \cite{han2026general} are the closest to ours. \cite{wang2025tight} obtained the sharp $\widetilde{\cO}(T^{3/5})$ benchmark in the twice-smooth regime using a \textit{very different} reduction based on contextual successive elimination, generalized least squares, and active-learning ideas. \cite{han2026general} \textit{extended} that line to general smoothness by combining the stationary subroutine of \cite{wang2025tight} with local polynomial regression under a unimodality condition. In contrast, our algorithm and proof do \textit{not} follow that route. We first localize the relevant price region and then treat the post-localization problem as a sequence of trust-region adversarial bandit convex optimization tasks. In the model class studied here, such reduction is what enables the \textit{curvature and higher-order smoothness} to work together cleanly and what yields \textit{minimax-sharp} horizon dependence for fixed $d$.

A further point of contact is \cite{fan2024policy}, who also reduce pricing to learning the oracle price map, but through a two-stage route: estimating the noise CDF and then reading off the oracle price via its analytical expression in terms of the CDF. This closed-form step imposes structural restrictions on the demand model that are stronger than what $\ORBIT$ requires. $\ORBIT$ instead works directly in policy space, approximating the oracle price map by a low-degree polynomial and reducing the residual problem to a bandit convex optimization over the polynomial coefficients. 

Finally, our analysis is methodologically connected to bandit convex optimization, an area extensively studied by \citet{flaxman2005online,agarwal2011stochastic,saha2011improved,hazan2014bandit,bubeck2016multi,bubeck2021kernel,lattimore2021improved,lattimore2024bandit,fokkema2024online}. Once the refinement stage of $\ORBIT$ is set up, its inner subroutine can be instantiated with any adversarial BCO algorithm. We use \citet{fokkema2024online} in our theoretical analysis for its sharpest known regret rate $\tilde{\cO}(\sqrt{T})$. The reduction itself, however, is not off-the-shelf: revenue is not globally concave in the local policy parameters, and the observations collected within a bin are \textit{endogenously} selected. The coarse-localization step and the conditioning argument bridge that gap.

\vspace{0.1cm}
\noindent\textit{Organization.}
The rest of the paper follows this modular structure. Section~\ref{sec:model} defines the general scalar-index pricing framework and proves the oracle-map regularity properties used throughout. Section~\ref{sec:oplcb} presents $\ORBIT$ under an abstract scalar-pilot interface. Section~\ref{sec:linear-main} specializes the framework to the baseline linear utility model, constructs the scalar pilot adaptively, derives the fully online upper bound, and establishes the matching lower bound. Section~\ref{sec:extensions} shows how other pilot estimators can be plugged into the same $\ORBIT$ interface for sparse linear and nonparametric utilities.  Section~\ref{sec:experiments} provides simulation results that illustrate the empirical performance of $\ORBIT$ and compare it to relevant benchmarks. Section~\ref{sec:conclusion} concludes with some discussions.

\vspace{0.1cm}\noindent\textit{Notation.} For any closed interval $I=[a,b]\subset\R$, we write $\proj_I(x):=\min\{\max\{x,a\},b\}$ for the Euclidean projection of $x$ onto $I$. Given $d\in \mathbb{Z}_+$, $1\leq p\leq \infty$, and $M > 0$, denote by $\mathbb{B}^d_p(M)$ the $d$-dimensional $\ell_p$-ball with radius $M$.

\section{Problem formulation and assumptions}\label{sec:model}

\subsection{Dynamic pricing with binary feedback}

In each period $t\in[T]$, a customer arrives with an observable covariate vector $\mc_t\in\cC$. The context sequence may be random, but is generated exogenously, and independent of the demand-noise sequence and the seller's internal randomization. Equivalently, throughout the analysis we may condition on the entire context sequence; conditional on it, the noises below remain independent draws from their common law. With some deterministic utility map $\mu_\ast:\cC\to\R$ and $u_t=\mu_\ast(\mc_t)$, the latent valuation is generated by
\begin{align*}
v_t:=u_t+\xi_t,
\end{align*}
where $\{\xi_t\}_{t\geq 1}$ are i.i.d. draws from an unknown zero-mean distribution $P_\Xi$ with cumulative distribution function $F_\Xi$. After observing $\mc_t$, the seller posts a price $p_t\in[0,p_{\max}]$ and observes the binary purchase indicator
\begin{align*}
y_t = \bm 1\{v_t\geq p_t\}.
\end{align*}
We work throughout with the normalized model class in which $v_t\in[0,p_{\max}]$. Accordingly, the seller restricts attention to the same price interval $[0,p_{\max}]$; prices above $p_{\max}$ induce zero demand and can thus be excluded from the benchmark without loss.

Denote by $g(z):=1-F_\Xi(z)$ the tail function of the noise distribution. Under the product-law assumption, conditional on $(\mc_t,p_t)$ the only remaining randomness is $\xi_t$. Hence, the conditional purchase probability and conditional revenue are given by
\begin{align*}
\E[y_t\mid \mc_t,p_t] = \Prob\big(\xi_t\geq p_t-u_t \mid \mc_t,p_t\big) = g(p_t-u_t), \quad R(\mc_t,p_t)=p_t g(p_t-u_t),
\end{align*}
respectively. The performance of a pricing policy $\pi := \{ \pi_t\}_{t\in [T]}$ is measured by the regret
\begin{align}\label{eq:regret-def}
\Regret_{\pi}(T):= \E\bigg[ \sum_{t \in [T]} \max_{p\in[0,p_{\max}]}R(\mc_t,p)-R(\mc_t,\pi_t(\mc_t)) \bigg].
\end{align}
When the policy is clear from the context, we suppress the subscript and write $\Regret(T)$ instead of $\Regret_\pi(T)$ for simplicity.

For notational convenience, for each real number $u$, we define the one-dimensional revenue function as 
\begin{align*}
r(u,p):=p g(p-u), \quad p\in[0,p_{\max}].
\end{align*}
The oracle benchmark in \eqref{eq:regret-def} at each $t$ is thus $\max_{p\in[0,p_{\max}]} r(u_t,p).$
If both $\mu_\ast$ and $g$ were known, the seller would post $p^\ast(\mu_\ast(\mc_t))$ in each period. The goal is to learn such oracle pricing rule when \textit{neither} component is known. 

At a high level, the contextual structure enters the decision problem only through the one-dimensional index $u_t$. This is what makes the model \textit{semiparametric} instead of fully nonparametric: although the noise law is unknown, the decision-relevant heterogeneity is compressed into a scalar latent state.  

\subsection{Smoothness and revenue curvature}

In this section, we impose some structural assumptions on the tail function $g(\cdot)$ and revenue function $r(\cdot)$ as a function of utility $u$ and price $p$, where no specific form of $u=\mu_\ast(\mc)$ is imposed beyond the following boundedness assumption.

\begin{assumption}[Boundedness]\label{assump:bounded}
There exist 
some $u_{\min}<u_{\max}$ such that $\mu_\ast(\mc)\in\cU:=[u_{\min},u_{\max}]$ for all $\mc\in\cC$. Denote by $\Delta_{\cU}:=u_{\max}-u_{\min}$ and $V:=p_{\max}+\max\{\abs{u_{\min}},\abs{u_{\max}}\}$.
\end{assumption}

Assumption~\ref{assump:bounded} above bounds the latent index instead of the noise directly. Since $p_t\in[0,p_{\max}]$ and $u_t=\mu_\ast(\mc_t)\in \cU$, the price-index gap $p_t-u_t$ always lies in a compact interval contained in $[-V,V]$. Consequently, the unknown tail function is evaluated only on a compact interval. For the common nonnegative-index case, one can take $u_{\min}=0$ and $u_{\max}=U$, but the results below are translation-invariant in the scalar index.

For the nonparametric component, we will assume that $g$ is $\beta$-H\"older smooth.
\begin{assumption}[$\beta$-H\"older smoothness of $g$]\label{assump:holder}
There exist constants $L_g>0$ and $\beta\geq 2$ such that $g:[-V,V]\to [0,1]$ is
$\lfloor\beta\rfloor$ times continuously differentiable, $\max_{1\leq k\leq \lfloor\beta\rfloor}\|g^{(k)}\|_\infty\leq L_g$, and for all $u,u'\in[-V,V]$,
\begin{align*}
\biggabs{ g(u') - \sum_{k=0}^{\lfloor\beta\rfloor} \frac{(u'-u)^k}{k!} g^{(k)}(u)} \leq L_g \abs{u'-u}^{\beta}.
\end{align*}
\end{assumption}

Assumption \ref{assump:holder} above is a standard smoothness condition in nonparametric estimation; see, e.g., \cite{gyorfi2002distribution,TsybakovNon}. It ensures that the revenue function, and thus the oracle price map $p^{\ast}(u)$, is locally regular and well approximated by a polynomial function, which serves as the foundation of the local refinement step in our algorithm.  The H\"older notation also encompasses the Lipschitz and second-order smooth settings studied in \cite{tullii2024improved,luo2024distribution,wang2025tight,luo2022contextual}, but the technical analysis in the current paper focuses on the regime of $\beta\geq 2$. The next assumption imposes curvature on the pricing side of the problem.

\begin{assumption}[Strong unimodality]\label{assump:rho}
There exist constants $0<\sigma_r\leq L_r 
$ such that for each $u\in\cU$,
\begin{enumerate}
    \item[1)] the maximizer $p^\ast(u)$ is unique and lies in the strict interior $(0,p_{\max})$;
    \item[2)] the global quadratic growth bounds hold that for all $p\in[0,p_{\max}]$, 
    \begin{align*}
        \frac{\sigma_r}{2}\abs{p-p^\ast(u)}^2 \leq r(u,p^\ast(u))-r(u,p) \leq \frac{L_r}{2}\abs{p-p^\ast(u)}^2.
    \end{align*}
\end{enumerate}
\end{assumption}

Assumption~\ref{assump:rho} above states that the price error translates into the quadratic revenue loss around the oracle price, uniformly over the scalar-index domain. Such condition has appeared in various pricing models \citep{broder2012dynamic,wang2014close,chen2021nonparametric,wang2025tight,han2026general}. In semiparametric models with a differentiable noise density, the quadratic upper and lower growth bounds in Assumption~\ref{assump:rho} can be derived from CDF-level shape conditions analogous to those in \citet{fan2024policy,javanmard2019dynamic}: a bounded density with bounded derivative and a monotone virtual-valuation map imply the same revenue geometry once the stationary price exists uniformly in the interior of the admissible price interval. Appendix~\ref{appenidx: compare-assumption} gives the precise derivations.

We conclude this section by discussing two important implications of Assumptions~\ref{assump:bounded}--\ref{assump:rho}. Both play important roles in our algorithm design and theoretical analysis later.

\medskip

\noindent\textbf{Local concavity.} Under Assumption~\ref{assump:bounded}--\ref{assump:rho}, the following local concave property holds around the optimal price $p^{\ast}(u)$ for each $u \in \cU$.

\begin{lemma}\label{lem:rho-exists}
Under Assumptions~\ref{assump:bounded}--\ref{assump:rho}, there exists a constant $\rho_0>0$ so that for each $u\in\cU$,
\begin{align*}
    [p^\ast(u)-\rho_0,p^\ast(u)+\rho_0]\subset (0,p_{\max})
\end{align*}
and
\begin{align*}
    -\frac{\partial^2}{\partial p^2} r (u,p)\geq \sigma_r/2 \quad \text{whenever } \abs{p-p^\ast(u)}\leq \rho_0.
\end{align*}
\end{lemma}
Such local concavity allows us to connect revenue maximization to bandit convex optimization, and plays an important role in the localization-then-refine design in Section~\ref{sec:oplcb}.

\medskip

\noindent\textbf{Oracle price map.} Assumption~\ref{assump:rho} ensures that the maximizer of $p\mapsto r(u,p)$ is unique for each $u\in\cU$. We therefore define the \textit{oracle price map} as 
\begin{align*}
p^\ast(u):=\argmax_{p\in[0,p_{\max}]} r(u,p), \quad u\in\cU.
\end{align*}
Assumptions~\ref{assump:bounded}--\ref{assump:rho} together entail the following regularity property of this map.

\begin{lemma}\label{lem:pstar-regularity}
Under Assumptions~\ref{assump:bounded}--\ref{assump:rho}, the oracle price map $p^\ast:\cU\to(0,p_{\max})$ belongs to class $C^{\beta-1}$. In particular, there exists a constant $L_p>0$ such that
\begin{align*}
    \abs{p^\ast(u)-p^\ast(v)}\leq L_p\abs{u-v},
    \quad \forall u,v\in\cU.
\end{align*}
\end{lemma}

Whenever a Taylor expansion of $p^\ast$ is used below, we use the following standard consequence of Lemma~\ref{lem:pstar-regularity}: after possibly increasing $L_p$, for all $x,u\in\cU$,
\begin{align*}
    \biggabs{p^\ast(u)-\sum_{k=0}^{\lceil \beta-1\rceil-1}\frac{(p^\ast)^{(k)}(x)}{k!}(u-x)^k} \leq L_p\abs{u-x}^{\beta-1}.
\end{align*}

Such regularity ensures that the oracle price map can be well approximated by standard nonparametric function classes (e.g., the local polynomials class), which makes direct learning of $p^{\ast}$ feasible. This observation drives the design of the policy space in Section~\ref{sec:oplcb}.

\section{$\ORBIT$ under a scalar pilot interface}\label{sec:oplcb}

In this section, we will present $\ORBIT$ as a conditional, model-agnostic module. The full pricing policy observes context $\mc_t$ in each period; $\ORBIT$ describes the pricing step after an outer pilot-construction mechanism has produced a scalar state $\tilde u_t$. The module itself does not estimate the utility map. It takes $\tilde u_t$ as its direct input, while the latent index $u_t=\mu_\ast(\mc_t)$ remains unobserved and is used only for analysis. Given these pilot states, $\ORBIT$ learns the oracle price map $u\mapsto p^\ast(u)$ on the one-dimensional scalar-index space. The model-specific question of how to construct $\tilde u_t$ is addressed in Section~\ref{sec:linear-main} for the baseline linear model, and in Section~\ref{sec:extensions} for additional utility classes.

The design has two local phases. First, each active bin runs a short grid-based search over prices and stores an anchor price. The purpose of this anchor is to put the bin in a safe local price range. Once the anchor is close to the oracle prices in the bin, we consider only prices in a small band around that anchor. Lemma~\ref{lem:rho-exists} ensures that in such a local range, the revenue curve is concave in price, or equivalently that negative revenue is convex in price. Second, the bin learns a local polynomial price rule around the anchor. For any fixed customer assigned to the bin, the price prescribed by such polynomial rule is linear in the polynomial coefficients. Hence, after the coarse search has localized the price range, choosing the polynomial coefficients becomes a convex bandit learning problem: the raw refinement generator selects coefficients, the policy posts the corresponding price, and the observed purchase outcome provides noisy feedback. Figure~\ref{fig:oplcb-flow} summarizes this flow.

\begin{assumption}[Scalar pilot interface]\label{assump:orbit-interface}
For some $0<\eta\leq 1/2$, $\ORBIT$ receives a realized stream of $N$ scalar pilot states
$\{\tilde u_s\}_{s=1}^N$ generated by an outer pricing policy, where $N$ may be any value not exceeding the upper pilot-input budget used to initialize the subroutine. The pilot stream and the stopping rule that determines $N$ may use exogenous contexts, observations from non-$\ORBIT$ exploration rounds, and pilot-side randomization independent of $\ORBIT$. They may not use $\ORBIT$-posted prices, $\ORBIT$ purchase outcomes, $\ORBIT$'s internal refinement randomization, or the fresh demand noises generated on $\ORBIT$ calls. Let $\mc_s$ be the context associated with the $s$th $\ORBIT$ call after any re-indexing, and set
$u_s=\mu_\ast(\mc_s)$ as the corresponding latent scalar index. Assume that the
associated index--pilot sequence $\{(u_s,\tilde u_s)\}_{s=1}^N$ satisfies that 
$u_s,\tilde u_s\in\cU$ for each $s\in[N]$, and
\begin{enumerate}
    \item[1)] \textbf{Independence:} the random object $(N,\{(u_s,\tilde u_s)\}_{s=1}^N)$ is independent of both the re-indexed demand noises $\{\xi_s\}_{s=1}^N$ used in these $\ORBIT$ calls and the internal randomization used by $\ORBIT$'s refinement generators. Equivalently, conditional on the realized pilot stream and its length, the $\ORBIT$-call noises remain independent draws from the original noise law, and $\ORBIT$'s internal randomization remains fresh and independent of the pilot stream.
    \item[2)] \textbf{Pilot accuracy:} $\sup_{s\in[N]} \abs{\tilde u_s-u_s}\leq \eta$.
\end{enumerate}
\end{assumption}

The feasibility requirement $\tilde u_s\in\cU$ gives each pilot state a unique bin assignment; all pilot constructions below enforce it by projection onto $\cU$. The strengthened independence condition is an interface restriction on the outer pilot: it rules out pilots that adapt future scalar states or the $\ORBIT$ stopping rule using $\ORBIT$ prices, $\ORBIT$ feedback, or $\ORBIT$ refinement randomization. This is the regime analyzed by the induced fixed-loss BCO reduction below. Throughout this section, the analysis is conditional on the associated index--pilot sequence and on its realized length. $\ORBIT$ itself uses only $\tilde u_s$ for binning and price selection; the latent index $u_s$ appears only in the analysis.

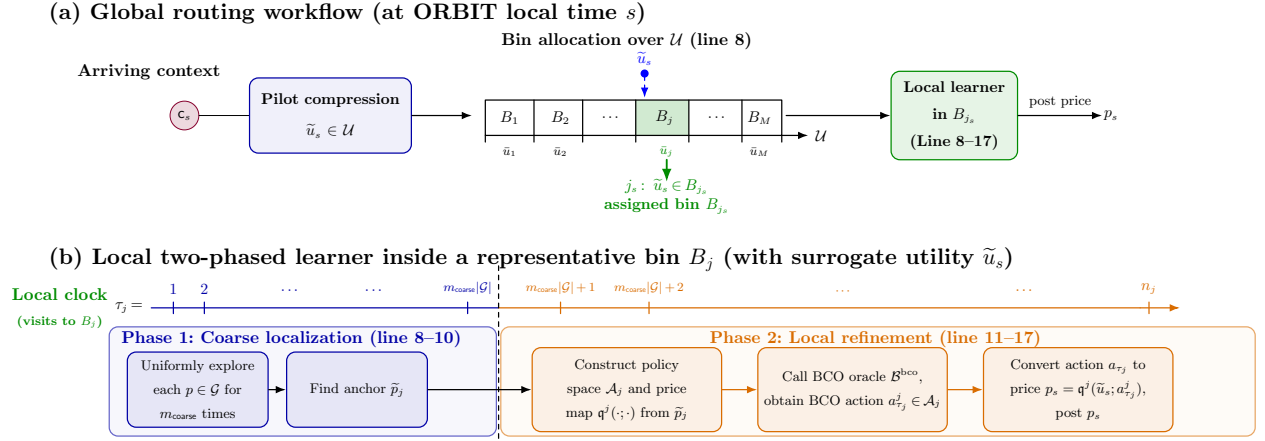
\begin{figure}[t]
\centering

\resizebox{\textwidth}{!}{%
\begin{tikzpicture}[>=Latex, line cap=round, line join=round, font=\small]
\colorlet{myblue}{blue!65!black}
\colorlet{myorange}{orange!85!black}
\colorlet{mygreen}{green!55!black}
\colorlet{mypurple}{purple!70!black}
\colorlet{mygray}{black!55}
\tikzset{
  title/.style={font=\bfseries\large, anchor=west},
  context/.style={
    circle,
    draw=mypurple,
    fill=mypurple!15,
    minimum size=6.5mm,
    inner sep=0pt,
    font=\small
  },
  pilotbox/.style={
    rounded corners=2mm, draw=myblue, fill=myblue!6, thick,
    align=center, minimum width=3.6cm, minimum height=1.7cm, text width=3.4cm
  },
  localbox/.style={
    rounded corners=2mm, draw=mygreen, fill=mygreen!10, thick,
    align=center, minimum width=2.7cm, minimum height=2.0cm, text width=2.6cm
  },
  bluephase/.style={
    rounded corners=2mm, draw=myblue!60, fill=myblue!3, thick
  },
  orangephase/.style={
    rounded corners=2mm, draw=myorange!60, fill=myorange!3, thick
  },
  bluebox/.style={
    rounded corners=2mm, draw=myblue, fill=myblue!10, thick,
    align=center, text width=2.9cm, minimum height=1.7cm, inner sep=4pt, font=\footnotesize
  },
  orangebox/.style={
    rounded corners=2mm, draw=myorange, fill=myorange!10, thick,
    align=center, text width=4cm, minimum height=1.9cm, inner sep=4pt, font=\footnotesize
  },
  noteBlue/.style={
    rounded corners=2mm, draw=myblue, dashed, fill=myblue!3,
    align=center, inner sep=5pt, font=\footnotesize
  },
  noteOrange/.style={
    rounded corners=2mm, draw=myorange, dashed, fill=myorange!4,
    align=left, inner sep=5pt, font=\footnotesize
  },
  zoomlabel/.style={
    rounded corners=1.5mm, draw=black!50, fill=black!5, inner sep=6pt, align=center, font=\bfseries
  },
  legendbox/.style={
    rounded corners=2mm, draw=black!55, dashed, fill=white, inner sep=6pt
  }
}
\node[title] at (0,10.5) {(a) Global routing workflow (at ORBIT local time $s$)};
\node[font=\bfseries] at (2.4,9.2) {Arriving context};
\node[context] (ct) at (3.2,8.2) {$\mc_s$};

\draw[thin] (ct.east) -- (6.0,8.2);
\node[pilotbox] (pilot) at (6.5,8.2) {%
\textbf{Pilot compression}\\[1mm]
$\tilde u_s \in \mathcal U$
};
\draw[->, thick] (pilot.east) -- (9.7,8.2);
\node[font=\bfseries] at (13.2,9.9) {Bin allocation over $\mathcal U$ (line~8)};
\def\binA{10.0}
\def\binB{11.1}
\def\binC{12.2}
\def\binE{13.4}  
\def\binF{14.6}  
\def\binH{15.8}
\def\binI{16.7}

\fill[mygreen!18] (\binE,7.75) rectangle (\binF,8.55);
\draw[thick] (\binA,7.75) rectangle (\binI,8.55);
\foreach \x in {\binB,\binC,\binE,\binF,\binH}{
    \draw[thick] (\x,7.75) -- (\x,8.55);
}
\foreach \x in {\binA,\binB,\binC,\binE,\binF,\binH,\binI}{
    \draw[thick] (\x,7.75) -- (\x,7.6);
}
\node at (10.55,8.15) {$B_1$};
\node at (11.70,8.15) {$B_2$};
\node at (12.85,8.15) {$\cdots$};
\node[font=\small] at (14.05,8.15) {$B_{j}$};
\node at (15.25,8.15) {$\cdots$};
\node at (16.2,8.15) {$B_M$};
\draw[->, thick] (\binI,7.75) -- (17.3,7.75);
\node[font=\small] at (17.6,7.75) {$\mathcal U$};
\node[font=\scriptsize] at (10.55,7.38) {$\bar u_1$};
\node[font=\scriptsize] at (11.70,7.38) {$\bar u_2$};
\node[font=\scriptsize, text=mygreen] at (14.10,7.38) {$\bar u_{j}$};
\node[font=\scriptsize] at (16.2,7.38) {$\bar u_M$};

\fill[blue] (13.6,9.15) circle (2.6pt);
\node[blue, above=0pt, font=\small] at (13.6,9.20) {$\tilde u_s$};
\draw[blue, dashed, ->, thick] (13.6,9.1) -- (13.6,8.6);
\draw[mygreen, very thick, ->] (14.1,7.2) -- (14.1,6.70);
\node[mygreen, font=\small]      at (14.1,6.6) {$j_s:\ \tilde u_s\in B_{j_s}$};
\node[mygreen, font=\bfseries\small] at (14.1,6.15) {assigned bin $B_{j_s}$};
\node[localbox] (lp) at (20.6,8.2) {%
\textbf{Local learner}\\
\textbf{in $B_{j_s}$}\\
\textbf{(Line~8--17)}
};
\draw[->, thick] (16.8,8.2) -- (lp.west);
\coordinate (P1) at ($(lp.east)+(0,0)$);
\coordinate (P2) at ($(lp.east)+(1.9,0)$);
\draw[->, thick] (P1) -- (P2);
\node[font=\footnotesize, above=1pt] at ($(P1)!0.5!(P2)$) {post price};
\node at ($(P2)+(0.2,0)$) {$p_s$};
\node[title] at (0,4.95) {(b) Local two-phased learner inside a representative bin $B_j$ (with surrogate utility $\tilde{u}_s$)};
\node[align=center, text=mygreen, font=\bfseries, anchor=east] at (1.6,3.85)
    {Local clock\\[-1mm]\scriptsize (visits to $B_j$)};
\node[anchor=east, font=\small] at (2.4,3.85) {$\tau_j=$};
\draw[myblue, thick] (2.45,3.85) -- (10.3,3.85);
\draw[myorange, thick] (10.3,3.85) -- (24.8,3.85);
\draw[->, thick, myorange] (24.8,3.85) -- (25.7,3.85);
\foreach \x in {2.95,3.65,9.6}{
    \draw[myblue, thick] (\x,3.72) -- (\x,3.98);
}
\node[myblue, font=\small, above=1pt] at (2.95,4.0) {$1$};
\node[myblue, font=\small, above=1pt] at (3.65,4.0) {$2$};
\node[myblue, font=\small, above=1pt] at (5.6,4.0) {$\cdots$};
\node[myblue, font=\small, above=1pt] at (7.5,4.0) {$\cdots$};
\node[myblue, font=\scriptsize, above=1pt] at (9.6,4.0) {$m_{\coarse}|\mathcal G|$};
\foreach \x in {11.7,13.7,25.0}{
    \draw[myorange, thick] (\x,3.72) -- (\x,3.98);
}
\node[myorange, font=\scriptsize, above=1pt] at (11.7,4.0) {$m_{\coarse}|\mathcal G|+1$};
\node[myorange, font=\scriptsize, above=1pt] at (13.7,4.0) {$m_{\coarse}|\mathcal G|+2$};
\node[myorange, font=\scriptsize, above=1pt] at (18.1,4.0) {$\cdots$};
\node[myorange, font=\small, above=1pt] at (22.2,4.0) {$\cdots$};
\node[myorange, font=\small, above=1pt] at (25.0,4.0) {$n_j$};
\draw[dashed, thick] (10.3,4.8) -- (10.3,0.85);
\draw[bluephase] (1.5,0.95) rectangle (10.25,3.42);
\draw[orangephase] (10.35,0.95) rectangle (27.4,3.42);
\node[myblue, font=\bfseries] at (5.62,3.15) {Phase 1: Coarse localization (line~8--10)};
\node[myorange, font=\bfseries] at (18.82,3.15) {Phase 2: Local refinement (line~11--17)};
\node[bluebox] (b1) at (3.5,2.00)
    {Uniformly explore each $p\in\mathcal G$ for $m_{\coarse}$ times};
\node[bluebox] (b2) at (7.1,2.00)
    {Find anchor $\tilde{p}_j$};
\draw[->, thick] (b1.east) -- (b2.west);
\node[orangebox] (o2) at (13.2,2.00) {%
Construct policy space $\mathcal A_j$ and price map $\frakq^j(\cdot;\cdot)$ from $\tilde p_j$};
\node[orangebox] (o3) at (18.3,2.00) {%
Call BCO oracle $\mathcal B^{\mathrm{bco}}$, obtain BCO action $a_{\tau_j}^j\in\mathcal A_j$};
\node[orangebox] (o4) at (23.4,2.00) {%
Convert action $a_{\tau_j}$ to price $p_{s}=\mathfrak q^j(\tilde u_{s}; a_{\tau_j}^j)$,\\
post $p_s$};
\draw[->, thick] (b2.east) -- (o2.west);
\draw[->, thick, myorange] (o2.east) -- (o3.west);
\draw[->, thick, myorange] (o3.east) -- (o4.west);
\end{tikzpicture}%
}
\caption{Overview of $\ORBIT$ under the pilot-estimator interface. At each global time $s$, the arriving context $\mc_s$ is compressed into a one-dimensional surrogate utility $\tilde u_s\in\mathcal U$ and assigned to a
bin $B_{j_s}$ in the partition of $\mathcal U$. Each bin $B_j$ maintains its own local clock $\tau_j$ and runs a two-phase learner: a coarse localization phase based on grid exploration, followed by a local refinement phase driven by a BCO oracle.
}
\label{fig:oplcb-flow}
\end{figure}

\subsection{Bins, price grid, and coarse localization}\label{subsec:bin}

Recall that $\cU=[u_{\min},u_{\max}]$ and $\Delta_{\cU}=u_{\max}-u_{\min}$.  Given a target bin width $h\in(0,1]$, $\ORBIT$ sets
\begin{align}\label{eq:bin-construction}
    M:=\lceil \Delta_{\cU}/h\rceil, \qquad \bar h:=\Delta_{\cU}/M,
\end{align}
and partitions $\cU$ into equal-width bins
\begin{align*}
    B_j=[u_{\min}+(j-1)\bar h,u_{\min}+j\bar h),\quad j<M, \qquad  B_M=[u_{\min}+(M-1)\bar h,u_{\max}].
\end{align*}
We denote the closure of $B_j$ as $\overline B_j=[u_{\min}+(j-1)\bar h,u_{\min}+j\bar h]$, and its midpoint as $\bar u_j=u_{\min}+(j-1/2)\bar h$. Since $\bar h\leq h$ and $M\leq C_{\cU}/h$ for a constant $C_{\cU}$ depending only on the length of $\cU$, all rates below are stated in terms of $h$.

The coarse phase employs the price grid
\begin{equation}\label{eq:grid-price-set}
\cG:=\{k\eta_{\mathrm{grid}}:k=0,1,\ldots,\floor{p_{\max}/\eta_{\mathrm{grid}}}\}\cup\{p_{\max}\},
\end{equation}
with duplicate elements removed and ordered increasingly. We also define the coarse localization scale $\rho_{\rm loc}:=\sqrt{\eta_{\mathrm{grid}}}$. Such scale is used first as the target accuracy for the coarse anchor. Later, after the anchor is constructed, the same scale determines the width of the local polynomial search region.

When a bin $j$ is visited, $\ORBIT$ first cycles through grid $\cG$ and samples each grid price for $m_{\coarse}:=\lceil m_0\log(eH)\rceil$ visits to that bin, where $H$ is the upper pilot-input budget supplied to Algorithm~\ref{alg:oplcb}. Let $\cT_j^{\coarse}$ be the set of $\ORBIT$ call indices assigned to bin $j$ during this coarse phase. For each $p\in\cG$, the empirical revenue mean in bin $j$ is denoted as $\widehat r_j(p)$. The first step in the analysis is a uniform concentration statement: despite the fact that a bin contains a range of true indices, its empirical grid means are close to the revenue curve at each index in the bin, up to the stochastic error and the deterministic bin/pilot bias. The probability bound below is written in terms of $H$, since the logarithmic schedule and union bounds use this upper budget as opposed to the eventual realized value of $N$.

\begin{lemma}\label{lem:mean-concentration}
Assume that Algorithm~\ref{alg:oplcb} is run with upper pilot-input budget $H$ on a realized stream of $N\leq H$ pilot states $\{\tilde u_s\}_{s=1}^N$, and the associated index--pilot sequence $\{(u_s,\tilde u_s)\}_{s=1}^N$ satisfies Assumption~\ref{assump:orbit-interface}. Then there exists a universal constant $C_{\rm mean}>0$ such that conditional on this associated index--pilot sequence, with probability at least $1-CM\abs{\cG}H^{-5}$, for each $j\in[M]$ that completes its coarse phase,
\begin{align}
    \label{eq:grid-mean-concentration} \abs{\widehat r_j(p) - r(u, p)} \leq C_{\rm mean}p_{\max}\big(m_0^{-1/2} + L_g(h+\eta)\big), \quad \forall p \in \cG,\ u\in \overline B_j .
\end{align}
\end{lemma}

Once the coarse phase in bin $j$ is complete, $\ORBIT$ stores the deterministic-tie-breaking grid maximizer
\begin{align*}
    \tilde p_j:=\min\bigl(\argmax_{p\in\cG}\widehat r_j(p)\bigr).
\end{align*}
This anchor is then fixed for the rest of the bin's lifetime. The desired localization event is
\begin{equation}\label{eq:coarse-event}
    \cE_{j,\mathsf{coarse}} := \left\{\abs{\tilde{p}_j - p^\ast(u)} \leq \rho_{\rm loc}/8,\quad \forall u \in \overline B_j \right\} .
\end{equation}

\begin{lemma}\label{lem:main-upper-coarse-compressed}
Assume that there exists an absolute constant $c_0\in(0,1)$ such that 
\begin{align*}
    \eta_{\rm grid} \leq c_0\min\bigg\{1,\, \rho_0^2,\, \frac{\sigma_r}{L_r}\bigg\},\quad 
    h,\eta \leq \frac{c_0\sigma_r\rho_{\rm loc}^{2}}{L_g p_{\max}},\quad
    m_0 \geq \frac{p_{\max}^{2}}{c_0\sigma_r^{2}\rho_{\rm loc}^{4}},
\end{align*}
where $\rho_{\rm loc}:=\sqrt{\eta_{\rm grid}}$. Then under~\eqref{eq:grid-mean-concentration},  event $\cE_{j,\coarse}$ holds for each bin $j$ that has completed its coarse phase; in particular, it holds for each bin that enters refinement.
\end{lemma}

\subsection{Local polynomial refinement}\label{subsec:refinement}

Only after a bin has completed coarse localization do we define its local search region. The anchor $\tilde p_j$ identifies the price level around which refinement is safe. The local refinement step then learns how the oracle price varies with the scalar index inside the bin.

For $q := \floor{\beta-1}$, define the degree-$q$ monomial basis
\begin{align*}
\psi(z) := (1,z,\dots,z^q)^\top \in \R^{q+1}, \qquad z\in[-1,1].
\end{align*}
When $\beta-1$ is an integer, this degree is one higher than the Taylor degree used by the comparator in Lemma~\ref{lem:local-policy-approximation}; the extra coefficient is harmless and set to zero in the approximation argument.
For each bin $j$ and coefficient vector $a\in\R^{q+1}$, let us define the local polynomial price map
\begin{equation}\label{eq:polynomial-price}
\frakq^j(u;a) = a^\top \psi\left(\frac{2(u-\bar{u}_j)}{\bar h}\right), \qquad u \in \overline B_j.
\end{equation}
When bin $j$ first enters refinement, $\ORBIT$ constructs the coefficient set
\begin{align}\label{eq:trust-region-main}
\cA_j := \Big\{a \in \R^{q+1} : \sup_{\abs{z}\leq 1}\abs{a^\top\psi(z)-\tilde p_j}\leq \rho_{\rm loc}/4\Big\},
\qquad
\cQ_j^{\mathsf{local}} := \{\frakq^j(\cdot;a): a \in \cA_j\}.
\end{align}
We call $\cA_j$ the bin's trust region. It is centered at the coarse anchor $\tilde p_j$ and restricts all candidate polynomial prices to lie within $\rho_{\rm loc}/4$ of that anchor throughout the bin. Once the anchor is localized, this band keeps the refinement prices inside the local-concavity neighborhood from Lemma~\ref{lem:rho-exists}. Consequently, negative revenue is convex in price in the relevant local range. Since $a\mapsto \frakq^j(\tilde u;a)$ is linear, the coefficient loss $a\mapsto -r(u,\frakq^j(\tilde u;a))$ is convex over $\cA_j$. The same localization also makes the projection onto $[0,p_{\max}]$ inactive on the good event.

The same local polynomial class also has the right approximation power. The next lemma follows from the $(\beta-1)$-smoothness of $p^\ast$ established in Lemma~\ref{lem:pstar-regularity}.

\begin{lemma}\label{lem:local-policy-approximation}
Under event $\cE_{j,\coarse}$, assume that $\rho_{\rm loc}\leq\rho_0$ and $h^{\beta-1}\leq \rho_{\rm loc}/(8L_p)$.  Then there exist $a^\mathrm{cmp}_j \in \cA_j$ and constant $C_{\mathrm{approx}}>0$ depending only on $\beta,L_p,L_r$ such that for any $\tilde{u},u \in \cU$ with $\tilde{u} \in B_j$ and $\abs{u -\tilde{u}} \leq \eta\leq 1$,
\begin{align*}
    r(u, p^\ast(u)) - r\big({u} ,\frakq^j(\tilde{u};{a^\mathrm{cmp}_j})\big) \leq C_{\mathrm{approx}} \big(\eta^2 +h^{2\beta -2}\big).
\end{align*}
\end{lemma}

With the bin partition, price grid, and trust-region class now in place, Algorithm~\ref{alg:oplcb} provides the operational $\ORBIT$ policy. The algorithm keeps a separate local state for each bin. During coarse localization, this state consists of a visit counter and empirical revenue means over the price grid. When a bin first enters refinement, $\ORBIT$ stores the selected anchor price, constructs the trust region $\cA_j$, and initializes a fresh bin-specific raw refinement generator with an empty ordered action-feedback history. Such bin-specific generator is an online stateful procedure: after receiving $\cA_j$, it returns one coefficient vector at each later visit to bin $j$ based only on that bin's ordered raw history and fresh internal randomization. It is never given the eventual number $n_j$ of refinement visits. We thus use \emph{anytime} to mean a prefix-valid, stoppable generator at the bin level. The realized value of $n_j$ is used below only as an ex-post analytical horizon. Appendix~\ref{app:raw-bco-wrapper} then explains how to build such raw anytime generator from horizon-dependent normalized BCO primitives by normalization and the doubling trick. Consequently, the refinement stage remains local to each bin, and each active bin can be analyzed as a separate coefficient-learning problem over its own trust region.

Operationally, Algorithm~\ref{alg:oplcb} is a stateful pricing module that owns the pricing and update step on each $\ORBIT$ call. On an outer round $t$, after the outer policy observes $\mc_t$ and computes a scalar pilot $\tilde u_t$, a call to $\ORBIT$ executes the appropriate branch of Algorithm~\ref{alg:oplcb}: the $\ORBIT$-selected price is posted, the resulting purchase outcome is fed back to the same $\ORBIT$ copy, and only $\ORBIT$'s local bin state is updated. In Algorithm~\ref{alg:oplcb}, index $s$ represents the re-indexed count of calls to this module. No $\ORBIT$-phase price, outcome, or internal refinement randomization is used by the pilot estimators in the constructions below.

\begin{algorithm}[t]
\caption{\textbf{OR}acle-price learning with \textbf{B}inning and trust-reg\textbf{I}on refinemen\textbf{T} ($\ORBIT$)}\label{alg:oplcb}
\begin{algorithmic}[1]
\State \textbf{Input:} structural interval $\cU$, smoothness order $\beta$, target bin width $h$, upper pilot-input budget $H$, coarse sampling constant $m_0>0$, and grid spacing $\eta_{\mathrm{grid}}>0$.
\Statex \hspace{\algorithmicindent}Here, $H$ is an upper bound used only for logarithmic tuning; the subroutine may be stopped after any number of scalar-pilot calls.
\Statex \hspace{\algorithmicindent}\textbf{Refinement generator:} $\ORBIT$-compatible raw anytime refinement generator $\cB_{\rm raw}^\bco(H)$; each bin receives a fresh copy that is not told its eventual number of refinement visits.
\State Construct bins $\{B_j\}_{j=1}^M$ by~\eqref{eq:bin-construction}, price grid $\cG$ by~\eqref{eq:grid-price-set}, localization scale $\rho_{\rm loc}:=\sqrt{\eta_{\mathrm{grid}}}$, and $m_{\coarse}:=\lceil m_0\log(eH)\rceil$.
\State Initialize, for each $j\in[M]$, the local variables
\begin{align*}
    \cD_j^\mathsf{bco}=\varnothing,\quad \tau_j=0,\quad N_j(p)=\widehat r_j(p)=0\quad \forall p\in\cG.
\end{align*}
\State \textbf{Subroutine input (arriving over time):} scalar pilot states generated by an outer policy after observing contexts.
\For{each scalar-pilot call $s=1,2,\ldots$ up to $H$ calls, before the subroutine is stopped}
    \State Receive the scalar pilot state $\tilde u_s$.
    \State Let $j_s$ be the unique index such that $\tilde u_s\in B_{j_s}$; set $\tau_{j_s}\leftarrow \tau_{j_s}+1$.
    \If{$\tau_{j_s}\leq \abs{\cG}m_{\coarse}$} \Comment{Coarse localization phase in bin $j_s$.}
        \State Set $m:=1+\lfloor(\tau_{j_s}-1)/m_{\coarse}\rfloor$; post $p_s:=$ the $m$th element of $\cG$ and observe $y_s$.
        \State Update $N_{j_s}(p_s)\leftarrow N_{j_s}(p_s)+1$ and
        \begin{align*}
        \widehat r_{j_s}(p_s)\leftarrow
        \frac{1}{N_{j_s}(p_s)}p_s y_s+\frac{N_{j_s}(p_s)-1}{N_{j_s}(p_s)}\widehat r_{j_s}(p_s).
        \end{align*}
    \Else \Comment{Local refinement phase in bin $j_s$.}
        \State Set $k_s:=\tau_{j_s}-\abs{\cG}m_{\coarse}$ for the local refinement counter in bin $j_s$.
        \If{$k_s=1$}
            \State Set and store $\tilde p_{j_s}:=\min\bigl(\argmax_{p\in\cG}\widehat r_{j_s}(p)\bigr)$ with deterministic tie-breaking; construct and store $\cA_{j_s}$ according to~\eqref{eq:trust-region-main}; initialize a fresh copy $\cB_{j_s}^{\bco}$ of $\cB_{\rm raw}^\bco(H)$ on action set $\cA_{j_s}$ with fresh internal randomization.
        \EndIf
        \State Feed $\cD_{j_s}^{\bco}$ to the bin-specific generator copy $\cB_{j_s}^{\bco}$ to obtain action $a_{k_s}^{j_s}\in\cA_{j_s}$.
        \State Let $\bar p_s:=\frakq^{j_s}(\tilde u_s;a_{k_s}^{j_s})$ according to~\eqref{eq:polynomial-price}; post $p_s:=\proj_{[0,p_{\max}]}(\bar p_s)$ and observe $y_s$.
        \State Construct the raw feedback $\ell_{k_s}^{j_s}:=-p_s y_s$ and update the ordered history by appending the new pair: $\cD_{j_s}^{\bco}\leftarrow\big(\cD_{j_s}^{\bco},(a_{k_s}^{j_s},\ell_{k_s}^{j_s})\big)$.
    \EndIf
\EndFor
\end{algorithmic}
\end{algorithm}

\subsection{Approximation and raw refinement analysis}\label{subsec:refinement-analysis}

For each bin $j$, let $\cT_j^{\refine} = \{s_{j,1}<\cdots<s_{j,n_j}\}\subset[N]$ be the random $\ORBIT$-call index set of the refinement phase in this bin; if the bin never enters refinement, set $\cT_j^{\refine}=\varnothing$ and $n_j=0$. On the good event where projection is inactive, any binwise refinement policy that chooses coefficients $a_s\in\cA_j$ satisfies the following learning--approximation decomposition
\begin{equation}\label{eq:refined-regret-decomposition}
\begin{aligned}
    \cR^\refine_j &:= \sum\nolimits_{s \in \cT_j^\refine} r(u_s,p^\ast(u_s)) - r(u_s,\frakq^j(\tilde{u}_s;a_s)) \\
    &\leq \underbrace{ \min_{a\in \cA_j}\sum\nolimits_{s \in \cT_j^\refine} r(u_s,p^\ast(u_s)) - r(u_s,\frakq^j(\tilde{u}_s;a))}_{:= \cR_j^\mathrm{approx}}\\
    &\quad  +\underbrace{\max_{a\in \cA_j}\sum\nolimits_{s \in \cT_j^\refine} r(u_s,\frakq^j(\tilde{u}_s; a)) - r(u_s,\frakq^j(\tilde{u}_s;a_s))}_{:= \cR_j^\mathrm{learn}}.
\end{aligned}
\end{equation}
The approximation term above can be controlled by Lemma~\ref{lem:local-policy-approximation}.

\begin{proposition}\label{prop:regret-approx-j} Assume that Algorithm~\ref{alg:oplcb} is run with upper pilot-input budget $H$ on a realized stream of $N\leq H$ pilot states $\{\tilde u_s\}_{s=1}^N$, and the associated index--pilot sequence $\{(u_s,\tilde u_s)\}_{s=1}^N$ satisfies Assumption~\ref{assump:orbit-interface}. Then under event $\cE_{j,\coarse}$ and the conditions of Lemma~\ref{lem:local-policy-approximation}, we have that for each bin $j$ with nonempty $\cT_{j}^\refine$, 
\begin{align*}
    \cR^\mathrm{approx}_j \leq C_{\mathrm{approx}} n_j \big(\eta^2+ h^{2\beta-2}\big),
\end{align*}
where $C_{\mathrm{approx}}>0$ is the constant from Lemma~\ref{lem:local-policy-approximation}.
\end{proposition}

\textit{Proof of Proposition~\ref{prop:regret-approx-j}}. 
On event $\cE_{j,\coarse}$, Lemma~\ref{lem:local-policy-approximation} gives a comparator $a_j^{\rm cmp}\in\cA_j$ whose per-round approximation loss is at most $C_{\rm approx}(\eta^2+h^{2\beta-2})$ for each refinement round in bin $j$. Since $\cR_j^{\rm approx}$ is the minimum over $a\in\cA_j$, it is no larger than the value at this comparator. Summing over the $n_j$ refinement rounds establishes the claim. This completes the proof of Proposition~\ref{prop:regret-approx-j}.

It remains to control $\cR_j^{\mathrm{learn}}$. After coarse localization, each bin can be viewed as a raw convex-bandit problem over the coefficient vector. At the $k$th refinement visit of bin $j$, the raw feedback passed to the refinement generator is $\ell_k^j:=-p_{s_{j,k}}y_{s_{j,k}}$. On the good coarse-localization event, the projection in Algorithm~\ref{alg:oplcb} is inactive, so we have that $p_{s_{j,k}}=\frakq^j(\tilde u_{s_{j,k}};a_k^j)$.

\subsubsection{The induced raw refinement problem.}
For clarity, we use the following generic convex-bandit notation only to express the localized coefficient-learning problem induced by a bin.

\begin{definition}[BCO instance]\label{def:bco-instance}
Fix an ambient dimensionality $d_b$, a background $\sigma$-algebra $\cH$, a compact convex action set $\cA\subset\R^{d_b}$ with nonempty interior, and a sequence of bounded convex loss functions $\{L_k\}_{k=1}^n$, where each $L_k:\cA\to\R$ is $\cH$-measurable. The associated $n$-step BCO instance is the following interaction between a learner and an environment.

Starting from the empty ordered history $\cD_0^{\bco}=\varnothing$, at each step $k\in[n]$:
\begin{enumerate}
    \item[1)] The learner selects an action $a_k\in\cA$ based on the action set, the ordered history $\cD_{k-1}^{\bco}$, and fresh internal randomization independent of the environment. The background $\sigma$-algebra $\cH$ is used to define the conditional loss sequence and the unbiasedness requirement; it is not assumed to reveal the loss functions to the learner.
    \item[2)] The environment incurs the loss $L_k(a_k)$.
    \item[3)] The learner observes a noisy bandit feedback $\ell_k$ satisfying that 
    \begin{equation}\label{eq:bco-noise-requirement}
    \E\left[\ell_k \mid \sigma\bigl(\cH,a_1,\ell_1,\dots,a_{k-1},\ell_{k-1},a_k\bigr)\right] = L_k(a_k),
    \end{equation}
    and updates the ordered history to $\cD_k^{\bco}:=((a_1,\ell_1),\ldots,(a_k,\ell_k))$.
\end{enumerate}
The regret of a BCO policy $\pi^{\bco}$ over this instance is given by 
\begin{align*}
\cR^{\bco}(n) := \sum_{k=1}^n L_k(a_k) - \min_{a\in\cA}\sum_{k=1}^n L_k(a).
\end{align*}
\end{definition}

The following $\sigma$-algebra is an analytical conditioning device; it is not information revealed to the $\ORBIT$ learner. For a refinement bin $j$ with $\cT_j^{\refine}=\{s_{j,1}<\cdots<s_{j,n_j}\}$, let us define
\begin{equation}\label{eq:def-Hj}
\cH_j:= \sigma\Big( \{(u_s,\tilde u_s)\}_{s=1}^N, \{(p_s,y_s)\}_{s\in \cT_j^{\mathrm{coarse}}} \Big)
\end{equation}
and for each $k\in[n_j]$,
\begin{equation}\label{eq:induced-bco-loss}
L_k^j(a) := -r\bigl(u_{s_{j,k}},\frakq^j(\tilde u_{s_{j,k}};a)\bigr),
\qquad a\in\cA_j.
\end{equation}
By construction, $\cA_j$ is compact, convex, and has nonempty interior. On the good coarse-localization event, the trust region keeps the candidate polynomial prices in the local-concavity neighborhood of Lemma~\ref{lem:rho-exists}, so the projection step is inactive and each $L_k^j$ is convex in the coefficient vector. The sequence $\{L_k^j\}_{k=1}^{n_j}$ is $\cH_j$-measurable by definition, and the raw feedback $\ell_k^j$ is conditionally unbiased since the action is chosen before the current demand noise is observed and that noise is independent of the conditioning history.

\begin{proposition}[Binwise refinement as a raw BCO instance]\label{prop:induced-bco-equivalence}
Assume that Algorithm~\ref{alg:oplcb} is run with upper pilot-input budget $H$ on a realized stream of $N\leq H$ pilot states $\{\tilde u_s\}_{s=1}^N$, and the associated index--pilot sequence $\{(u_s,\tilde u_s)\}_{s=1}^N$ satisfies Assumption~\ref{assump:orbit-interface}. Assume further that event $\cE_{j,\coarse}$ holds for a bin $j$ with nonempty $\cT_j^{\refine}$, $ \rho_{\rm loc}\leq \rho_0$, and $L_p\eta\leq \rho_{\rm loc}/8$.
Then conditional on $\cH_j$ defined in~\eqref{eq:def-Hj}, the bin-$j$ refinement rounds form an ex-post $n_j$-step raw BCO instance with dimensionality $q+1$, action set $\cA_j$, and loss sequence $\{L_k^j\}_{k=1}^{n_j}$ defined in~\eqref{eq:induced-bco-loss}. The raw feedback $\ell_k^j=-p_{s_{j,k}}y_{s_{j,k}}$ is conditionally unbiased for $L_k^j(a_k^j)$, and under the identification $a\leftrightarrow\frakq^j(\cdot;a)$, we have that 
\begin{align*}
    \cR_j^{\mathrm{learn}}=\cR_j^{\bco}(n_j).
\end{align*}
In particular, the bin-specific raw refinement generator used by Algorithm~\ref{alg:oplcb}, viewed over its realized $n_j$ visits, is simply a policy for this ex-post induced instance. Here, we do not require the generator to know $n_j$ when it is initialized.
\end{proposition}

Proposition~\ref{prop:induced-bco-equivalence} above reduces the local refinement stage to a raw convex-bandit problem in the coefficient space used by $\ORBIT$. The horizon $n_j$ in the proposition is an ex-post analytical horizon, not an input to Algorithm~\ref{alg:oplcb} or the binwise generator copy. Such distinction matters since the number of future visits to a bin is unknown when its refinement generator is initialized. The main algorithm interacts with the binwise refinement generator only through coefficient vectors $a_k^j\in\cA_j$ and raw feedback $\ell_k^j=-p_{s_{j,k}}y_{s_{j,k}}$; model objects are used only in the analysis. The regret analysis thus uses the following \textit{raw anytime generator} property. Appendix~\ref{app:raw-bco-wrapper} provides a concrete implementation of this property.

\begin{definition}[$\ORBIT$-compatible raw anytime refinement generator]\label{def:raw-refinement-generator}
A family of binwise raw refinement generators $\{\cB_{\rm raw}^\bco(H_{\rm alg}):H_{\rm alg}\geq1\}$ is called $\ORBIT$-compatible if it satisfies the following raw anytime contract. Consider any induced bin instance satisfying the conditions of Proposition~\ref{prop:induced-bco-equivalence} and that the realized number of refinement visits obeys $n_j\leq H_{\rm alg}$. A fresh copy of $\cB_{\rm raw}^\bco(H_{\rm alg})$ is initialized on the raw action set $\cA_j$ without knowing $n_j$. At each visit it is queried with its own ordered local history $\cD_{j,k-1}^{\bco}$ and updated only with the raw feedback $\ell_k^j=-p_{s_{j,k}}y_{s_{j,k}}$. The resulting actions are adapted to this local history and fresh internal randomization.

Further, for each $\sigma$-algebra $\cH$ with respect to which the induced losses $\{L_k^j\}_{k\leq n_j}$ and the realized visit count $n_j$ are fixed, and 
the future bin-$j$ feedback remains conditionally unbiased, we have that 
\begin{align*}
    \E\left[\cR_j^{\rm learn}\mid \cH\right] \leq C_{\bco}\sqrt{n_j}\,\mathrm{polylog}(H_{\rm alg},\beta).
\end{align*}
Equivalently, for each $\cH$-measurable event $\mathcal G$ on which the conditions of Proposition~\ref{prop:induced-bco-equivalence} are satisfied, it holds that 
\begin{align*}
    \E\left[\mathbf 1_{\mathcal G}\cR_j^{\rm learn}\mid \cH\right] \leq \mathbf 1_{\mathcal G}C_{\bco}\sqrt{n_j}\,\mathrm{polylog}(H_{\rm alg},\beta).
\end{align*}
The constant $C_{\bco}$ may depend on fixed structural constants and 
the chosen refinement-generator family, but not on $H_{\rm alg}$, $j$, or the realized horizon $n_j$. The global parameter $H_{\rm alg}$ may be used for harmless logarithmic tuning; it is not the binwise horizon.
\end{definition}

\paragraph{Implementation of the refinement generator.}
In our implementation, we exploit a simple combination of the doubling trick and the online Newton method of \citet{fokkema2024online}, which yields $\widetilde{\cO}(\mathrm{poly}(q)\sqrt{n})$ anytime regret for each binwise coefficient-learning problem. Appendix~\ref{app:lg-theorem} presents the detailed BCO oracle guarantee, and Appendix~\ref{app:proof-of-prop:scb} gives a short proof of the resulting anytime bound for completeness.

Such raw anytime contract is the only refinement property used below. $\ORBIT$ supplies the raw coefficient set and raw bandit feedback, and the generator returns the next coefficient vector in the same raw coordinate system. Appendix~\ref{app:raw-bco-wrapper} shows that the contract is implemented by an affine normalization of the raw coefficient set, a shift-and-scale of the bounded feedback, and a doubling wrapper around a horizon-dependent normalized BCO routine. Hence, the proof below is modular: it applies to any generator satisfying Definition~\ref{def:raw-refinement-generator} above; the implementation used in this paper is described in Appendix \ref{app:proof-of-prop:scb}.

\begin{lemma}\label{lem:main-upper-refine}
Assume that Algorithm~\ref{alg:oplcb} is run with upper pilot-input budget $H$ on a realized stream of $N\leq H$ pilot states $\{\tilde u_s\}_{s=1}^N$, and the associated index--pilot sequence $\{(u_s,\tilde u_s)\}_{s=1}^N$ satisfies Assumption~\ref{assump:orbit-interface}. Assume further that $\rho_{\rm loc}\leq\rho_0$ and $L_p\eta\leq\rho_{\rm loc}/8$, and use $\cB_{\rm raw}^\bco(H)$ from an $\ORBIT$-compatible raw anytime refinement family in the sense of Definition~\ref{def:raw-refinement-generator}. Then we have that for each bin $j$ with $n_j>0$,
\begin{align*}
\E\Big[\bm{1}_{\cE_{j,\mathsf{coarse}}}\,\cR_j^\mathrm{learn}\,\Big\lvert\, \{(u_s,\tilde u_s)\}_{s=1}^N \Big] \leq C_{\bco}\sqrt{n_j}\,\mathrm{polylog}(H,\beta).
\end{align*}
More generally, let $\cH_{\coarse}^{\rm all}$ be the $\sigma$-algebra generated by the associated index--pilot sequence and all coarse-phase observations from all bins. If $\mathcal G$ is any $\cH_{\coarse}^{\rm all}$-measurable event satisfying that $\mathcal G\subseteq\cE_{j,\coarse}$, it holds that 
\begin{align*}
    \E\Big[\bm{1}_{\mathcal G}\,\cR_j^\mathrm{learn}\,\Big\lvert\, \{(u_s,\tilde u_s)\}_{s=1}^N \Big] \leq C_{\bco}\sqrt{n_j}\,\mathrm{polylog}(H,\beta).
\end{align*}
For bins with $n_j=0$, we set $\cR_j^{\rm learn}=0$, so the same bound is trivial. Here, constant $C_{\bco}$ depends only on the fixed structural constants, and the expectation is taken with respect to the pricing, refinement-generator, and demand randomness conditional on the associated index--pilot sequence.
\end{lemma}

\subsection{Conditional $\ORBIT$ regret guarantee}\label{subsec:conditional-orbit-guarantee}

The conditional regret bound follows by summing the coarse cost, the raw refinement learning cost, and the local polynomial approximation cost.

\begin{theorem}[Conditional $\ORBIT$ guarantee]\label{thm:upper}
Assume that Assumptions~\ref{assump:bounded}--\ref{assump:rho} hold, and let $\rho_0$ be the local-concavity radius in Lemma~\ref{lem:rho-exists}. Run Algorithm~\ref{alg:oplcb} with upper pilot-input budget $H$ and generator $\cB_{\rm raw}^\bco(H)$ from an $\ORBIT$-compatible raw anytime refinement family in the sense of Definition~\ref{def:raw-refinement-generator}. Assume further that the subroutine is stopped after $N\leq H$ pilot inputs, and the resulting index--pilot sequence $\{(u_s,\tilde u_s)\}_{s=1}^N$ satisfies Assumption~\ref{assump:orbit-interface}. Let $\eta_{\rm grid}>0$ be a fixed grid spacing independent of $H$, and set $\rho_{\rm loc}:=\sqrt{\eta_{\rm grid}}$. Then there exists an absolute constant $c_0\in(0,1)$ such that if
\begin{align*}
    \eta_{\rm grid} \leq c_0\min\{1,\rho_0^2,\sigma_r/L_r\},\quad 
    h,\eta \leq \frac{c_0\sigma_r\rho_{\rm loc}^2}{L_g p_{\max}},\quad
    h^{\beta-1} \leq \frac{\rho_{\rm loc}}{8L_p}, \\
    L_p\eta\leq \frac{\rho_{\rm loc}}{8},\quad
    m_{\coarse}=\lceil m_0\log(eH)\rceil, \quad m_0\geq \frac{p_{\max}^2}{c_0\sigma_r^2\rho_{\rm loc}^4},
\end{align*}
Algorithm~\ref{alg:oplcb}, run with localization scale $\rho_{\rm loc}$ and trust-region half-width $\rho_{\rm loc}/4$, satisfies that 
\begin{align}\label{eq:orbit-thm-decompose}
     &\E\bigg[\sum_{s=1}^N \big[ r(u_s,p^\ast(u_s))-r(u_s,p_s)\big] \bigg\lvert \{(u_s,\tilde u_s)\}_{s=1}^N \bigg] \notag\\
     &\quad\leq C\frac{\abs{\cG}m_0\log(eH)}{h} +C_{\bco}\sqrt{\frac{N}{h}}\,\mathrm{polylog}(H,\beta) +C N\big(h^{2\beta-2}+\eta^2\big),
\end{align}
where constant $C$ depends only on the structural model constants and 
$\beta$, and $C_{\bco}$ is the constant in the raw anytime refinement guarantee. Since $\eta_{\rm grid}$ and $m_0$ are fixed structural choices, the first term is $\widetilde \cO(1/h)$. In particular, if $H=T$ and $N\leq T$, provided that the pilot accuracy $\eta$ satisfies the displayed smallness conditions, the choice of $h:=T^{-1/(4\beta-3)}$ yields that 
\begin{align*}
    \E\bigg[\sum_{s=1}^N \big[ r(u_s,p^\ast(u_s))-r(u_s,p_s)\big] \bigg\lvert \{(u_s,\tilde u_s)\}_{s=1}^N \bigg] =\widetilde{\cO}\Big(T^{\frac{2\beta-1}{4\beta-3}} + T\eta^2\Big).
\end{align*}
\end{theorem}

Theorem \ref{thm:upper} above separates the realized number $N$ of scalar-pilot calls from the upper budget $H$. The generator may use $H$ for logarithmic tuning, but no bin-specific copy is given its future visit count $n_j$. Since $\rho_0>0$, the displayed parameter conditions are feasible for each fixed structural instance: one may choose a sufficiently small fixed $\eta_{\rm grid}$, set $\rho_{\rm loc}=\sqrt{\eta_{\rm grid}}$, and then choose $m_0$ sufficiently large. This is an instance-wise tuning statement. The algorithm receives the pilot states and the numerical grid and sampling constants; the theorem asserts the regret guarantee when those numerical choices satisfy the structural inequalities. Uniform tuning over a whole model class requires known class-level bounds such as $\rho_0\geq \underline\rho>0$, $\sigma_r\geq \underline\sigma>0$, and upper bounds on $L_g,L_r,L_p,p_{\max}$, and 
the length of $\cU$.

The three terms on the right-hand side of~\eqref{eq:orbit-thm-decompose} above correspond directly to the three parts of the policy
\begin{align}\label{eq:orbit-detailed-decompose}
\underbrace{h^{-1}}_{\text{coarse localization}} + \underbrace{\sqrt{N/h}}_{\text{local learning}} + \underbrace{N (h^{2\beta-2} + \eta^2)}_{\text{local approximation}},
\end{align}
up to logarithmic and fixed structural-grid factors. The term $\widetilde{\cO}(h^{-1})$ is the total cost of the coarse localization. The term $\widetilde{\cO}(\sqrt{N/h})$ represents the aggregate learning cost of the local raw refinement generators. The term $Nh^{2\beta-2}$ stems from the approximation cost of replacing the oracle price map inside each bin with a degree-$q$ polynomial, and $N\eta^2$ is the cost induced by the pilot error. We are now ready to prove Theorem~\ref{thm:upper}. 

\textit{Proof of Theorem~\ref{thm:upper}}. 
Let us fix the associated index--pilot sequence $\{(u_s,\tilde u_s)\}_{s=1}^N$. Note that the total number of coarse pulls is at most $\sum_{j=1}^M \abs{\cT_j^{\coarse}} \leq M\abs{\cG}m_{\coarse}$.
Since Algorithm~\ref{alg:oplcb} always posts prices in $[0,p_{\max}]$, the regret per $\ORBIT$ call is at most $p_{\max}$. It follows that 
\begin{align}
    \label{eq:proof-thm-upper-coarse} R_{\coarse} \leq p_{\max}M\abs{\cG}m_{\coarse} \leq C\frac{\abs{\cG}m_0\log(eH)}{h}.
\end{align}
Denote by  $\cE_{\coarse}:=\cap_{j:\,\cT_j^{\refine}\ne\varnothing}\cE_{j,\coarse}$. By invoking Lemmas~\ref{lem:mean-concentration} and~\ref{lem:main-upper-coarse-compressed}, it holds that $\Prob\big(\cE_{\coarse}^c\mid \{(u_s,\tilde u_s)\}_{s=1}^N\big) \leq CM\abs{\cG}H^{-5}$.
On event $\cE_{\coarse}^c$, the projection step still guarantees prices in $[0,p_{\max}]$, so the failure contribution is at most
\begin{align*}
    p_{\max}N\cdot CM\abs{\cG}H^{-5} \leq C\frac{\abs{\cG}}{h}H^{-4},
\end{align*}
where we have used $N\leq H$ and $M\leq C_{\cU}/h$. This term is dominated by the coarse-localization term in~\eqref{eq:proof-thm-upper-coarse} for all large $H$.

On event $\cE_{\coarse}$, the refinement regret decomposes binwise according to~\eqref{eq:refined-regret-decomposition}. For each active bin $j$, event $\cE_{\coarse}$ is $\cH_{\coarse}^{\rm all}$-measurable and satisfies that $\cE_{\coarse}\subseteq\cE_{j,\coarse}$. Then an application of Proposition~\ref{prop:regret-approx-j} and Lemma~\ref{lem:main-upper-refine} with $\mathcal G=\cE_{\coarse}$ for that bin yields that 
\begin{align*}
    &\E\Big[\bm 1_{\cE_{\coarse}}\sum_{j=1}^M \cR_j^{\refine}\,\Big\lvert\, \{(u_s,\tilde u_s)\}_{s=1}^N\Big]  \leq C\sum_{j=1}^M n_j(h^{2\beta-2}+\eta^2) +C_{\bco}\mathrm{polylog}(H,\beta)\sum_{j=1}^M\sqrt{n_j}\\
    &\quad \leq CN(h^{2\beta-2}+\eta^2) +C_{\bco}\mathrm{polylog}(H,\beta)\sqrt{M\sum_{j=1}^M n_j} \leq CN(h^{2\beta-2}+\eta^2) +C_{\bco}\sqrt{\frac{N}{h}}\mathrm{polylog}(H,\beta),
\end{align*}
where we have used $\sum_j n_j\leq N$ and $M\leq C_{\cU}/h$. Therefore, combining this estimate with~\eqref{eq:proof-thm-upper-coarse} establishes~\eqref{eq:orbit-thm-decompose}. Further, with $H=T$ and $N\leq T$, the choice of $h=T^{-1/(4\beta-3)}$ balances the upper bounds on $\sqrt{N/h}$ and $Nh^{2\beta-2}$ by the corresponding terms with $N$ replaced by $T$, and it satisfies the required smallness conditions for all sufficiently large $T$. This concludes the proof of Theorem~\ref{thm:upper}.

\section{The linear utility model: adaptive pilot and minimax optimality}\label{sec:linear-main}
 
\subsection{Baseline linear utility model}\label{subsec:baseline-linear}
 
We now specialize the scalar-index framework to the baseline linear utility model. This is the main model employed for the fully online upper bound and the matching lower bound, and is the most widely adopted utility model in the contextual pricing literature with binary feedback~\citep{fan2024policy,wang2025tight,han2026general,tullii2024improved,luo2024distribution}.
 
\begin{assumption}[Linear utility]\label{assump:linear}
There exists a constant $C_\theta>0$ such that
\begin{enumerate}
    \item[1)] $\cC\subset \mathbb{B}^d_2(1)$;
    \item[2)] $\theta_\ast\in \mathbb{B}^d_2(C_\theta)$ and $\mc^\top\theta_\ast\geq0$ for each $\mc\in\cC$;
    \item[3)] $\mu_\ast(\mc)=\mc^\top\theta_\ast$.
\end{enumerate}
\end{assumption}
 
Assumption~\ref{assump:linear} above implies Assumption~\ref{assump:bounded}; one may always take the conservative structural interval $\cU=[0,C_\theta]$, while a smaller compact interval containing the actual image of $\mu_\ast$ may also be used when it is part of the instance description. The nonnegativity condition ensures that the deterministic utility component is nonnegative. It does not by itself imply boundedness of the realized valuation $v_t=\mc_t^\top\theta_\ast+\xi_t$; the bounded realized-valuation normalization used by uniform-price pilots is stated in the next subsection. A canonical pair of $\cC$ and $\theta_\ast$ satisfying this assumption is given by 
\begin{align*}
    \cC=\mathbb{B}^{d-1}_2(1/\sqrt2)\times\{1/\sqrt2\},
    \qquad
    \theta_\ast=(\bar\theta_\ast,C_\theta/\sqrt2)
\end{align*}
for any $\bar\theta_\ast\in\mathbb{B}^{d-1}_2(C_\theta/\sqrt2)$.

\subsection{Adaptive construction of the scalar pilot}\label{subsec:adaptive-linear-pilot}

The conditional guarantee in Theorem~\ref{thm:upper} becomes a fully online policy once we can construct scalar pilots satisfying Assumption~\ref{assump:orbit-interface}. For the low-dimensional linear model, we use the adaptive elliptical exploration.
 
Observe that if $P\sim\mathrm{Unif}[0,p_{\max}]$ is independent of $(\mc,\xi)$, we have that 
\begin{align*}
    \E\big[p_{\max}\bm 1\{\mc^\top\theta_\ast+\xi\geq P\}\mid \mc\big]=\E[\mc^\top\theta_\ast+\xi\mid\mc]=\mc^\top\theta_\ast
\end{align*}
since $\Prob(P\leq v\mid v)=v/p_{\max}$ for $v\in[0,p_{\max}]$ and $\E[\xi]=0$. For an exploration time $\tau$, the pseudo-response $p_{\max}y_\tau$ is thus an unbiased observation of $\mc_\tau^\top\theta_\ast$. Given the exploration set $\cT_{t-1}^{\mathsf{exp}}$ before round $t$, let us define
\begin{align}\label{eq:linear-ridge-regression}
    \bA_t:=\bI+\sum_{\tau\in\cT_{t-1}^{\mathsf{exp}}}\mc_\tau\mc_\tau^\top,
    \qquad
    \hat\theta_t:=\bA_t^{-1}\sum_{\tau\in\cT_{t-1}^{\mathsf{exp}}}p_{\max}y_\tau\mc_\tau .
\end{align}
The uncertainty score
\begin{align*}
    w_t:=C_w\|\mc_t\|_{\bA_t^{-1}},
    \qquad
    C_w:=32(C_\theta+p_{\max})\sqrt{\log(eT)}
\end{align*}
is the standard elliptical confidence radius used in linear contextual bandits \citep{chu2011contextual}. It measures how well the current context is covered by the uniform-price samples collected so far.  If $w_t>\eta$, the current direction is not yet well covered by the previous uniform-price samples, and the policy explores by posting a fresh uniform
price. If $w_t\leq\eta$, the utility index is already estimated accurately enough for the target pilot precision, and the policy sends  $\tilde u_t=\proj_{\cU}(\mc_t^\top\hat\theta_t)$ to $\ORBIT$.  Such adaptive exploration rule follows the uncertainty-triggered design of \citet{tullii2024improved}, but its role here is different. We use it only to certify scalar utility-index estimates for $\ORBIT$. In this way, the policy does \textit{not} need the context distribution to cover all directions in advance. It explores a direction only when that direction appears and is still uncertain under the current design. Algorithm~\ref{alg:adaptive-init} provides the full procedure.
The lemma below records the resulting confidence guarantee.
\begin{lemma}\label{lem:linear-confidence}
Under Assumption~\ref{assump:linear}, Algorithm~\ref{alg:adaptive-init} satisfies that with probability at least $1-\cO(T^{-3})$,
\begin{align}
    \label{eq:linear-uncertainty-bound} \abs{\mc_t^\top(\hat{\theta}_t- \theta_\ast)} \leq w_t, \quad \forall t \in [T].
\end{align}
\end{lemma}

\begin{algorithm}[t]
\caption{Adaptive linear pilot coupled with $\ORBIT$}\label{alg:adaptive-init}
\begin{algorithmic}[1]
    \State \textbf{Input:} structural interval $\cU$, target pilot accuracy $\eta$, horizon $T$, and $\ORBIT$ parameters.
    \State Initialize $\cT^{\mathsf{exp}}_0=\varnothing$ and initialize $\ORBIT$ with upper pilot-input budget $T$.
    \For{$t=1,\dots,T$}
        \State Observe $\mc_t$; form $\bA_t$ and $\hat\theta_t$ by~\eqref{eq:linear-ridge-regression}; compute $w_t=C_w\|\mc_t\|_{\bA_t^{-1}}$.
        \If{$w_t>\eta$}
            \State Post $p_t\sim\mathrm{Unif}[0,p_{\max}]$, observe $y_t$, and set $\cT_t^{\mathsf{exp}}=\cT_{t-1}^{\mathsf{exp}}\cup\{t\}$.
        \Else
            \State Set $\cT_t^{\mathsf{exp}}=\cT_{t-1}^{\mathsf{exp}}$ and $\tilde u_t=\proj_{\cU}(\mc_t^\top\hat\theta_t)$.
            \State Call $\ORBIT$ with scalar pilot $\tilde u_t$; the $\ORBIT$-selected price is posted, the resulting $y_t$ is fed back to $\ORBIT$, and $\ORBIT$ updates its internal state.
        \EndIf
    \EndFor
\end{algorithmic}
\end{algorithm}
 
\begin{remark}[Calling convention for $\ORBIT$]\label{rem:calling-convention}
For precision, we record how Algorithm~\ref{alg:adaptive-init} invokes $\ORBIT$ (Algorithm~\ref{alg:oplcb}) on the local time axis introduced in Section~\ref{sec:oplcb}. ``Initialize $\ORBIT$'' in line~2 of Algorithm~\ref{alg:adaptive-init} executes only the initialization block of Algorithm~\ref{alg:oplcb} and does not enter its main loop. ``Call $\ORBIT$ with scalar pilot $\tilde u_t$'' in line~9 executes one iteration of $\ORBIT$'s main loop and then exits, retaining all internal state for the next invocation. With such convention, each non-exploration round of Algorithm~\ref{alg:adaptive-init} corresponds to one increment of $\ORBIT$'s local clock $s$, and exploration rounds do not advance $s$.
\end{remark}

The procedure above is not new by itself: a uniform-pricing exploration phase gated by a context-based uncertainty score was introduced first by~\citet{tullii2024improved}, where it is paired with a successive-elimination subroutine in the $\beta=1$ regime. What is \textit{new} here is that the same adaptive scheme can be used as a pilot generator for $\ORBIT$ in the smooth regime $\beta\geq2$. This works since $\ORBIT$ places only interface-level requirements on the pilot sequence (Assumption~\ref{assump:orbit-interface}), which tolerates non-stationary distribution of the input pilot sequence due to adaptive exploration. In contrast, the stationary subroutines of~\citet{wang2025tight,han2026general} require the pilot sequence to satisfy \textit{stronger} stationarity assumptions and therefore \textit{cannot} be coupled with adaptive pilot procedures of this form.

We now record the three consequences needed to apply Theorem~\ref{thm:upper}: the independence of the associated index--pilot sequence from the $\ORBIT$-round demand noises, the pilot accuracy on non-exploration rounds, and a bound on the number of exploration rounds.

\begin{proposition}\label{prop:adaptive-init-guarantee}
Assume the linear utility model in Assumption~\ref{assump:linear}. Let $\cT^\ORBIT:=[T]\setminus\cT_T^{\mathsf{exp}}$. Then Algorithm~\ref{alg:adaptive-init} satisfies that with error level $0<\eta\leq1/2$,
\begin{enumerate}
    \item Assume further that $\cT^{\ORBIT} = \{t_1,\dots,t_N\}$ and the $\ORBIT$ subroutine is invoked with pilot states $\{\tilde u_{t_s}\}_{s=1}^N$. Then the associated index--pilot sequence
    $\{(u_{t_s},\tilde u_{t_s})\}_{s=1}^N$, with $u_{t_s}=\mc_{t_s}^\top\theta_\ast$ and $\tilde u_{t_s}=\proj_{\cU}(\mc_{t_s}^\top\hat\theta_{t_s})$, is independent of the $\ORBIT$ demand-noise sequence $\{\xi_{t_s}\}_{s=1}^N$ and 
    $\ORBIT$'s internal refinement randomization.
    \item With probability at least $1-\cO(T^{-3})$,
    \begin{align*}
    \abs{\tilde u_t-u_t}\leq\eta, \qquad \forall t\in\cT^\ORBIT .
\end{align*}
    \item $\abs{\cT_T^{\mathsf{exp}}}= \widetilde\cO(d\eta^{-2})$.
\end{enumerate}
\end{proposition}
 
On the high-probability event in Proposition~\ref{prop:adaptive-init-guarantee} above, the $\ORBIT$ subsequence satisfies Assumption~\ref{assump:orbit-interface}; the exploration rounds contribute only their count times a bounded per-round regret. Hence, combining this with Theorem~\ref{thm:upper} gives the fully online bound.

\begin{corollary}[Fully online regret guarantee]\label{cor:upper}
Assume that Assumptions~\ref{assump:bounded}--\ref{assump:rho} and \ref{assump:linear} hold. Run $\ORBIT$ with upper pilot-input budget $T$, generator $\cB_{\rm raw}^\bco(T)$ from an $\ORBIT$-compatible raw anytime refinement family, balanced bin width $h=T^{-1/(4\beta-3)}$, grid spacing $\eta_{\rm grid}$, localization scale $\rho_{\rm loc}:=\sqrt{\eta_{\rm grid}}$, and the remaining parameters satisfying Theorem~\ref{thm:upper}. Assume further that the target pilot accuracy satisfies that $0<\eta\leq \min\{ 1/2, c_0\sigma_r\rho_{\rm loc}^2 /L_g p_{\max} \}$ and $ L_p\eta\leq \rho_{\rm loc}/8$.
Then Algorithm~\ref{alg:adaptive-init}, coupled with $\ORBIT$, satisfies that 
\begin{align*}
\Regret(T) = \widetilde{\cO}\Big( T^{\frac{2\beta-1}{4\beta-3}} + T\eta^2+d\eta^{-2}\Big).
\end{align*}
In particular, if $\eta^2\asymp\sqrt{d/T}$ satisfies the displayed smallness conditions, we have that 
\begin{align*}
\Regret(T) = \widetilde{\cO}\Big(T^{\frac{2\beta-1}{4\beta-3}} + \sqrt{dT}\Big).
\end{align*}
\end{corollary}

\subsection{Matching lower bound for the linear model}\label{sec:lower}

We now turn to the lower bound. The construction is a hypercube of local perturbations. Each coordinate flips the sign of a smooth bump on one cell of the latent-index space, shifting the local oracle price on that cell while leaving the rest of the model unchanged. Any policy is thus forced to solve many local inference problems in parallel. To keep the statement compact, we use the following terminology.

\begin{definition}[Lower-bound-normalized linear instance]\label{def:lb-normalized-instance}
A pricing instance $\mathfrak I$ is called lower-bound normalized if it satisfies the following fixed one-dimensional conditions: $d=1$, $p_{\max}=1$, $C_\theta=1$, $\theta_\ast=1$, the utility map is $\mu_\ast(\mc)=\mc$, and the context space is contained in $[0,1]$. The noise has mean zero and for each context in the instance, the realized valuation $\mu_\ast(\mc)+\xi$ lies in $[0,p_{\max}]$ almost surely. Assumptions~\ref{assump:bounded}--\ref{assump:rho} hold on the scalar-index interval supplied with the instance, with structural constants bounded above and where applicable, bounded below by positive numerical constants that are independent of $T$. Further, the local-concavity radius in Lemma~\ref{lem:rho-exists} is at least $1/16$.
\end{definition}

The hard family constructed below consists entirely of lower-bound-normalized linear instances. Consequently, the lower bound applies inside the same class covered by the upper-bound analysis.

We adopt the following normalized constants in the construction
\begin{align*}
    C_{\rm loc}=\frac{1}{32},\qquad
    B:=1-C_{\rm loc}=\frac{31}{32},\qquad
    w:=\gamma T^{-1/(4\beta-3)}, \qquad
    M:=\floor{\frac{1}{64w}}
\end{align*}
where $\gamma>0$ is a sufficiently small numerical constant. The structural interval $\cU$ in Assumption~\ref{assump:bounded} is the fixed interval $\cU_{\rm lb}:=[\mu_0,\mu_0+C_{\rm loc}]$ in the centered construction below, where $\mu_0$ is the common centering constant defined shortly. It contains the scalar-index support of the hard instance, and giving this interval to the learner can only make the lower bound stronger. The constant $C_{\rm loc}$ controls only the size of the local coordinate grid used in the hard family. Let $g_0$ be a smooth auxiliary tail that agrees with the truncated-linear tail $z\mapsto 1-z/B$ on a fixed interior strip; see Appendix~\ref{app:sec5-proofs} for the detailed construction. We first describe the construction in an uncentered local coordinate $x$ and then center the noise distribution to respect the standing zero-mean normalization.

We set
\begin{align*}
    x_j:=2jw, \quad j \in [M],
\end{align*}
and let $\varphi\in C_0^\infty([-1/8,1/8])$ be an odd bump with $\varphi(0)=0$ and $\varphi'(0)=1$. We define the baseline oracle price in the local coordinate as
\begin{align*}
    p_0^\ast(x):=\frac{B+x}{2}, \quad p_j^{0} := p_0^{\ast}(x_j), \quad z_j:=p_j^0-x_j=\frac{B}{2}-jw.
\end{align*}
For each sign vector $\omega\in\{-1,+1\}^M$, let us define the auxiliary tail
\begin{align*}
g_\omega(z):=g_0(z)+\kappa w^\beta \sum_{j \in [M]} \omega_j \varphi\Big(\frac{z-z_j}{w}\Big).
\end{align*}
Since $\varphi$ is odd and each bump support is contained in the interior of $[0,B]$, it holds that $\int_0^B \varphi((z-z_j)/w)\d z=w\int\varphi=0$. Hence, once $\kappa$ is chosen small enough so that $g_\omega$ is a valid tail function, all auxiliary laws with tails $g_\omega$ have the same mean
\begin{align*}
    \mu_0=\int_0^B g_\omega(z)\d z,
\end{align*}
which does not depend on $\omega$. 

The actual hard instance uses the centered noise tail
\begin{align*}
    \bar g_\omega(z):=g_\omega(z+\mu_0),
\end{align*}
that is, the tail of $X-\mu_0$ when $X$ has auxiliary tail $g_\omega$. It also utilizes scalar contexts $\mc_j:=\mu_0+x_j$ with $\theta_\ast=1$ and $\Prob(\mc_t=\mc_j)=M^{-1}$. Then the noise has mean zero and
\begin{align*}
    p\,\bar g_\omega(p-\mc_j)=p\,g_\omega(p-x_j),
\end{align*}
so the actual centered instance has exactly the same revenue geometry as the auxiliary local-coordinate construction. For readability, the lower-bound analysis below writes $c_j$ for the local coordinate $x_j$, writes $c_t=\mc_t-\mu_0$ for the observed local coordinate, and works with the translated revenue $p g_\omega(p-c_j)$. Such deterministic reparametrization is common to all environments; equivalently, one may grant $\mu_0$ to the learner. Denote by $P_\omega$ the law of the full transcript under the corresponding centered instance, whose actual tail is $\bar g_\omega$.

\begin{theorem}[Lower bound]\label{thm:lower}
Fix $\beta\geq 2$. Then there exists a constant $c_\beta>0$ such that for all sufficiently large $T$ and each pricing policy $\pi$, one can find a lower-bound-normalized linear instance $\mathfrak I$ satisfying that 
\begin{align*}
    \Regret_\pi^{\mathfrak I}(T) \geq c_\beta T^{\frac{2\beta-1}{4\beta-3}}.
\end{align*}
Equivalently, the minimax regret over the lower-bound-normalized linear instances is at least this quantity.
\end{theorem}

The proof of Theorem~\ref{thm:lower} above is presented in Appendix~\ref{app:sec5-proofs}. The argument mirrors the local structure exploited by $\ORBIT$. Each latent-index cell carries one bit. Flipping that bit shifts the oracle price by order $w^{\beta-1}$, so an incorrect local decision costs order $w^{2\beta-2}$ per visit. A Kullback--Leibler (KL) divergence calculation reveals that little regret on a cell keeps the two paired environments statistically close, while large regret already gives the desired lower bound. Thus, aggregating the resulting local two-point bounds over $M\asymp 1/w$ cells and choosing $w\asymp T^{-1/(4\beta-3)}$ yield the exponent in Theorem~\ref{thm:lower}.

\section{Broader applications beyond linear utility}\label{sec:extensions}

Section~\ref{sec:linear-main} has established a regret guarantee for the $d$-dimensional linear utility model exploiting the ridge confidence ellipsoid in Lemma~\ref{lem:linear-confidence}. That construction is, however, specific to linear utilities. In this section, we will extend to the more general utility setting when the underlying utility function $\mu_\ast$ lies in some function class $\cF$. We introduce a more modular explore-first template for pilot estimation as in most previous works of \citet{han2026general,wang2025tight,fan2024policy,gong2025minimax,chen2024dynamic,luo2024distribution}. 
The seller first runs a randomized burn-in phase, collects pseudo-responses from uniform prices, passes the resulting offline dataset to an estimation oracle, and then freezes the returned utility estimate as the scalar pilot throughout the subsequent $\ORBIT$ phase, as detailed in Algorithm~\ref{alg:Init}.

\begin{algorithm}[t]
\caption{Explore-then-$\ORBIT$}\label{alg:Init}
\begin{algorithmic}[1]
    \State \textbf{Input:} structural interval $\cU$, exploration length $n_{\mathsf{exp}}$, estimation oracle $\cB^{\mathsf{est}}$, horizon $T$, and $\ORBIT$ parameters.
    \State Initialize dataset $\cD= \varnothing$.
    \For{$t=1,\dots,n_{\mathsf{exp}}$}
        \State Observe $\mc_t$, post $p_t\sim \mathrm{Unif}[0,p_{\max}]$, observe $y_t$, set $Z_t:=p_{\max}y_t$, and add $(\mc_t,Z_t)$ to $\cD$.
    \EndFor
    \State Compute $\widehat\mu:=\cB^{\mathsf{est}}(\cD)$ and freeze this estimator for all future rounds.
    \State Initialize $\ORBIT$ with upper pilot-input budget $H=T-n_{\mathsf{exp}}$ and the remaining $\ORBIT$ parameters.
    \For{$t=n_{\mathsf{exp}}+1,\dots,T$}
        \State Observe $\mc_t$, set $\tilde u_t=\proj_{\cU}(\widehat\mu(\mc_t))$, and call $\ORBIT$ with scalar pilot $\tilde u_t$; the $\ORBIT$-selected price is posted, the resulting $y_t$ is fed back to $\ORBIT$, and $\ORBIT$ updates its internal state.
    \EndFor
\end{algorithmic}
\end{algorithm}

Due to the two phases design, the pilot estimator constructed in Algorithm~\ref{alg:Init} satisfies naturally the independence condition required in Assumption~\ref{assump:orbit-interface}, so we will focus on the discussion of pilot accuracy in the followed part.

Such modularity \textit{shifts} the statistical burden from  $\ORBIT$ to the offline estimation oracle over $\cF$. Unlike the adaptive construction in Section~\ref{sec:linear-main}, which does not require assumptions on the context distribution, the explore-first approach relies on whatever regularity conditions are needed for the oracle to deliver a uniform error guarantee, which may be different for different $\cF$. We state the general requirement through the performance of an offline estimation oracle condition following 
\citet{gong2025minimax}.

\begin{assumption}[Offline estimation oracle]\label{cond:offline-estimation}
Fix a utility class $\cF$ and a context distribution $P_{\cC}$. We say that an $\cF$-dependent offline estimation oracle $\cB^{\mathsf{est}}$ has error rate $\varepsilon_n(\delta)$ under $P_{\cC}$ if for each $\mu\in\cF$ and each i.i.d. dataset $\cD_n=\{(\mc_i,Z_i)\}_{i=1}^n$ satisfying that $\mc_i\sim P_{\cC}$, $\E[Z_i\mid \mc_i]=\mu(\mc_i)$, and $\abs{Z_i}\leq p_{\max}$, the output $\widehat\mu:=\cB^{\mathsf{est}}(\cD_n)$ satisfies that 
\begin{align*}
    \Prob\left(\sup_{\mc\in\cC}\abs{\widehat\mu(\mc)-\mu(\mc)}\leq \varepsilon_n(\delta)\right)
    \geq 1-\delta .
\end{align*}
\end{assumption}

The theorem below is an immediate consequence of the conditional $\ORBIT$ guarantee. It shifts the statistical burden from $\ORBIT$ to the offline oracle over $\cF$.

\begin{theorem}[Regret with offline utility oracles]\label{thm:regret-with-offline-oracles}
Assume that Assumptions~\ref{assump:bounded}--\ref{assump:rho} hold, $\mu_\ast\in\cF$, and the contexts are i.i.d. from $P_{\cC}$. Let $\delta_T=T^{-3}$. Assume further that Assumption~\ref{cond:offline-estimation} holds for $\cB^{\mathsf{est}}$ under $(\cF,P_{\cC})$, and for all sufficiently large $n$,
\begin{align}\label{eq:offline-rate-condition}
    \varepsilon_n(\delta_T)\leq \cV_T(\cF)n^{-\alpha}
\end{align}
for some $\alpha>0$. Run Algorithm~\ref{alg:Init} with
\begin{align*}
    n_{\mathsf{exp}}\asymp T^{\frac{1}{1+2\alpha}}\cV_T(\cF)^{\frac{2}{1+2\alpha}},
\end{align*}
estimation oracle $\cB^{\mathsf{est}}$, and $\ORBIT$ parameters as in Theorem~\ref{thm:upper}. If $n_{\mathsf{exp}}<T$ and the resulting pilot accuracy satisfies the smallness conditions required in Theorem~\ref{thm:upper}, we have that 
\begin{align*}
    \Regret(T)
    =\widetilde{\cO}\left(
        T^{\frac{2\beta-1}{4\beta-3}}
        +T^{\frac{1}{1+2\alpha}}\cV_T(\cF)^{\frac{2}{1+2\alpha}}
    \right).
\end{align*}
\end{theorem}

The error rate condition~\eqref{eq:offline-rate-condition} above is well-understood for many function classes, such as the linear models \citep{fan2024policy,wang2025tight,han2026general}, generalized linear models \citep{li2017provably}, and finite classes \citep{gong2025minimax}. Plugging the associated oracles into Theorem~\ref{thm:regret-with-offline-oracles} above directly recovers or extends the corresponding regret guarantees, so we do not pursue these directions here.

In the remainder of this section, we will instead focus on two utility classes that are adopted widely in contextual pricing yet fall outside the linear analysis of Section~\ref{sec:linear-main}. Section~\ref{subsec:linear-sparse-utility} considers a sparse high-dimensional linear utility, where a Lasso-based oracle applies under a compatibility condition on the context covariance. Section~\ref{subsec:nonparametric-utility} considers a $\gamma$-H\"{o}lder
nonparametric utility, for which a local-polynomial oracle applies under a bounded-density condition. For both cases, Theorem~\ref{thm:regret-with-offline-oracles} gives \textit{improved or new} regret bounds \textit{without} the need of 
stronger assumptions made in prior works.

\subsection{Sparse high-dimensional linear utility}\label{subsec:linear-sparse-utility}

Our first instantiation is the sparse high-dimensional linear model, widely
adopted in contextual pricing under various feedback
types~\citep{javanmard2019dynamic,ban2021personalized,javanmard2020multi}.

\begin{definition}[Sparse linear utility model]\label{def:linear-sparse}
A $d$-dimensional, $s$-sparse linear utility model has $\mu_\ast(\mc)=\mc^\top\theta_\ast$ for some $\cC\subset\mathbb B_\infty^d(1)$ and $\theta_\ast\in\mathbb B_1^d(C_\theta)$ satisfying that $\|\theta_\ast\|_0\leq s$ and $\mc^\top\theta_\ast\geq0$ for all $\mc\in\cC$.
\end{definition}

The model above satisfies Assumption~\ref{assump:bounded} with $\cU=[0,C_\theta]$. The $\ell_\infty$-bound on $\mc$ in Definition~\ref{def:linear-sparse} is the standard convention of~\citet{bastani2020online,hao2020high}, ensuring a $d$-independent lower bound on $c_{\min}$; under only an-$\ell_2$ bound one implicitly has $c_{\min} = \cO(1/d)$, which inflates the final regret by an
extra $\sqrt{d}$ factor.

\paragraph{Context distribution.} To utilize sparsity, we adopt the following standard compatibility
condition~\citep{buhlmann2011statistics,bastani2020online,javanmard2019dynamic}:

\begin{assumption}[Compatibility condition]\label{assumption:compatibility}
The contexts are i.i.d. from a distribution $P_{\cC}$ on $\cC$. Let
$\Sigma:=\E_{P_{\cC}}[\mc\mc^\top]$. There exists some $c_{\min}>0$ such that for each $I\subset[d]$,
\[
    \theta^\top \Sigma \theta
    \geq c_{\min}\,\frac{\lVert \theta_I \rVert_1^2}{\lvert I \rvert},
    \quad \forall \theta \in \R^d
    \text{ with } \lVert \theta_{I^c}\rVert_1 \leq 3\lVert \theta_I\rVert_1.
\]
\end{assumption}

\paragraph{Offline regression oracle.} For the offline oracle, we employ the Lasso estimator
\begin{align}\label{eq:f-Lasso}
    \widehat\mu(\mc):=\mc^\top\widehat\theta, \qquad
    \widehat\theta  \in\argmin_{\theta\in\R^d} \frac{1}{n_{\mathsf{exp}}}\sum_{t=1}^{n_{\mathsf{exp}}}  (Z_t-\mc_t^\top\theta)^2 +\lambda\|\theta\|_1,  \qquad
    \lambda=C_\lambda p_{\max}\sqrt{\frac{\log(dT)}{n_{\mathsf{exp}}}}.
\end{align}
Under Assumption~\ref{assumption:compatibility}, the standard Lasso theory shows that with probability at least $1-\cO(T^{-3})$,
\begin{align*}
    \sup_{\mc\in\cC}\abs{\widehat\mu(\mc)-\mu_\ast(\mc)} \leq
    C\frac{s p_{\max}}{c_{\min}}\sqrt{\frac{\log(dT)}{n_{\mathsf{exp}}}},
\end{align*}
up to the usual high-probability passage from population to empirical compatibility for the burn-in design; equivalently, one may condition directly on a realized compatible burn-in design. Consequently, Assumption~\ref{cond:offline-estimation} holds with $\alpha=1/2$ and $\cV_T(\cF)=\widetilde{\cO}(s)$, and Theorem~\ref{thm:regret-with-offline-oracles} yields the following bound.

\begin{corollary}\label{thm:regret-sparse}
Under Assumptions~\ref{assump:bounded}--\ref{assump:rho}
and~\ref{assumption:compatibility} with the sparse utility model in
Definition~\ref{def:linear-sparse}, Algorithm~\ref{alg:Init} with
$n_{\mathsf{exp}} \asymp s\sqrt{T}$, Lasso oracle in~\eqref{eq:f-Lasso},
and  $\ORBIT$ parameters as in Theorem~\ref{thm:upper} satisfies that 
\[
    \Regret(T)
    = \widetilde{\cO}\left(
        T^{\frac{2\beta-1}{4\beta-3}}
        + s\sqrt{T}
    \right).
\]
\end{corollary}

\subsection{Nonparametric utility model}\label{subsec:nonparametric-utility}

Our second example is a $\gamma$-H\"older nonparametric utility,
extending the $\gamma \in \{2,4\}$ setting of~\citet{chen2024dynamic} to
arbitrary smoothness $\gamma > 0$.

\begin{definition}[$\gamma$-H\"older utility model]\label{def:holder-utility}
Let $\gamma>0$ and $\cC=[0,1]^d$. The utility function $\mu_\ast:\cC\to[0,1]$ belongs to a standard $\gamma$-H\"older ball with radius $L_\mu$ on $\cC$.
\end{definition}

The model above satisfies Assumption~\ref{assump:bounded} with $\cU=[0,1]$. We impose the usual bounded-density condition for sup-norm nonparametric regression.

\paragraph{Context distribution.} We impose the following regularity condition on the covariate distribution as in \citet{chen2024dynamic, TsybakovNon}.

\begin{assumption}[Density condition]\label{assumption:density}
For each $t \in [n_{\mathsf{exp}}]$, the contexts are i.i.d. from a distribution $P_{\cC}$ on $\cC$ with density $f$ satisfying that $0 < f_{\min}\leq f(\mc)\leq f_{\max}<\infty$ for all $\mc\in\cC$.
\end{assumption}

\paragraph{Offline regression oracle.} Let $\ell_\gamma:=\lceil\gamma\rceil-1$ and choose bandwidth $b = \big(\frac{\log T}{ n_{\mathsf{exp}}}\big)^{\frac{1}{2\gamma+d}}$. The local-polynomial oracle of degree $\ell_\gamma$ returns
\begin{align}\label{eq:f-locpoly}
    \widehat\mu(x):=(\widehat\vartheta_x)_0, \qquad  \widehat\vartheta_x \in\argmin_{\vartheta}  \sum_{t:\|\mc_t-x\|_\infty\leq b}  \left(Z_t- \sum_{\abs{r}\leq\ell_\gamma}\vartheta_r\left(\frac{\mc_t-x}{b}\right)^r  \right)^2.
\end{align}
By the classical sup-norm guarantee for local polynomial regression \citep{TsybakovNon}, under Assumption~\ref{assumption:density}, we have that with probability at least $1-\cO(T^{-3})$,
\begin{align*}
    \sup_{\mc\in\cC}\abs{\widehat\mu(\mc)-\mu_\ast(\mc)} \leq C_{\gamma,d}\left(\frac{\log T}{n_{\mathsf{exp}}}\right)^{\frac{\gamma}{2\gamma+d}},
\end{align*}
where $C_{\gamma,d}$ depends on $\gamma,d,L_\mu,f_{\min},f_{\max}$, and $p_{\max}$. Thus, Assumption~\ref{cond:offline-estimation} holds with $\alpha=\gamma/(2\gamma+d)$ and $\cV_{T}(\cF)=\widetilde{\cO}(1)$, and Theorem~\ref{thm:regret-with-offline-oracles} gives the following bound.

\begin{corollary}\label{thm:holder-regret}
Under Assumptions~\ref{assump:bounded}--\ref{assump:rho}, bounded realized-valuation normalization, and Assumption~\ref{assumption:density} with the $\gamma$-H\"older utility model in Definition~\ref{def:holder-utility}, Algorithm~\ref{alg:Init} with $n_{\mathsf{exp}}\asymp T^{\frac{2\gamma+d}{4\gamma+d}}$, local-polynomial oracle in~\eqref{eq:f-locpoly}, and $\ORBIT$ parameters as in Theorem~\ref{thm:upper} satisfies that for all sufficiently large $T$ with $n_{\mathsf{exp}}<T$,
\begin{align*}
    \Regret(T) =\widetilde{\cO}\left( T^{\frac{2\beta-1}{4\beta-3}}+T^{\frac{2\gamma+d}{4\gamma+d}}\right).
\end{align*}
\end{corollary}

In comparison, under the same utility model and assumptions, the best-known regret prior to our work was achieved in \citet{chen2024dynamic}
\begin{align*}
    \widetilde{\cO}\Big( dT^{\frac{2\beta+1}{4\beta-1}} + T^{\frac{2\gamma+d}{4\gamma+d}} + \bm{1}\{\gamma=4\}\, T^{\frac{7}{13}} \Big), \quad \gamma \in \{2,4\}.
\end{align*}
Corollary~\ref{thm:holder-regret} above extends this result to all $\gamma > 0$ regime, and improves the dependence on both H\"older smoothness exponents $\gamma$ and $\beta$.

\section{Experiments}\label{sec:experiments}

In this section, we complement the theoretical results with a numerical study under three settings: linear utility with i.i.d.\ uniform-on-sphere covariates (Section~\ref{subsec: exp-linear-iid}), linear utility with i.i.d. ill-conditioned anisotropic transformed-sphere covariates (Section~\ref{subsec: exp-heterogeneous}), and sparse linear utility with i.i.d.\ uniform-cube covariates (Section~\ref{subsec: exp-sparse}).

Throughout the section, we simulate the \textit{linear semiparametric model} $u_t=\mc_t^\top\theta_\ast$, for the non-sparse case (Section~\ref{sec:linear-main}) and the sparse case (Section~\ref{subsec:linear-sparse-utility}) with H\"older exponent $\beta = 2$. The valuation distribution with tail function $g(\cdot)$ is fixed and shared across all methods. More precisely, with the smooth cutoff function 
\begin{align*}
\varphi(t) := \begin{cases}
0, & t\leq 0,\\
e^{-1/t}/(e^{-1/t}+e^{-1/(1-t)}), & t\in(0,1),\\
1, & t\geq 1,
\end{cases}   
\end{align*}
we set \begin{align*}
    g(z):=1-\varphi\big((z{+}0.3)/0.6\big).
\end{align*}
As a result, the valuation noise $\xi_t$ is supported on $[-0.3,0.3]$.

The methods compared are \textit{ORBIT-Adaptive} (Algorithm~\ref{alg:oplcb} with the adaptive pilot of Algorithm~\ref{alg:adaptive-init}), \textit{Explore-then-ORBIT-Lasso} (Algorithm~\ref{alg:Init} with the Lasso oracle of Section~\ref{subsec:linear-sparse-utility}), the doubling-episodic explore-then-commit algorithm of \citet{fan2024policy} with the OLS and Lasso oracles (\textit{ETC-OLS}, \textit{ETC-Lasso}), and the LPSP-style baseline of \citet{han2026general} (\textit{LPSP}), which we implement from their reference code. The algorithm in \citet{wang2025tight} is nearly the same as the implementation of \citet{han2026general} with $\beta = 2$, so is covered by the LPSP baseline.

In our implementation of $\ORBIT$, we employ the zeroth-order online gradient descent for the induced BCO problem in the local refinement step~\citep{flaxman2005online}. This differs from the theoretical version of Algorithm~\ref{alg:oplcb}, whose analysis is based on the online Newton-step refinement subroutine of~\citet{fokkema2024online}. The BCO algorithm of~\citet{fokkema2024online} is designed primarily for theoretical guarantees and in our experiments, does not scale well with the horizon $T$ or with the dimensionality of the local polynomial parameterization. For computational simplicity and empirical stability, we therefore replace it with the zeroth-order online gradient-descent subroutine of~\citet{flaxman2005online}, which is easier to implement and performs well in our numerical experiments.

\subsection{Linear utility with non-degenerate covariates}\label{subsec: exp-linear-iid}

We compare the empirical performance of ORBIT-Adaptive, ETC-OLS, and LPSP under the linear utility model investigated in Section~\ref{sec:linear-main}.

\vspace{0.1cm}
\noindent\textbf{Utility model and context distribution.} Fix dimensionality $d\geq 2$. At each time $t$, we sample $\mc_t=(\mc_t^{\mathrm{rand}},1)$ with $\mc_t^{\mathrm{rand}}\sim_{\text{i.i.d.}} \text{Unif}(\mathbb S^{d-2})$. The true parameter is set to $\theta_\ast=(\mathbf{1}_{d-1}/\sqrt{d-1},2)$. Such a design satisfies Assumption~\ref{assump:linear}, up to a $\sqrt{2}$ rescaling of the context bound. Moreover, the latent index $u_t=\theta_\ast^\top \mc_t$ lies in $[1,3]$, and the perturbed utility $u_t+\xi_t$ lies in $[0.7,3.3]$. Hence, the choice of $p_{\max}=3.5$ ensures that the oracle optimal price is strictly interior for every realization considered in the experiment.

\vspace{0.1cm}
\noindent\textbf{Environment parameters.} Under this environment, we sweep the horizon $T\in\{3\cdot 10^3,10^4,3\cdot 10^4,10^5\}$ and  dimensionality $d\in\{5,10,20\}$, using $50$ independent repetitions for each $(T,d)$ setting.

\vspace{0.1cm}
\noindent\textbf{Results.} Figure~\ref{fig: linear-iid-regret} depicts the cumulative regret as a function of $T$ on a log--log scale. ORBIT-Adaptive has performance comparable to the baselines when $d=5$, and becomes substantially better as the dimensionality grows. This pattern is consistent with the theory in two aspects. First, ORBIT-Adaptive has a milder dimensionality dependence, with the leading pilot-estimation contribution scaling as $\sqrt{dT}$, whereas the guarantees for ETC-OLS and LPSP have dependence on $d$ of order $\mathcal O(d^3)$. Second, the regret guarantee of ORBIT-Adaptive, as well as the amount of exploration required by its adaptive pilot, does not inflate as $\lambda_{\min}^{-1}(\mathbb E[\mc_t\mc_t^\top])$ grows. This is important in the present bounded-context design, where the minimum eigenvalue of the design covariance inevitably decreases as $d$ grows. In fact, the reported regret of ORBIT-Adaptive decreases with $d$: under the uniform-sphere distribution and the chosen signal-spread parameter $\theta_\ast$, the adaptive pilot used by ORBIT becomes easier as $d$ increases (to see this, notice that the variance of $\mc_t^\top \theta_\ast$ scales as $\cO(1/d)$). In contrast, ETC-OLS and LPSP rely on exploration lengths that grow sharply as the minimum eigenvalue of the design covariance decays, and thus cannot exploit this favorable structure.

\begin{figure}[h!]
    \captionsetup[sub]{font=small}
    \centering
    \begin{subfigure}{0.32\textwidth}\includegraphics[width=\linewidth]{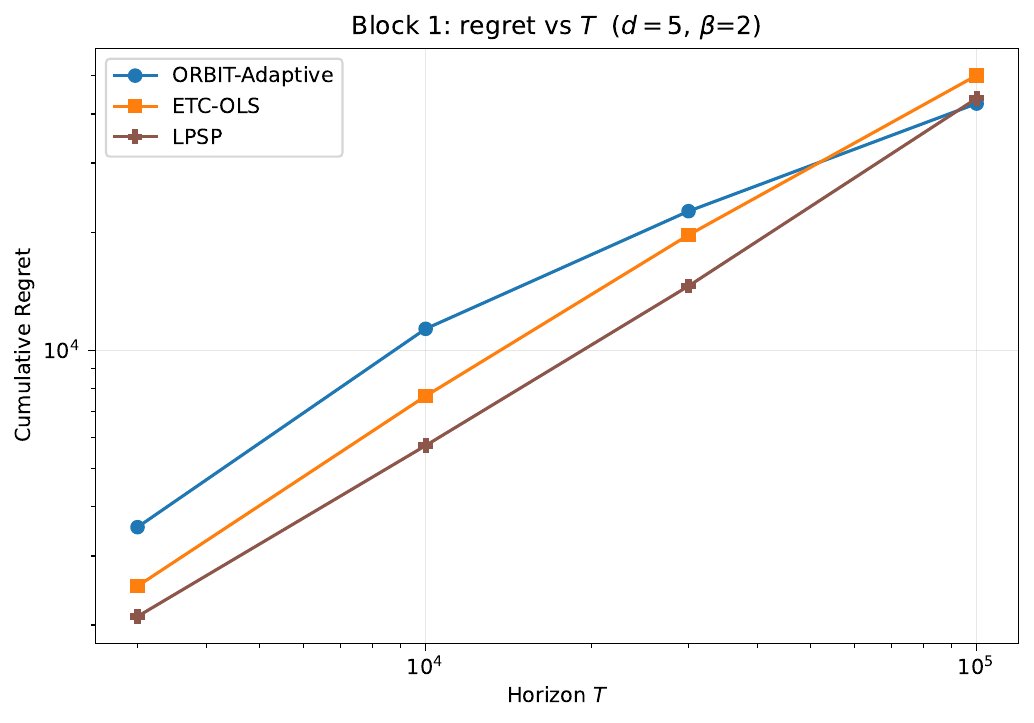}\caption*{$d=5$}\end{subfigure}
    \begin{subfigure}{0.32\textwidth}\includegraphics[width=\linewidth]{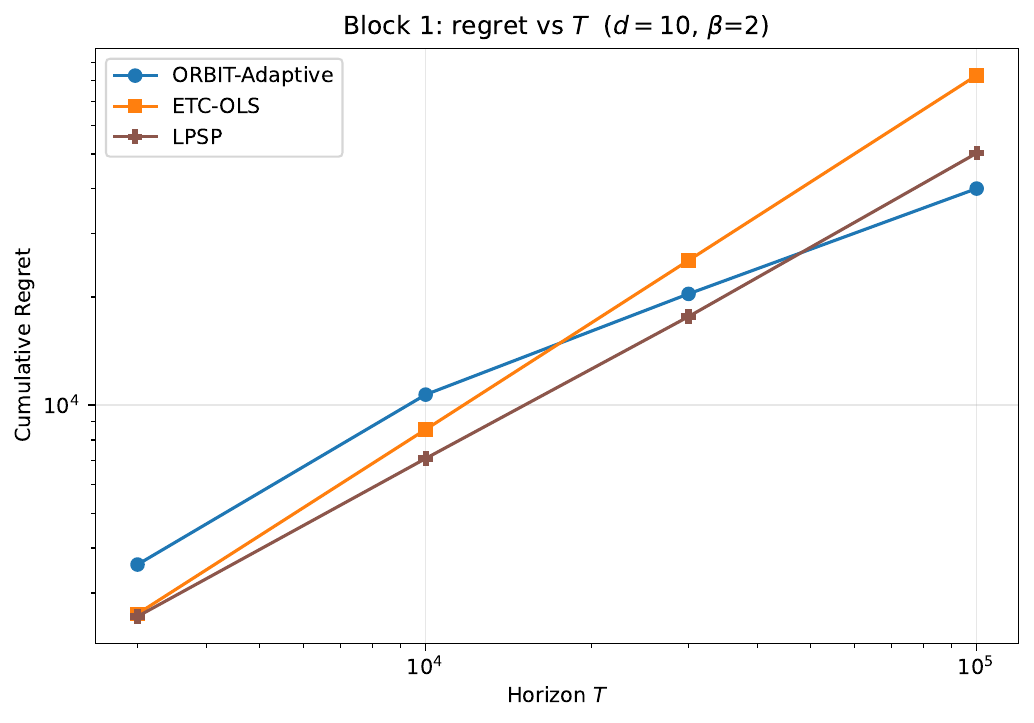}\caption*{$d=10$}\end{subfigure}
    \begin{subfigure}{0.32\textwidth}\includegraphics[width=\linewidth]{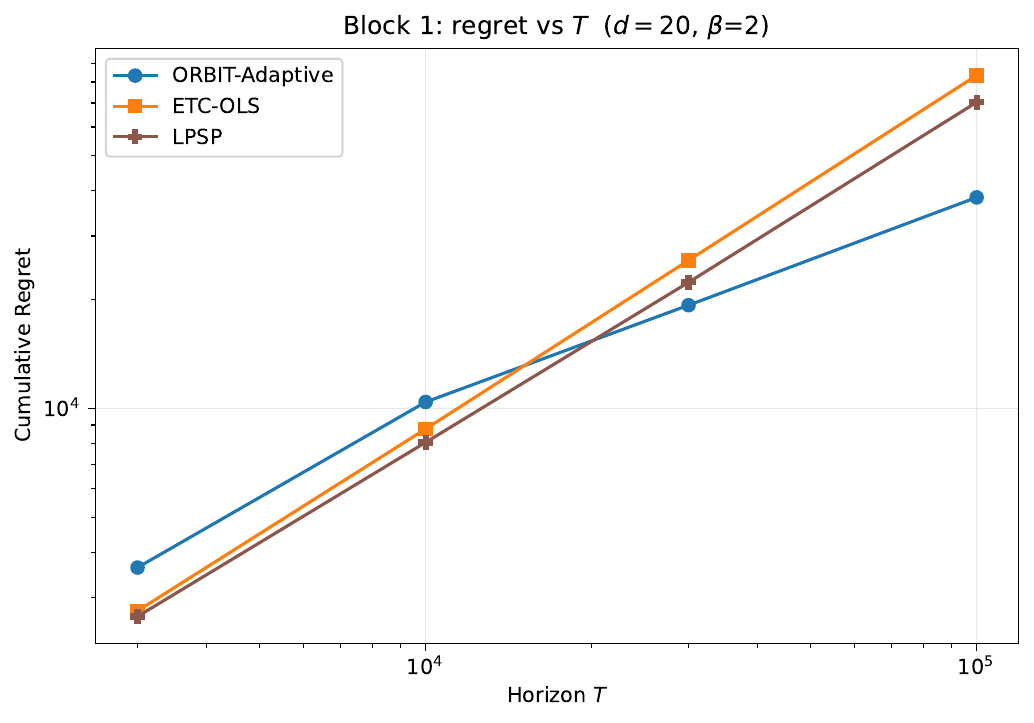}\caption*{$d=20$}\end{subfigure}
    \caption{Cumulative regret versus horizon $T$ for i.i.d.\ context dimensions $d\in\{5,10,20\}$, $50$ repetitions per setting.}
    \label{fig: linear-iid-regret}
\end{figure}

\subsection{Linear utility with ill-conditioned covariates}\label{subsec: exp-heterogeneous}

To expose the dependence of performance on the covariate structure, in this setting we compare the numerical results of ORBIT-Adaptive, ETC-OLS, and LPSP with ill-conditioned covariates.

\noindent\textbf{Utility model and context distribution.} We fix dimensionality $d\geq 2$ and sample $\mc_t = (\mc_t^{\mathrm{rand}}, 1)$ with $\mc_t^{\mathrm{rand}} = \Sigma_\varepsilon^{1/2} z_t$, where $z_t \sim_{\text{i.i.d.}} \text{Unif}(\mathbb S^{d-2})$ and
\begin{align*}
\Sigma_\varepsilon = (1-\varepsilon)v v^\top + \varepsilon I_{d-1} \qquad\text{ for some } \|v\|_2=1,\ v\perp\bm{1}_{d-1}.
\end{align*}
Such construction yields that $\mathrm{Cov}(\mc_t^{\mathrm{rand}}) = \Sigma_\varepsilon/(d-1)$: the design covariance is \emph{exactly proportional} to $\Sigma_\varepsilon$, and in particular, $\lambda_{\min}(\mathrm{Cov}(\mc_t^{\mathrm{rand}}))=\varepsilon/(d-1)$ depends linearly on the $\epsilon$ parameter.

We employ the same $\theta_\ast=(\mathbf{1}_{d-1}/\sqrt{d-1},2)$ as in Section~\ref{subsec: exp-linear-iid}. Under such design, the required assumptions are satisfied for the same reason.

\vspace{0.1cm}
\noindent\textbf{Environment parameters.} We fix $T=5\cdot 10^4$, $d=5$, and change $\varepsilon$ from $1.0$ to $0.05$ across the grid $\{1.0,0.5,0.2,0.1,0.05\}$. At $\varepsilon=1$, the design recovers the isotropic-on-sphere baseline of Section~\ref{subsec: exp-linear-iid} at $d=5$; as $\varepsilon\to 0$, the covariance of context then degenerates to the rank-one matrix $vv^\top/(d-1)$. We repeat the experiment 50 times for each $\varepsilon$.

\vspace{0.1cm}
\noindent\textbf{Results.} Figure~\ref{fig: heterogeneous-regret} reports the cumulative regret of the three algorithms as $\varepsilon$ varies, on a log--log scale. As $\varepsilon$ decreases, the regrets of both ETC-OLS and LPSP increase sharply. This is because their pilot-estimation stages depend on the inverse minimum singular value of the context covariance matrix, which deteriorates under stronger context concentration. In contrast, ORBIT-Adaptive exhibits the opposite trend and performs better as $\varepsilon$ decreases. Such observation is consistent with the uncertainty criterion in Algorithm~\ref{alg:adaptive-init}: concentration around a fixed one-dimensional direction reduces the intrinsic difficulty of piloting for the adaptive exploration procedure, even though it makes the ambient covariance matrix more ill-conditioned.

\begin{figure}[h!]
    \centering
    \includegraphics[width=0.55\textwidth]{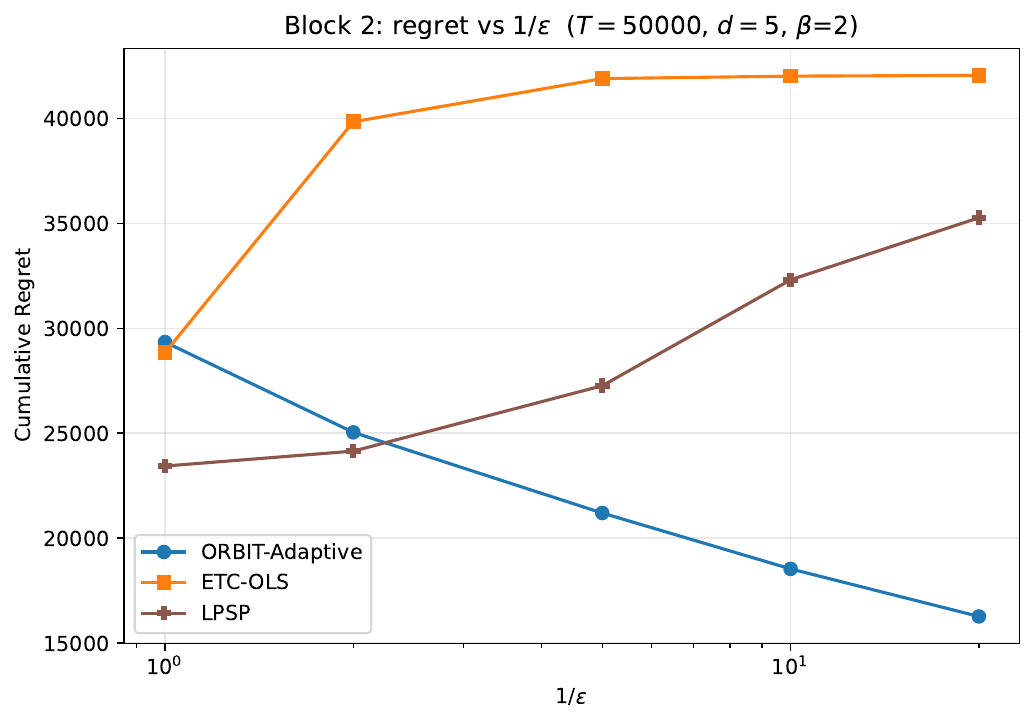}
    \caption{Cumulative regret versus $1/\varepsilon$ for anisotropic covariates, as $\varepsilon$ changes from $1$ to $0.05$, $50$ repetitions per setting.}
    \label{fig: heterogeneous-regret}
\end{figure}

\subsection{Sparse linear utility}\label{subsec: exp-sparse}
Finally, we compare the performance of Explore-then-ORBIT-Lasso, ETC-OLS, and ETC-Lasso under the sparse utility setting as described in Section~\ref{subsec:linear-sparse-utility}. We do not include the LPSP benchmark since its joint least-squares estimation step does not scale well in $d$, as described in \citet{han2026general}.

\begin{figure}[t]
    \captionsetup[sub]{font=small}
    \centering
    \begin{subfigure}{0.49\textwidth}\includegraphics[width=\linewidth]{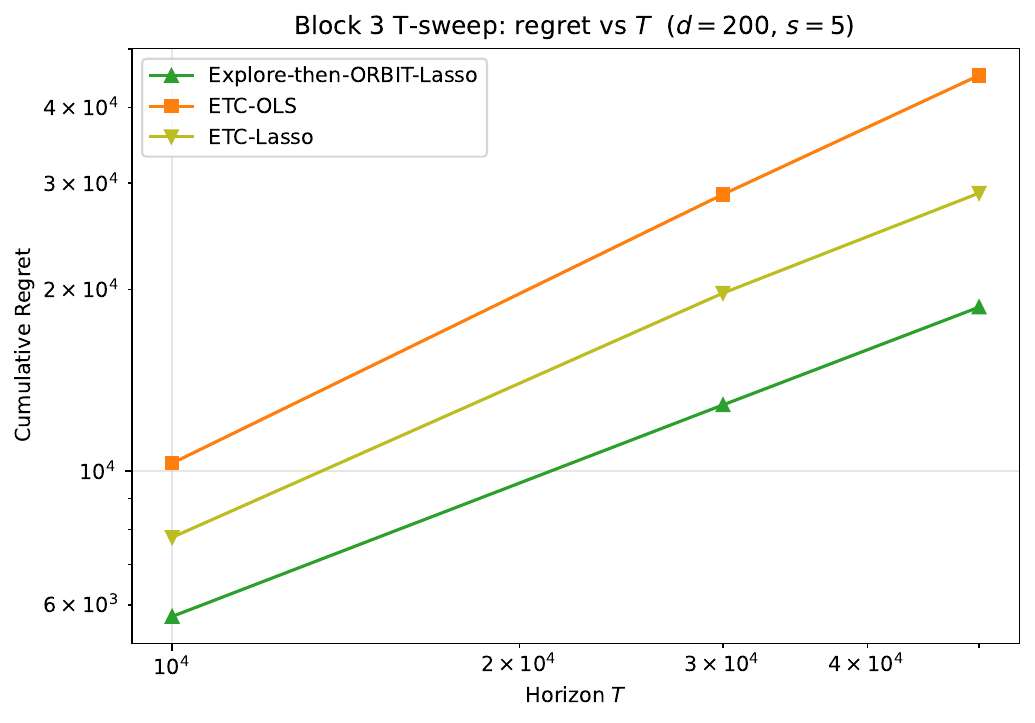}\caption*{regret vs $T$  ($d=200$)}\end{subfigure}
    \begin{subfigure}{0.49\textwidth}\includegraphics[width=\linewidth]{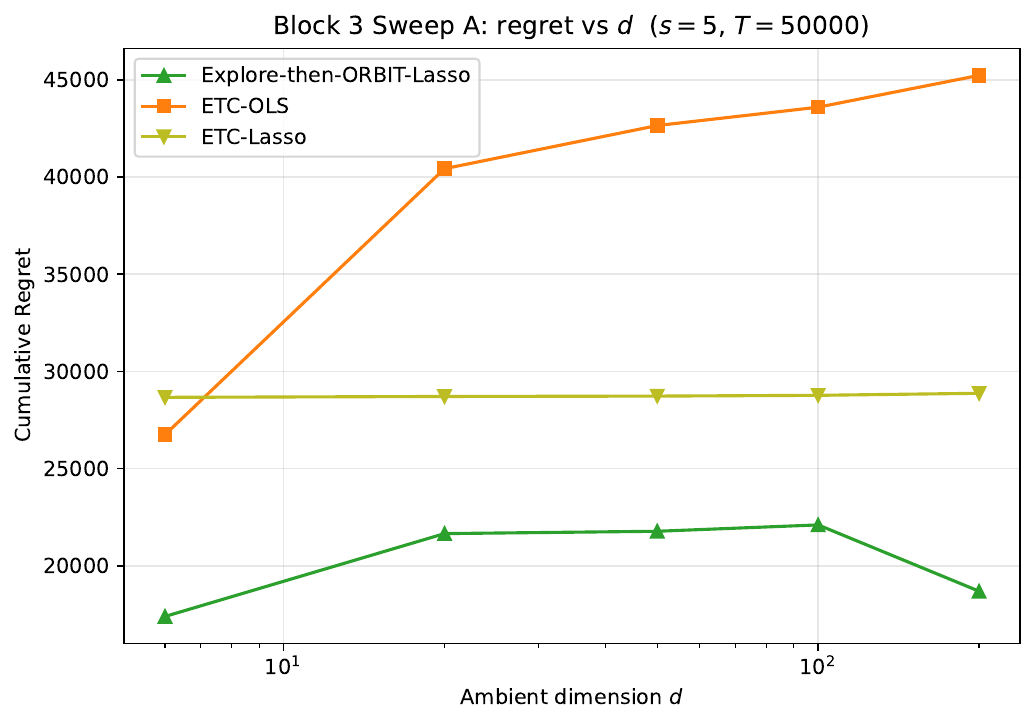}\caption*{regret vs $d$  ($T=5\cdot 10^4$)}\end{subfigure}
    \caption{Sparse linear utility at $s=5$, $50$ repetitions per setting. \textbf{Left:} cumulative regret versus horizon $T\in\{10^4,3\cdot 10^4,5\cdot 10^4\}$ at the largest dimension $d=200$ (log--log axes). \textbf{Right:} cumulative regret versus $d\in\{6,20,50,100,200\}$ at $T=5\cdot 10^4$ (log $x$-axis).}
    \label{fig: sparse-rate}
\end{figure}

\vspace{0.1cm}
\noindent\textbf{Utility model and context distribution.} Fix dimensionality $d\geq 2$ and  sparsity level $s$. At each time $t$, we sample $\mc_t=(\mc_t^{\mathrm{rand}},1)$, where the components of $\mc_t^{\mathrm{rand}}$ are drawn i.i.d.\ from $\mathrm{Unif}([-1,1])$. The true parameter $\theta_\ast$ is constructed as follows. We randomly select exactly $s$ nonzero non-intercept coordinates, assign each selected coordinate magnitude $1/s$ with an independent random sign, and set the last coordinate, corresponding to the intercept term, to be $2$. Such construction gives that $\|\theta_\ast\|_1=3$, so Definition~\ref{def:linear-sparse} is satisfied with $C_\theta=3$. Moreover, we have that $u_t\in[1,3]$ and $u_t+\xi_t\in[0.7,3.3]$. Thus, choosing $p_{\max}=3.5$ covers the full range of possible valuations and keeps the oracle price interior throughout the experiment.

\vspace{0.1cm}
\noindent\textbf{Environment parameters.} We report the cumulative regret on a log--log scale for horizons $T\in\{10^4,3\cdot 10^4,5\cdot 10^4\}$ with fixed dimensionality $d=200$. We also present the regret at $T=5\cdot 10^4$ as dimensionality varies over $d\in\{6,20,50,100,200\}$. For both experiments, we fix the sparsity level at $s=5$ and use $50$ independent repetitions for each setting.

\vspace{0.1cm}
\noindent\textbf{Results.} Figure~\ref{fig: sparse-rate} presents the results. The two panels show that both Explore-then-ORBIT-Lasso and ETC-Lasso substantially outperform ETC-OLS in the high-dimensional setting, reflecting their nearly dimension-free dependence under sparsity.

\section{Discussions} \label{sec:conclusion}

We have in this paper developed an oracle-price-map view of contextual pricing. Under the smoothness of the unknown noise tail and a strong revenue-geometry condition, the regret-relevant object is the one-dimensional map $u\mapsto p^\ast(u)$. $\ORBIT$ learns such map through a coarse-to-fine architecture: a scalar pilot locates the latent state, a short grid phase localizes a safe anchor price in each active bin, and a local polynomial convex-bandit routine refines prices inside the resulting trust region. For the baseline linear utility model, an adaptive pilot construction yields the fully online regret $\widetilde\cO(T^{\frac{2\beta-1}{4\beta-3}}+\sqrt{dT})$, and the lower bound unveils that the horizon exponent is minimax sharp for fixed $d$.

The same interface suggests several directions for future work. One is to design adaptive choices of the grid and trust-region parameters that avoid structural tuning constants. Another is to combine oracle-map learning with operational constraints such as inventory, capacity, and fairness. A third direction is to extend this framework to other semiparametric decision problems, such as auction reserve pricing where the target object is again a low-dimensional oracle decision rule rather than the full demand model.

\bigskip

\bibliographystyle{informs2014}
\bibliography{mybib}

\newpage
\begin{APPENDICES}

\section{CDF-shape regularity and quadratic revenue geometry}\label{appenidx: compare-assumption}
This appendix records a precise calculation connecting the cumulative distribution function (CDF)-level shape conditions to the quadratic revenue-growth bounds used in Assumption~\ref{assump:rho}. In particular, we will focus on the shape assumptions on the CDF function used in \citet{fan2024policy,javanmard2019dynamic}, i.e., Assumption 2.1+(A.26) in \cite{fan2024policy} and Assumption 1 in \cite{javanmard2019dynamic}. Let $f_\Xi=F_\Xi'$ be the density, and recall that $g=1-F_\Xi$. We define 
\begin{align*}
    \phi(v):=v-\frac{1-F_\Xi(v)}{f_\Xi(v)}=v-\frac{g(v)}{f_\Xi(v)} .
\end{align*}

\begin{assumption}\label{assumption:cdf-shape}
There exist some constants $c_l,c_u,c_\phi,M_f>0$ such that on $[-V,V]$, 
    \begin{enumerate}
        \item[(i)] $f_\Xi$ is continuously differentiable, $c_l\leq f_\Xi(z)\leq c_u$, and $|f_\Xi'(z)|\leq M_f$ for all $z\in[-V,V]$.
        \item[(ii)] $\phi'(z)\geq c_\phi$ for all $z\in[-V,V]$. 
    \end{enumerate}
\end{assumption}

\begin{lemma}\label{lem:local-concave-via-fan}
Assume that Assumptions~\ref{assump:bounded} and \ref{assumption:cdf-shape} hold. Fix any $u\in\cU$ and any price $p^\dagger\in[0,p_{\max}]$ such that $z^\dagger:=p^\dagger-u\in[-V,V]$ and $\phi(z^\dagger)=-u$. Then we have that for each $p\in[0,p_{\max}]$, 
    \begin{align*}
        \frac{c_lc_\phi}{2}\abs{p-p^\dagger}^2 \leq r(u,p^\dagger)-r(u,p) \leq \frac{c_u}{2}\left(2+\frac{M_f}{c_l^2}\right)\abs{p-p^\dagger}^2 .
    \end{align*}
\end{lemma}

\textit{Proof}. Let us write $z = p - u$. Since $g' = -f_\Xi$, a direct calculation leads to 
    \begin{align}\label{eq:cdf-rp}
        r_p(u,p) = g(z)-p f_\Xi(z) =-f_\Xi(z)\big(u+\phi(z)\big).
    \end{align}
    For each $s\in[0,p_{\max}]$, Assumption~\ref{assump:bounded} entails that $s-u\in[-V,V]$, so the density and $\phi'$ bounds in Assumption~\ref{assumption:cdf-shape} apply. In view of $\phi(z^\dagger)=-u$, \eqref{eq:cdf-rp} can be rewritten as
    \begin{align*}
        r_p(u,s) = f_\Xi(s-u)\{\phi(z^\dagger)-\phi(s-u)\}.
    \end{align*}

    We will first prove the lower bound. If $p<p^\dagger$, we have that for $s\in[p,p^\dagger]$,
    \begin{align*}
        \phi(z^\dagger)-\phi(s-u) \geq c_\phi\{z^\dagger-(s-u)\} = c_\phi(p^\dagger-s).
    \end{align*}
    Using $f_\Xi\geq c_l$ and integrating, we can deduce that 
    \begin{align*}
        r(u,p^\dagger)-r(u,p) = \int_p^{p^\dagger}r_p(u,s)\mathrm{d}s \geq c_lc_\phi\int_p^{p^\dagger}(p^\dagger-s)\mathrm{d}s = \frac{c_lc_\phi}{2}\abs{p-p^\dagger}^2 .
    \end{align*}
    If $p>p^\dagger$, it holds that for $s\in[p^\dagger,p]$,
    \begin{align*}
        -r_p(u,s) = f_\Xi(s-u)\{\phi(s-u)-\phi(z^\dagger)\} \geq c_lc_\phi(s-p^\dagger),
    \end{align*}
    and integration over $[p^\dagger,p]$ gives the same lower bound.

    For the upper bound, notice that
    \begin{align*}
        \phi'(z) = 2+\frac{g(z)f_\Xi'(z)}{f_\Xi(z)^2}.
    \end{align*}
    Since $0\leq g\leq 1$, $\abs{f_\Xi'}\leq M_f$, and $f_\Xi\geq c_l$ on $[-V,V]$, it follows that 
    \begin{align*}
        \abs{\phi'(z)}\leq M_\phi:=2+\frac{M_f}{c_l^2}, \quad z\in[-V,V].
    \end{align*}
    If $p<p^\dagger$, we have that for $s\in[p,p^\dagger]$,
    \begin{align*}
        r_p(u,s)  = f_\Xi(s-u)\{\phi(z^\dagger)-\phi(s-u)\} \leq c_uM_\phi(p^\dagger-s).
    \end{align*}
    Consequently, integrating yields that 
    \begin{align*}
        r(u,p^\dagger)-r(u,p) \leq  \frac{c_uM_\phi}{2}\abs{p-p^\dagger}^2.
    \end{align*}
    The case of $p>p^\dagger$ is identical after integrating $-r_p(u,s)$ over $[p^\dagger,p]$. This establishes the claimed upper and lower quadratic revenue-growth bounds, which completes the proof of Lemma~\ref{lem:local-concave-via-fan}.

\section{Proofs for Section~\ref{sec:model}}\label{app:sec2-proofs}

In this appendix, we establish the structural results from Section~\ref{sec:model}. We will first show that the oracle price map is smooth and then extract a uniform local-concavity radius from the quadratic-growth condition.

\subsection{Proof of Lemma~\ref{lem:pstar-regularity}}

Recall that $\cU=[u_{\min},u_{\max}]$. Denote by $A:=[0,p_{\max}]$, and recall that $r(u,p)=p g(p-u)$. To begin, note that $r$ is continuous on $\cU\times A$. Since $A$ is compact and by Assumption~\ref{assump:rho}-(1), for each $u\in\cU$ the maximizer $p^\ast(u)\in \argmax_{p\in A} r(u,p)$ is unique, an application of Berge's maximum theorem shows that the argmax map $p^\ast:\cU\to(0,p_{\max})$ is continuous. Let us define
$F(u,p):=\partial_p r(u,p)=g(p-u)+p g'(p-u).$
Fix $u_0\in\cU$ and set $p_0:=p^\ast(u_0)$. Since $p_0\in(0,p_{\max})$ and $p_0-u_0\in(-V,V)$, there exists an open neighborhood $N_{u_0}$ of $(u_0,p_0)$ on which $F$ is of class $C^{\beta-1}$ (for noninteger $\beta$, this is understood in the usual H\"older sense $C^{\lfloor \beta-1\rfloor,\beta-1-\lfloor \beta-1\rfloor}$). Moreover, because $p^\ast(u)$ is an interior maximizer, we have that 
$F(u,p^\ast(u))=0$ for all $u\in\cU$.

Let us fix $u\in\cU$ and write $p^\ast:=p^\ast(u)$. Since $r(u,\cdot)$ is $C^2$ on $[0,p_{\max}]$, an application of Taylor's theorem at the interior maximizer leads to 
\begin{align*}
r(u,p^\ast)-r(u,p)= -\frac{1}{2} r_{pp}(u,p^\ast) (p-p^\ast)^2 + \mathfrak{o}\big((p-p^\ast)^2\big) \quad\text{as } p\to p^\ast.
\end{align*}
Dividing by $(p-p^\ast)^2$ and using Assumption~\ref{assump:rho}-(2) yield that 
\begin{align*}
\frac{\sigma_r}{2} \leq -\frac{1}{2} r_{pp}(u,p^\ast) \leq \frac{L_r}{2};
\end{align*}
that is,
\begin{align*}
\sigma_r\leq -r_{pp}(u,p^\ast(u))\leq L_r \quad\forall u\in\cU.
\end{align*}
Hence, it holds that 
\begin{align*}
\partial_p F(u,p^\ast(u))=r_{pp}(u,p^\ast(u))\leq -\sigma_r<0 \quad\forall u\in\cU.
\end{align*}

We now fix $u_0\in\cU$. Since $F(u_0,p^\ast(u_0))=0$ and $\partial_p F(u_0,p^\ast(u_0))\neq 0$, an application of the local implicit-function theorem in H\"older spaces gives an open interval $I_{u_0}\subset \mathbb R$ containing $u_0$, an open interval $J_{u_0}\subset(0,p_{\max})$ containing $p^\ast(u_0)$, and a function $\phi_{u_0}\in C^{\beta-1}(I_{u_0})$ such that $F(u,\phi_{u_0}(u))=0$ for all $u\in I_{u_0}$ and $\phi_{u_0}(u)$ is the unique solution of $F(u,p)=0$ in $J_{u_0}$.

Since $p^\ast$ is continuous and $p^\ast(u_0)=\phi_{u_0}(u_0)$, after shrinking $I_{u_0}$ if necessary we can assume that
$p^\ast(u)\in J_{u_0} \quad \forall u\in I_{u_0}\cap\cU.$
Since $F(u,p^\ast(u))=0$ for each $u$, the uniqueness of the zero in $J_{u_0}$ leads to $p^\ast(u)=\phi_{u_0}(u) \quad \forall u\in I_{u_0}\cap\cU.$
Consequently, $p^\ast$ is of class $C^{\beta-1}$ in a neighborhood of each $u_0\in\cU$, interpreted relative to $\cU$ at the endpoints.

Finally, since $\cU$ is compact, finitely many such neighborhoods cover $\cU$. On overlaps, the corresponding local representations agree, since they all coincide with $p^\ast$. Thus, we have that $p^\ast\in C^{\beta-1}(\cU).$ In particular, on the compact interval $\cU$, the derivative of $p^\ast$ is bounded, so $p^\ast$ is Lipschitz:
$\abs{p^\ast(u)-p^\ast(v)} \leq L_p \abs{u-v} \quad \forall u,v\in\cU.$
This completes the proof of Lemma~\ref{lem:pstar-regularity}.

\subsection{Proof of Lemma~\ref{lem:rho-exists}}

By invoking Lemma~\ref{lem:pstar-regularity}, map $p^\ast(\cdot)$ is continuous over $\cU$. As a result, function $u\mapsto \min\{p^\ast(u), p_{\max}-p^\ast(u)\}$ is continuous on the compact interval $\cU$ and strictly positive everywhere. Hence, it holds that 
$\delta_0:=\min_{u\in\cU}\min\{p^\ast(u), p_{\max}-p^\ast(u)\}>0.$ Let us define $H(u,p):=-r_{pp}(u,p).$
Since $\beta\geq 2$ and $g\in C^\beta([-V,V])$, map $H$ is continuous on $\cU\times[0,p_{\max}].$
Moreover, the Taylor argument from the proof of Lemma~\ref{lem:pstar-regularity} shows that
$H(u,p^\ast(u))\geq \sigma_r, \ \forall u\in\cU.$

We next claim that there exists some $\delta_1>0$ such that
$H(u,p)\geq \sigma_r/2$ whenever $u\in\cU$ and $\abs{p-p^\ast(u)} \leq \delta_1.$
Assume, to the contrary, that no such $\delta_1$ exists. Then for each $n\geq 1$, there exist some $u_n\in\cU$ and $p_n\in[0,p_{\max}]$ such that
$\abs{p_n-p^\ast(u_n)} \leq \frac{1}{n}$ and $H(u_n,p_n)<\sigma_r/2.$
By compactness of $\cU$, after passing to a subsequence we can assume that $u_n\to u_\infty\in\cU.$
Since $p^\ast$ is continuous, we have that  $p^\ast(u_n)\to p^\ast(u_\infty)$, and hence $p_n\to p^\ast(u_\infty)$.
By continuity of $H$, we then obtain that $H(u_\infty,p^\ast(u_\infty))\leq \sigma_r/2$,
contradicting $H(u_\infty,p^\ast(u_\infty))\geq \sigma_r$.
This establishes the claim.

We set $\rho_0:=\min\{\delta_0/2,\delta_1\}.$ Then for each $u\in\cU$, it holds that $[p^\ast(u)-\rho_0, p^\ast(u)+\rho_0]\subset (0,p_{\max})$ since $\rho_0\leq \delta_0/2$. We also have that 
$-r_{pp}(u,p)=H(u,p)\geq \sigma_r/2$ whenever $\abs{p-p^\ast(u)}\leq \rho_0$ because $\rho_0\leq \delta_1$.  This concludes the proof of Lemma~\ref{lem:rho-exists}.

\medskip
\section{Proofs for Section~\ref{sec:oplcb}}\label{app:sec3-proofs}

In this appendix, we provide the proofs of the upper-bound lemmas from Section~\ref{sec:oplcb}. 

\subsection{Proof of Lemma~\ref{lem:mean-concentration}}

Observe that conditional on the associated index--pilot sequence $\{(u_s,\tilde u_s)\}_{s=1}^N$, the pilot states, bin assignments, and coarse sampling times are deterministic. Fix a bin $j$ that has completed its coarse phase and a price $p\in\cG$. Let $t_1,\ldots,t_{m_{\coarse}}$ be the coarse times in bin $j$ at which price $p$ is posted. Define
\begin{align*}
	\cF_k:=\sigma\Big(\{(u_s,\tilde u_s)\}_{s=1}^N, \{(p_{t_i},y_{t_i})\}_{i=1}^k\Big).
\end{align*}
Since $p_{t_k}=p$ and $\xi_{t_k}$ is independent of $\cF_{k-1}$, we have that 
\begin{align*}
	\E[p y_{t_k}\mid \cF_{k-1}]=p g(p-u_{t_k})=r(u_{t_k},p).
\end{align*}
The martingale differences $D_k:=p y_{t_k}-r(u_{t_k},p)$ are bounded by $2p_{\max}$. An application of the Azuma--Hoeffding inequality gives that with probability at least $1-CH^{-5}$,
\begin{align*}
	\biggabs{\widehat r_j(p)-\frac1{m_{\coarse}}\sum_{k=1}^{m_{\coarse}}r(u_{t_k},p)} \leq C p_{\max}\sqrt{\frac{\log(eH)}{m_{\coarse}}} \leq C p_{\max}m_0^{-1/2}.
\end{align*}
For any $u\in\overline B_j$, the containing-bin assignment leads to $\tilde u_{t_k}\in B_j\subset\overline B_j$, and Assumption~\ref{assump:orbit-interface} yields that 
\begin{align*}
	\abs{u-u_{t_k}} \leq \abs{u-\tilde u_{t_k}}+\abs{\tilde u_{t_k}-u_{t_k}} \leq h+\eta.
\end{align*}
Since $g$ is Lipschitz on $[-V,V]$, it holds that $\abs{r(u,p)-r(u_{t_k},p)}\leq p_{\max}L_g(h+\eta)$. Combining the last two expressions establishes~\eqref{eq:grid-mean-concentration} for the fixed pair $(j,p)$. Therefore, taking a union bound over all $M\abs{\cG}$ pairs gives the stated probability. This concludes the proof of Lemma~\ref{lem:mean-concentration}.

\subsection{Proof of Lemma~\ref{lem:main-upper-coarse-compressed}}

Denote by 
\begin{align*}
	 \Delta_{\rm mean}:=C_{\rm mean}p_{\max}\big(m_0^{-1/2}+L_g(h+\eta)\big).
\end{align*}
Under the stated choices of $h,\eta$, and $m_0$, by taking $c_0$ sufficiently small we have that 
\begin{align}\label{eq:coarse-proof-mean-small}
    \Delta_{\rm mean}\leq \frac{\sigma_r\rho_{\rm loc}^2}{512}.
\end{align}
We also take $c_0\leq 1/64$, so that $\eta_{\rm grid}=\rho_{\rm loc}^2\leq \rho_{\rm loc}/8$.
Fix any bin $j$ that has completed its coarse phase and any $u\in\overline B_j$. For any $p\in\cG$ satisfying that $\abs{p-p^\ast(u)}\geq \rho_{\rm loc}/8$, it follows from Assumption~\ref{assump:rho} that 
\begin{align*}
	r(u,p^\ast(u))-r(u,p) \geq \frac{\sigma_r\rho_{\rm loc}^2}{128}.
\end{align*}
Hence, combining this with~\eqref{eq:grid-mean-concentration} and~\eqref{eq:coarse-proof-mean-small} gives that 
\begin{align}\label{eq:coarse-proof-far}
    r(u,p^\ast(u))-\widehat r_j(p) \geq \frac{3\sigma_r\rho_{\rm loc}^2}{512}.
\end{align}

On the other hand, in view of the construction of $\cG$, there exists some $p^{\cG}(u)\in\cG$ with $\abs{p^{\cG}(u)-p^\ast(u)}\leq \eta_{\rm grid}$. Since $\eta_{\rm grid}=\rho_{\rm loc}^2$ and $\eta_{\rm grid}\leq c_0\sigma_r/L_r$, the upper quadratic-growth bound leads to 
\begin{align*}
	r(u,p^\ast(u))-r(u,p^{\cG}(u))\leq \frac{L_r}{2}\eta_{\rm grid}^2\leq \frac{\sigma_r\rho_{\rm loc}^2}{512}.
\end{align*}
Along with~\eqref{eq:grid-mean-concentration} and~\eqref{eq:coarse-proof-mean-small}, this entails that 
\begin{equation}\label{eq:coarse-proof-near}
    r(u,p^\ast(u))-\widehat r_j(p^{\cG}(u))\leq \frac{\sigma_r\rho_{\rm loc}^2}{256}.
\end{equation}
Comparing~\eqref{eq:coarse-proof-far} and~\eqref{eq:coarse-proof-near}, no grid point farther than $\rho_{\rm loc}/8$ from $p^\ast(u)$ can maximize $\widehat r_j$. The concentration event is uniform over $u\in\overline B_j$. Assume, to the contrary, that there exists some $u_0\in\overline B_j$ with $\abs{\tilde p_j-p^\ast(u_0)}>\rho_{\rm loc}/8$. Then applying the preceding comparison with $u=u_0$ shows that the stored grid maximizer $\tilde p_j$ cannot maximize $\widehat r_j$, contradicting its definition. Therefore, we can obtain that 
\begin{align*}
	\abs{\tilde p_j-p^\ast(u)}\leq \rho_{\rm loc}/8, \quad \forall u\in\overline B_j,
\end{align*}
which is exactly $\cE_{j,\coarse}$. This completes the proof of Lemma~\ref{lem:main-upper-coarse-compressed}.

\subsection{Proof of Lemma~\ref{lem:local-policy-approximation}}

By resorting to Lemma~\ref{lem:pstar-regularity}, the oracle price map has the required order-$(\beta-1)$ Taylor regularity. Denote by $m:=\lceil\beta-1\rceil-1$, and 
\begin{align*}
	T_j(u):= \sum_{k=0}^{m} \frac{(p^\ast)^{(k)}(\bar u_j )}{k!}(u-\bar u_j )^k
\end{align*}
the Taylor polynomial around $\bar u_j$. For all $\tilde u\in\overline B_j$, it holds that 
\begin{align}\label{eq:proof-lem-approx-taylor}
    \abs{T_j(\tilde u)-p^\ast(\tilde u)}\leq L_p \bar h^{\beta-1}\leq L_p h^{\beta-1}.
\end{align}
Define $a_j^{\rm cmp}\in\R^{q+1}$ by
\begin{align*}
	a_{j,k}^{\rm cmp} =\frac{(p^\ast)^{(k)}(\bar u_j)}{k!}\left(\frac{\bar h}{2}\right)^k,
    \quad k=0,\ldots,m, \qquad a_{j,k}^{\rm cmp}=0, \quad k=m+1,\ldots,q .
\end{align*}
Then we have that $T_j(\tilde u)=\frakq^j(\tilde u;a_j^{\rm cmp})$ for each $\tilde u\in\overline B_j$. When $\beta-1$ is an integer, the degree $q=\lfloor\beta-1\rfloor$ is one larger than $m$; the extra highest-order coefficient is set to zero and is thus harmless. To verify $a_j^{\rm cmp}\in\cA_j$, let us fix any $\abs{z}\leq 1$ and set $u_z:=\bar u_j+\bar h z/2\in\overline B_j$. Under $\cE_{j,\coarse}$ and $h^{\beta-1}\leq\rho_{\rm loc}/(8L_p)$, it holds that 
\begin{align*}
	\abs{(a_j^{\rm cmp})^\top\psi(z)-\tilde p_j} =\abs{T_j(u_z)-\tilde p_j} \leq \abs{T_j(u_z)-p^\ast(u_z)}+\abs{p^\ast(u_z)-\tilde p_j} \leq \rho_{\rm loc}/4.
\end{align*}
Consequently, taking the supremum over $\abs{z}\leq 1$ gives that $a_j^{\rm cmp}\in\cA_j$.

For any $u,\tilde u$ satisfying the lemma conditions, we have that 
\begin{align*}
	\abs{T_j(\tilde u)-p^\ast(u)} \leq \abs{T_j(\tilde u)-p^\ast(\tilde u)}+\abs{p^\ast(\tilde u)-p^\ast(u)} \leq   L_p h^{\beta-1}+L_p\eta.
\end{align*}
Further, applying Lemma~\ref{lem:rho-exists} at the scalar index $\tilde u$ and using $\abs{T_j(\tilde u)-p^\ast(\tilde u)}\leq \rho_{\rm loc}/8\leq \rho_0$, we can show that $T_j(\tilde u)\in(0,p_{\max})$. Therefore, the upper quadratic-growth bound in Assumption~\ref{assump:rho} is applicable and yields that 
\begin{align*}
	 r(u,p^\ast(u))-r(u,T_j(\tilde u)) \leq \frac{L_r}{2}\big(L_p h^{\beta-1}+L_p\eta\big)^2 \leq C_{\rm approx}(h^{2\beta-2}+\eta^2).
\end{align*}
Substituting $T_j(\tilde u)=\frakq^j(\tilde u;a_j^{\rm cmp})$ establishes the conclusion of the lemma. This concludes the proof of Lemma~\ref{lem:local-policy-approximation}.

\subsection{Proof of Proposition~\ref{prop:induced-bco-equivalence}}\label{subsec: appendix-bco-instance}

We will work under a fixed model instance $(\mu_\ast,g)$. Conditional on $\cH_j$, the bin-$j$ coarse anchor $\tilde p_j$, trust region $\cA_j$, refinement indices $\{s_{j,k}\}_{k\leq n_j}$, pilot states $\{\tilde u_{s_{j,k}}\}_{k\leq n_j}$, and losses $\{L_k^j\}_{k\leq n_j}$ are fixed.

First, $\cA_j$ is convex since it is defined by linear inequalities. Let $a^{\rm ctr}:=(\tilde p_j,0,\ldots,0)$. It follows from the finite-dimensional norm equivalence
\begin{equation}\label{eq:equivalence-norm}
    c_q\|a\|_2 \leq \sup_{\abs{z}\leq1}\abs{a^\top\psi(z)} \leq C_q\|a\|_2
\end{equation}
with constants depending only on $q$ that 
\begin{align*}
  \{a:\|a-a^{\rm ctr}\|_2\leq C_q^{-1}\rho_{\rm loc}/4\} \subset \cA_j\subset
  \{a:\|a-a^{\rm ctr}\|_2\leq c_q^{-1}\rho_{\rm loc}/4\}.
\end{align*}
Hence, $\cA_j$ is compact and has nonempty interior.

Next, on event $\cE_{j,\coarse}$, we have that for each refinement index $s=s_{j,k}$ and each $a\in\cA_j$,
\begin{align*}
\begin{aligned}
    \abs{\frakq^j(\tilde u_s;a)-p^\ast(u_s)} \leq \abs{\frakq^j(\tilde u_s;a)-\tilde p_j} + \abs{\tilde p_j-p^\ast(\tilde u_s)} + \abs{p^\ast(\tilde u_s)-p^\ast(u_s)} \leq \rho_{\rm loc}/4+\rho_{\rm loc}/8+L_p\eta \leq \rho_{\rm loc}/2 \leq \rho_0 .
\end{aligned}
\end{align*}
Thus, the candidate polynomial price remains inside the local-concavity neighborhood from Lemma~\ref{lem:rho-exists}. Consequently, function $p\mapsto -r(u_s,p)$ is convex on the relevant price interval. Since $a\mapsto \frakq^j(\tilde u_s;a)$ is linear, function $a\mapsto L_k^j(a)$ is convex on $\cA_j$. Moreover, it holds that 
\begin{align*}
    \abs{\frakq^j(\tilde u_s;a)-p^\ast(\tilde u_s)} \leq \rho_{\rm loc}/4+\rho_{\rm loc}/8 <
    \rho_0,
\end{align*}
so Lemma~\ref{lem:rho-exists} also implies that  $\frakq^j(\tilde u_s;a)\in(0,p_{\max})$. Hence, the projection step in Algorithm~\ref{alg:oplcb} is inactive on these refinement rounds. In particular, we have that $L_k^j(a)\in[-p_{\max},0]$ for all $a\in\cA_j$.

It remains to verify the conditional unbiasedness of the raw feedback. Let
\begin{align*}
    \mathcal F_{j,k-1}:=\sigma(\cH_j,\cD_{j,k-1}^{\bco})
\end{align*}
be the information available to the bin-$j$ refinement generator before choosing its $k$th coefficient. The action $a_k^j$ is selected using this local history and fresh internal randomization, before the current demand noise $\xi_{s_{j,k}}$ is observed. By Assumption~\ref{assump:orbit-interface} and the exogeneity of the demand noises, $\xi_{s_{j,k}}$ is independent of $\sigma(\mathcal F_{j,k-1},a_k^j)$. Since the projection is inactive, we can deduce that 
\begin{align*}
\begin{aligned}
    \E\left[ \ell_k^j \,\middle|\, \sigma(\mathcal F_{j,k-1},a_k^j) \right] &= -\frakq^j(\tilde u_{s_{j,k}};a_k^j) \E\left[ \mathbf 1\left\{ u_{s_{j,k}}+\xi_{s_{j,k}} \geq \frakq^j(\tilde u_{s_{j,k}};a_k^j) \right\} \,\middle|\, \sigma(\mathcal F_{j,k-1},a_k^j) \right] \\
    &= -\frakq^j(\tilde u_{s_{j,k}};a_k^j) g\left( \frakq^j(\tilde u_{s_{j,k}};a_k^j)-u_{s_{j,k}} \right) =L_k^j(a_k^j).
\end{aligned}
\end{align*}
Hence, the raw feedback is conditionally unbiased for the induced loss.

Finally, it follows from the definition of $L_k^j$ that 
\begin{align*}
\begin{aligned}
    \sum_{k=1}^{n_j}L_k^j(a_k^j) - \min_{a\in\cA_j}\sum_{k=1}^{n_j}L_k^j(a) = \max_{a\in\cA_j} \sum_{k=1}^{n_j} \Big[ r\big(u_{s_{j,k}},\frakq^j(\tilde u_{s_{j,k}};a)\big) - r\big(u_{s_{j,k}},\frakq^j(\tilde u_{s_{j,k}};a_k^j)\big) \Big] =\cR_j^{\rm learn}.
\end{aligned}
\end{align*}
Under the identification $a\leftrightarrow \frakq^j(\cdot;a)$, the BCO regret is thus exactly the binwise learning regret. This completes the proof of Proposition~\ref{prop:induced-bco-equivalence}.

\subsection{Proof of Lemma~\ref{lem:main-upper-refine}}
\subsubsection{The horizon-dependent online-Newton BCO guarantee.}\label{app:lg-theorem}

This subsection verifies that the raw anytime refinement contract in Definition~\ref{def:raw-refinement-generator} is implemented by the online-Newton method for noisy bandit convex optimization. We will exploit the adversarial guarantee in Theorem~1 of \cite{fokkema2024online}. In the notation below, the theorem provides a high-probability regret bound of order $d_b^{7/2}\sqrt n$ up to logarithmic factors for convex losses with noisy value feedback. The losses are allowed to be adversarial after conditioning on a background $\sigma$-algebra, and the feedback noise is conditionally sub-Gaussian. The cited method is an online-Newton BCO algorithm; in our application, it is run only after an affine normalization of the binwise coefficient set.

\begin{proposition}[Horizon-dependent online-Newton BCO routine]\label{prop:noisy-bco-online-newton}
For each horizon $n\geq1$, ambient dimensionality $d_b\geq1$, failure level $\delta\in(0,1)$, convex set $\cA\subset\R^{d_b}$, $D\geq1$, and $G\geq0$ satisfying that 
\begin{align}\label{eq:normalized-scales}
    \mathbb{B}^{d_b}_2(1)  \subset \cA \subset \mathbb{B}^{d_b}_2(D),
\end{align}
and admitting the action-set oracle access required by the online-Newton BCO algorithm, there exists a horizon-dependent online-Newton BCO routine $\mathfrak B_{n,\delta}^{\mathrm{ON}}$ such that the following holds. For any background $\sigma$-algebra $\cH$, any sequence of $\cH$-measurable convex $G$-Lipschitz losses $L_k:\cA\to[0,1]$, and any feedback sequence satisfying that 
\begin{align*}
    \ell_k=L_k(a_k)+\epsilon_k, \qquad \E[\epsilon_k\mid \sigma(\cH,a_1,\ell_1,\ldots,a_{k-1},\ell_{k-1},a_k)]=0,
\end{align*}
with conditionally sub-Gaussian $\epsilon_k$ whose sub-Gaussian proxy is bounded by a universal constant, the routine satisfies the high-probability regret bound conditional on $\cH$ 
\begin{align}\label{eq:on-newton-high-prob}
    \Prob\left( \cR^\bco(n) \leq C_F d_b^{7/2}\sqrt n\,\mathrm{polylog}(n,d_b,D,1+G,1/\delta) \ \middle\vert\ \cH  \right)  \geq 1-\delta,
\end{align}
where $C_F$ is a universal constant. In particular, if the routine is run with $\delta_n:=(en)^{-2}$, we have that 
\begin{align}\label{eq:on-newton-expected}
    \E\left[\cR^\bco(n)\mid \cH\right] \leq C_F' d_b^{7/2}\sqrt n\,\mathrm{polylog}(n,d_b,D,1+G)
\end{align}
for another universal constant $C_F'$.
\end{proposition}

\textit{Proof}. 
Let us condition on $\cH$. The loss sequence is then fixed and can be treated as an oblivious adversarial convex loss sequence. By the assumed feedback representation, the learner observes a noisy value $L_k(a_k)+\epsilon_k$ with conditionally mean-zero sub-Gaussian noise. Applying the adversarial online-Newton guarantee in Theorem~1 of \cite{fokkema2024online} leads to \eqref{eq:on-newton-high-prob}. To obtain the conditional expectation bound, let us choose $\delta_n=(en)^{-2}$. Since $L_k\in[0,1]$, the regret is deterministically bounded by $n$. Hence, it holds that 
\begin{align*}
    \E\left[\cR^\bco(n)\mid\cH\right]
    &\leq C_F d_b^{7/2}\sqrt n\,\mathrm{polylog}(n,d_b,D,1+G,1/\delta_n)+n\delta_n \\
    &\leq C_F' d_b^{7/2}\sqrt n\,\mathrm{polylog}(n,d_b,D,1+G),
\end{align*}
since $n\delta_n\leq1$ is absorbed into the displayed rate. The argument is conditional on an arbitrary realization of $\cH$, so the conditional claims follow. This concludes the proof of Proposition~\ref{prop:noisy-bco-online-newton}.

Proposition~\ref{prop:noisy-bco-online-newton} above is the online-Newton implementation theorem used by $\ORBIT$'s refinement stage. In our normalized bin instances, the shifted feedback lies in $[0,1]$ and is conditionally unbiased for the shifted loss. Hence, the noise $\bar\ell_k-\bar L_k(\bar a_k)$ is bounded by one and thus conditionally sub-Gaussian with a universal proxy. The normalization and anytime arguments below will show how $\ORBIT$ calls the online-Newton routine on the raw coefficient scale.

Two tasks remain before Proposition~\ref{prop:noisy-bco-online-newton} can be invoked for the BCO instance constructed in Proposition~\ref{prop:induced-bco-equivalence}: i) the action set, losses, feedback, and scale parameters must be normalized to satisfy~\eqref{eq:normalized-scales} with structural constants and provide the required bounded conditionally sub-Gaussian feedback; and ii) since $n_j$ is unknown to the learner, the horizon-dependent online-Newton routine must be wrapped in a prefix-valid anytime scheme.

The second task can be resolved by the doubling trick in Appendix~\ref{app:proof-of-prop:scb}. We will first address the normalization task.

\noindent\textbf{1. Normalization of action set.}\quad 
Let $D_q := C_q/c_q$, where $c_q,C_q$ are the norm-equivalence constants from~\eqref{eq:equivalence-norm}; for the monomial basis $\psi(z) = (1, z, \dots, z^q)^\top$, $D_q$ is a finite constant depending only on $q$. Define the affine map $T_a:\R^{q+1}\to\R^{q+1}$ through 
\begin{align*}
    T_a(a) := \frac{4 C_q}{\rho_{\rm loc}}\big(a-a^{\rm ctr}\big),\qquad \bar{\cA}_j := T_a(\cA_j).
\end{align*}
By the inclusions established above, it holds that 
\begin{align*}
    \mathbb{B}^{q+1}_2(1) \subset \bar{\cA}_j \subset \mathbb{B}^{q+1}_2(D_q),
\end{align*}
which matches~\eqref{eq:normalized-scales} with $D = D_q$. The transformed action set admits the oracle access needed by the online-Newton routine. Indeed, membership of $\bar a$ is equivalent to checking the explicit inequality defining $\cA_j$ for $T_a^{-1}(\bar a)$; since $q$ is fixed, this reduces to maximizing a univariate polynomial on $[-1,1]$. Since $T_a$ is a bijection, every adapted sequence $\{a_k\}\subset\cA_j$ corresponds to an adapted sequence $\{\bar a_k\}=\{T_a(a_k)\}\subset \bar{\cA}_j$, and convexity of each loss is preserved under the affine pullback $L_k^j\circ T_a^{-1}$.

\noindent\textbf{2. Normalization of loss sequence.}\quad
The losses and feedback are bounded uniformly: since $r(u,p)=p\,g(p-u)\in[0,p_{\max}]$ for $p\in[0,p_{\max}]$ and $g\in[0,1]$, we have that 
\begin{align*}
    L_k^j(a) = -r\big(u_{s_{j,k}},\frakq^j(\tilde u_{s_{j,k}};a)\big) \in [-p_{\max},0],\qquad \ell_k = -p_{s_{j,k}}\,y_{s_{j,k}} \in [-p_{\max},0].
\end{align*}
Define the affine shift-and-scale
\begin{align*}
    \bar L_k^j(\bar a) := \frac{L_k^j(T_a^{-1}(\bar a))+p_{\max}}{p_{\max}},\qquad \bar\ell_k := \frac{\ell_k+p_{\max}}{p_{\max}},
\end{align*}
so that $\bar L_k^j,\bar\ell_k\in[0,1]$, convexity is preserved, and the unbiasedness condition~\eqref{eq:bco-noise-requirement} carries over by linearity
\begin{align*}
    \E[\bar\ell_k\mid \sigma(\cH_j,\bar a_1,\bar\ell_1,\dots,\bar a_{k-1},\bar\ell_{k-1},\bar a_k)] = \frac{\E[\ell_k\mid\cdot]+p_{\max}}{p_{\max}} = \bar L_k^j(\bar a_k).
\end{align*}
Moreover, we have that $\bar\ell_k\in[0,1]$ and $\bar L_k^j(\bar a_k)\in[0,1]$, so $\bar\ell_k-\bar L_k^j(\bar a_k)$ is bounded in $[-1,1]$ and conditionally sub-Gaussian with a universal proxy. Consequently, the feedback requirement in Proposition~\ref{prop:noisy-bco-online-newton} is satisfied.

The normalized BCO regret relates to the original by $\cR_j^{\bco}(n_j) = p_{\max}\cdot \bar\cR_j^{\bco}(n_j)$, since the additive shift $p_{\max}$ cancels in the regret difference. The normalized losses are Lipschitz with a structural constant. Indeed, on the relevant price domain, it holds that 
\begin{align*}
    \abs{\partial_p r(u,p)}=\abs{g(p-u)+p g'(p-u)}\leq 1+p_{\max}L_g,
\end{align*}
and $\|\psi(z)\|_2\leq\sqrt{q+1}$ for $\abs{z}\leq1$, so $\|\nabla_a L_k^j(a)\|_2\leq(1+p_{\max}L_g)\sqrt{q+1}$ before normalization. Since $T_a^{-1}(\bar a)=a^{\rm ctr}+(\rho_{\rm loc}/(4C_q))\bar a$, the normalized Lipschitz constant is at most
\begin{align*}
    G_q:=\frac{(1+p_{\max}L_g)\sqrt{q+1}}{p_{\max}}\cdot\frac{\rho_{\rm loc}}{4C_q},
\end{align*}
which is a fixed structural constant for the online-Newton guarantee. Thus, if the realized horizon $n_j$ were known in advance and the online-Newton routine were run with that horizon and failure level $\delta_{n_j}=(en_j)^{-2}$, Proposition~\ref{prop:noisy-bco-online-newton} with $d_b=q+1$, $D=D_q$, and $G=G_q$ would give that 
\begin{align}\label{eq:bco-regret-normalized-bound}
    \E\big[\cR_j^{\bco}(n_j)\big]\leq p_{\max}\cdot C_F'\, (q+1)^{7/2}\sqrt{n_j}\,\mathrm{polylog}(n_j,q,D_q,1+G_q).
\end{align}
This expression is only a scale calculation for the normalized instance. The actual $\ORBIT$ implementation does \textit{not} know $n_j$ and thus uses the \textit{anytime wrapper} described next.

\subsubsection{Anytime variant of the online-Newton BCO routine.}\label{app:proof-of-prop:scb}

Throughout this section, we will work under the normalization~\eqref{eq:normalized-scales}: the action set $\cA\subset\R^{q+1}$ satisfies
that $\mathbb{B}^{q+1}_2(1)\subset\cA\subset\mathbb{B}^{q+1}_2(D)$, and each loss
$L_k:\cA\to[0,1]$ is convex and $G$-Lipschitz, with unbiased $[0,1]$-valued bandit feedback
$\ell_k$ in the sense of Definition~\ref{def:bco-instance}. Since the feedback and the loss values both lie in $[0,1]$, the feedback noise is conditionally sub-Gaussian with a universal proxy. We call a copy of $\mathfrak B_{n,\delta}^{\mathrm{ON}}$ a \emph{fresh online-Newton copy} when it is initialized with empty internal history $\cD_0^{\bco}=\varnothing$, its epoch length $n$ is fixed in advance, and its failure level $\delta$ is fixed before the epoch starts.

\begin{algorithm}[t]
\caption{Normalized anytime online-Newton BCO wrapper $\mathfrak B^{\bco}_{\mathrm{any}}(\cA)$ via the doubling trick}\label{alg:bco-anytime}
\begin{algorithmic}[1]
\State \textbf{Input:} action set $\cA\subset\R^{q+1}$ with $\mathbb{B}^{q+1}_2(1)\subset\cA\subset\mathbb{B}^{q+1}_2(D)$; horizon-dependent online-Newton BCO routines $\{\mathfrak B_{n,\delta}^{\mathrm{ON}}:n\geq1,\delta\in(0,1)\}$ from Proposition~\ref{prop:noisy-bco-online-newton}.
\State Initialize epoch index $r\gets 0$ and global step counter $t\gets 0$.
\Loop
    \State Set epoch length $n_r\gets 2^r$ and failure level $\delta_r\gets (en_r)^{-2}$.
    \State Spawn a fresh online-Newton copy $\pi_r\gets \mathfrak B_{n_r,\delta_r}^{\mathrm{ON}}$ with empty history $\cD_0^{\bco,(r)}=\varnothing$.
    \For{$k=1,\dots,n_r$}
        \State $t\gets t+1$.
        \State Query $\pi_r$ for action $a_t\in\cA$ given history $\cD_{k-1}^{\bco,(r)}$.
        \State Play $a_t$; environment incurs loss $L_t(a_t)$.
        \State Observe noisy feedback $\ell_t$ satisfying~\eqref{eq:bco-noise-requirement}.
        \State Set $\cD_k^{\bco,(r)}$ to the ordered history obtained by appending $(a_t,\ell_t)$ to $\cD_{k-1}^{\bco,(r)}$.
    \EndFor
    \State Discard $\pi_r$; $r\gets r+1$.
\EndLoop
\end{algorithmic}
\end{algorithm}

\begin{proposition}[Anytime regret of the doubling-trick online-Newton wrapper]\label{prop:scb}
Given any BCO instance described in Definition~\ref{def:bco-instance} with action-set normalization~\eqref{eq:normalized-scales}, losses and feedback in $[0,1]$, and $G$-Lipschitz convex losses, Algorithm~\ref{alg:bco-anytime}, instantiated with the horizon-dependent online-Newton routines from Proposition~\ref{prop:noisy-bco-online-newton}, satisfies that for each $n\geq 1$,
\begin{align}\label{eq:scb-anytime-bound}
    \E[\cR^\bco(n)\mid \cH] \leq  \widetilde C (q+1)^{7/2}\sqrt{n}\,\mathrm{polylog}(n,q,D,1+G)
\end{align}
for some absolute constant $\widetilde C>0$. Consequently, the same unconditional bound holds after taking expectations.
\end{proposition}

\textit{Proof}. 
Let us write $r_n:=\floor{\log_2 n}$ and $\ell_n:=n-(2^{r_n}-1)\in[1,2^{r_n}]$. The wrapper completes the dyadic blocks of lengths $1,2,\dots,2^{r_n-1}$, followed by the first $\ell_n$ rounds of a fresh online-Newton copy of epoch length $2^{r_n}$. The regret against the best single action over all $n$ rounds is upper bounded by the sum of the block regrets since
\begin{align*}
    \sum_b \min_{a\in\cA}\sum_{t\in b}L_t(a)\leq \min_{a\in\cA}\sum_{t=1}^n L_t(a).
\end{align*}
For a completed block of length $2^r$, we will apply the conditional expected bound~\eqref{eq:on-newton-expected} from Proposition~\ref{prop:noisy-bco-online-newton} conditional on the $\sigma$-algebra generated by $\cH$ and all histories from earlier blocks. The loss functions in the current block are still fixed under this enlarged background information, and the future feedback in the block remains conditionally unbiased and bounded, thereby conditionally sub-Gaussian after subtracting its conditional mean. Taking a tower expectation leads to the same conditional bound given $\cH$. For the last incomplete block, we will apply the horizon-$2^{r_n}$ guarantee to the artificial loss sequence obtained by appending $2^{r_n}-\ell_n$ zero losses after the first $\ell_n$ real losses, with deterministic zero feedback on the padded suffix. The algorithm's first $\ell_n$ actions are unchanged under such padding, and the comparator loss on the padded suffix is also zero; hence, the prefix regret is bounded by the same $O(\sqrt{2^{r_n}})$ guarantee. Therefore, we can obtain that 
\begin{align*}
    \E[\cR^\bco(n)\mid\cH] &\leq \sum_{r=0}^{r_n} C_F' (q+1)^{7/2}\sqrt{2^r}\,\mathrm{polylog}(n,q,D,1+G) \\
    &\leq \widetilde C (q+1)^{7/2}\sqrt n\,\mathrm{polylog}(n,q,D,1+G),
\end{align*}
which establishes~\eqref{eq:scb-anytime-bound}. This completes the proof of Proposition~\ref{prop:scb}.

\subsubsection{Raw-scale wrapper used by $\ORBIT$.}\label{app:raw-bco-wrapper}
The $\ORBIT$ pseudocode is written entirely on the original coefficient scale: it initializes a binwise generator on $\cA_j$, receives an action $a_k^j\in\cA_j$, and records the raw feedback $\ell_k^j=-p_{s_{j,k}}y_{s_{j,k}}$. The following \textit{wrapper} is the formal implementation of that interface.

For a bin $j$, let $a^{\rm ctr}=(\tilde p_j,0,\ldots,0)$, and $T_a(a)=4C_q(a-a^{\rm ctr})/\rho_{\rm loc}$ the affine map defined above. The wrapper stores an internal copy of Algorithm~\ref{alg:bco-anytime} on the transformed set $\bar\cA_j=T_a(\cA_j)$. When queried at an ordered raw local history $((a_i,\ell_i))_{i<k}$, it converts that history into
\begin{align*}
    \left(\left(T_a(a_i),\frac{\ell_i+p_{\max}}{p_{\max}}\right)\right)_{i<k},
\end{align*}
queries the internal normalized copy for $\bar a_k\in\bar\cA_j$, and returns the raw action $a_k=T_a^{-1}(\bar a_k)$. After the pricing algorithm observes raw feedback $\ell_k$, the wrapper sends $(\ell_k+p_{\max})/p_{\max}$ to its internal copy. Thus, the main algorithm never needs to manipulate the transformed action set or the shifted feedback directly. When indexed by a global upper bound $H_{\rm alg}$, this raw-scale wrapper is the concrete family denoted as $\cB_{\rm raw}^\bco(H_{\rm alg})$ in Algorithm~\ref{alg:oplcb}.

\begin{proposition}[Guarantee of the raw-scale online-Newton wrapper]\label{prop:raw-bco-wrapper}
Consider any bin $j$ for which the conditions of Proposition~\ref{prop:induced-bco-equivalence} hold, and assume that the binwise refinement actions are generated by the raw-scale wrapper just described with upper pilot-input budget $H_{\rm alg}\geq n_j$. Its internal normalized copy uses Algorithm~\ref{alg:bco-anytime} and the horizon-dependent online-Newton routines from Proposition~\ref{prop:noisy-bco-online-newton}. Then we have that conditional on any background $\sigma$-algebra $\cH$ with respect to which the induced losses $\{L_k^j\}_{k\leq n_j}$ and $n_j$ are fixed and the future bin-$j$ feedback remains conditionally unbiased, 
\begin{align*}
    \E\big[\cR_j^{\rm learn}\mid \cH\big] \leq C_{\rm raw}\sqrt{n_j}\,\mathrm{polylog}(H_{\rm alg},\beta),
\end{align*}
where $C_{\rm raw}$ depends only on fixed structural constants and 
the constant in Proposition~\ref{prop:noisy-bco-online-newton}. If $\mathcal G$ is $\cH$-measurable and the conditions hold on $\mathcal G$, the corresponding localized bound $\E[\mathbf 1_{\mathcal G}\cR_j^{\rm learn}\mid \cH]\leq \mathbf 1_{\mathcal G}C_{\rm raw}\sqrt{n_j}\,\mathrm{polylog}(H_{\rm alg},\beta)$ also holds. Consequently, the raw-scale wrapper, as a family indexed by $H_{\rm alg}$, is $\ORBIT$-compatible in the raw anytime sense of Definition~\ref{def:raw-refinement-generator}.
\end{proposition}

\textit{Proof}. On the event when Proposition~\ref{prop:induced-bco-equivalence} applies, the original binwise refinement problem is a valid BCO instance over $\cA_j$ with losses $L_k^j\in[-p_{\max},0]$ and raw feedback $\ell_k^j\in[-p_{\max},0]$. The affine map $T_a$ sends $\cA_j$ to a set $\bar\cA_j$ satisfying that 
\begin{align*}
    \mathbb B_2^{q+1}(1)\subset \bar\cA_j\subset \mathbb B_2^{q+1}(D_q),
\end{align*}
and the shifted losses and feedback
\begin{align*}
    \bar L_k^j(\bar a)=\frac{L_k^j(T_a^{-1}(\bar a))+p_{\max}}{p_{\max}}, \qquad \bar\ell_k^j=\frac{\ell_k^j+p_{\max}}{p_{\max}}
\end{align*}
lie in $[0,1]$, satisfy the same conditional-unbiasedness identity, and have bounded conditionally sub-Gaussian noise after subtracting the conditional mean. The Lipschitz bound derived above gives a structural constant $G_q$ for the normalized losses. It follows from Proposition~\ref{prop:scb} that the internal copy of Algorithm~\ref{alg:bco-anytime} has expected normalized regret at most
\begin{align*}
    C(q+1)^{7/2}\sqrt{n_j}\,\mathrm{polylog}(H_{\rm alg},q,D_q,1+G_q).
\end{align*}
The additive shift by $p_{\max}$ cancels from the regret difference and the scale factor contributes exactly $p_{\max}$, so the original raw regret is $p_{\max}$ times the normalized regret. Since $q,D_q,G_q$ are fixed once $\beta$ and the structural constants are fixed and $n_j\leq H_{\rm alg}$, these factors are absorbed into $C_{\rm raw}\mathrm{polylog}(H_{\rm alg},\beta)$. If a background-measurable event indicator is multiplied in front of the regret, the same conditional bound applies on that event and the claim follows by the tower property. This concludes the proof of Proposition~\ref{prop:raw-bco-wrapper}.

We are now ready to present the proof of Lemma~\ref{lem:main-upper-refine} by putting all the discussions above together.

\textit{Proof of Lemma~\ref{lem:main-upper-refine}}. Let $\cH_{\coarse}^{\rm all}$ be the $\sigma$-algebra generated by the associated index--pilot sequence and all coarse observations from all bins. Conditional on $\cH_{\coarse}^{\rm all}$, the quantities $n_j$, $\tilde p_j$, $\cA_j$, and the future refinement times in bin $j$ are fixed, since bin assignments and local clocks are functions only of the pilot states. On event $\cE_{j,\coarse}$, Proposition~\ref{prop:induced-bco-equivalence} constructs a valid BCO instance for bin $j$. Enlarging the background information from $\cH_j$ to $\cH_{\coarse}^{\rm all}$ does not affect validity: the bin-specific raw refinement generator copy used by Algorithm~\ref{alg:oplcb} is initialized with $\cA_j$ and updated only with the bin-$j$ ordered raw history, while future bin-$j$ demand noises remain independent of $\cH_{\coarse}^{\rm all}$ and 
the fresh internal randomization. Since the run uses $\cB_{\rm raw}^\bco(H)$ and $n_j\leq N\leq H$, the contract in Definition~\ref{def:raw-refinement-generator} applies with $H_{\rm alg}=H$.

Let $\mathcal G\subseteq\cE_{j,\coarse}$ be any $\cH_{\coarse}^{\rm all}$-measurable event. Applying the $\ORBIT$-compatible raw anytime guarantee in Definition~\ref{def:raw-refinement-generator} with background $\sigma$-algebra $\cH_{\coarse}^{\rm all}$ gives that 
\begin{align*}
    \E[\bm 1_{\mathcal G}\cR_j^{\rm learn}\mid \cH_{\coarse}^{\rm all}] \leq \bm 1_{\mathcal G} C_{\bco}\sqrt{n_j}\,\mathrm{polylog}(H,\beta).
\end{align*}
Therefore, taking the conditional expectation with respect to the associated index--pilot sequence yields the general claim; choosing $\mathcal G=\cE_{j,\coarse}$ leads to the displayed bound in the first part of the lemma. This completes the proof of Lemma~\ref{lem:main-upper-refine}.

\medskip
\section{Proofs for the upper bound in Section~\ref{subsec:adaptive-linear-pilot}}\label{app:sec4-proofs}

In this appendix, we will prove the online pilot-estimation guarantee from Section~\ref{subsec:adaptive-linear-pilot} and then combine it with the core upper bound.

\subsection{Proof of Lemma~\ref{lem:linear-confidence}}

We condition on the context sequence $\{\mc_t\}_{t=1}^T$. The exploration set is then deterministic since the trigger $w_t=C_w\|\mc_t\|_{\bA_t^{-1}}$ depends only on the contexts and previous exploration indicators. For each exploration time $\tau$, the posted price is uniform on $[0,p_{\max}]$, and independent of $(\mc_\tau,\xi_\tau)$ and 
the past. By the pseudo-response identity established in Section~\ref{subsec:adaptive-linear-pilot}, it holds that 
\begin{align*}
    \E[p_{\max}y_\tau\mid \mc_\tau]=\mc_\tau^\top\theta_\ast .
\end{align*}
Denote by $\zeta_\tau:=p_{\max}y_\tau-\mc_\tau^\top\theta_\ast$. Since $p_{\max}y_\tau\in\{0,p_{\max}\}$ and $0\leq \mc_\tau^\top\theta_\ast\leq C_\theta$, we have that $\abs{\zeta_\tau}\leq p_{\max}+C_\theta$. With the filtration taken just before the uniform exploration price at time $\tau$ is drawn, $\E[\zeta_\tau\mid\text{past},\mc_\tau]=0$ for exploration times. It follows from 
\begin{align*}
    \hat\theta_t-\theta_\ast =-\bA_t^{-1}\theta_\ast+\bA_t^{-1}\sum_{\tau\in\cT_{t-1}^{\mathsf{exp}}}\mc_\tau\zeta_\tau
\end{align*}
that 
\begin{align*}
    \abs{\mc_t^\top \bA_t^{-1}\theta_\ast} \leq \|\mc_t\|_{\bA_t^{-1}}\|\theta_\ast\|_{\bA_t^{-1}} \leq C_\theta\|\mc_t\|_{\bA_t^{-1}}.
\end{align*}
For the martingale term, an application of the Azuma--Hoeffding inequality yields that for each fixed $t$, with probability at least $1-CT^{-4}$,
\begin{align*}
    \abs{\mc_t^\top \bA_t^{-1}\sum_{\tau\in\cT_{t-1}^{\mathsf{exp}}}\mc_\tau\zeta_\tau} \leq C (p_{\max}+C_\theta)\sqrt{\log(eT)}\left(\sum_{\tau\in\cT_{t-1}^{\mathsf{exp}}}(\mc_t^\top \bA_t^{-1}\mc_\tau)^2\right)^{1/2} \leq C (p_{\max}+C_\theta)\sqrt{\log(eT)}\,\|\mc_t\|_{\bA_t^{-1}},
\end{align*}
where the last inequality has used $\sum_{\tau\in\cT_{t-1}^{\mathsf{exp}}}\mc_\tau\mc_\tau^\top\preceq \bA_t$. Therefore, taking a union bound over $t\leq T$ and using the definition of $C_w$ establish the desired result. This concludes the proof of Lemma~\ref{lem:linear-confidence}.

\subsection{Proof of Proposition~\ref{prop:adaptive-init-guarantee}}

For the first statement, observe that the exploration indicator is determined by $w_t$, which is a deterministic function of the context sequence and previous exploration indicators. Hence, $\cT^\ORBIT$ is context-measurable. Conditional on the full context sequence and 
the algorithmic randomization used to draw exploration prices, the two subfamilies $\{\xi_t:t\in\cT_T^{\mathsf{exp}}\}$ and $\{\xi_t:t\in\cT^\ORBIT\}$ are independent since the demand noises are independent across periods. The pilot functions on $\ORBIT$ rounds are measurable with respect to the contexts, exploration prices, and exploration-round observations, and thus depend on demand noise only through $\{\xi_t:t\in\cT_T^{\mathsf{exp}}\}$. Since the pilot update never uses $\ORBIT$ prices, $\ORBIT$ outcomes, or $\ORBIT$ internal randomization, the same associated index--pilot sequence is also independent of $\ORBIT$'s internal refinement randomization. This establishes the stated independence condition for the $\ORBIT$ interface.

For the second statement, on the event in Lemma~\ref{lem:linear-confidence}, each $t\in\cT^\ORBIT$ satisfies that $w_t\leq\eta$, and thus 
\begin{align*}
    \abs{\tilde u_t-u_t} =\abs{\proj_{\cU}(\mc_t^\top\hat\theta_t)-\mc_t^\top\theta_\ast} \leq \abs{\mc_t^\top(\hat\theta_t-\theta_\ast)} \leq w_t\leq\eta,
\end{align*}
where the projection inequality above has utilized $u_t=\mc_t^\top\theta_\ast\in\cU$.

For the third statement, we will exploit the following standard core-set bound.
\begin{lemma}[Lemma~5.1 of \citet{yin2022efficient}]\label{lem:core-set-size}
Given $\lambda,\bar\eta > 0$ and an arbitrary sequence $\{\mc_t\}_{t=1}^T$ with $\|\mc_t\|_2\leq1$, initialize $\cC_{\mathsf{core}} = \varnothing$ and consider the following procedure for expanding $\cC_{\mathsf{core}}$: for $t=1,\dots,T$,
\begin{enumerate}
    \item[1)] If $\mc_t^\top(\sum_{\mc \in \cC_{\mathsf{core}}} \mc\mc^\top + \lambda \bI)^{-1}\mc_t > \bar\eta,$ add $\mc_t$ to $\cC_{\mathsf{core}}$.
    \item[2)] Otherwise, keep $\cC_{\mathsf{core}}$ unchanged.
\end{enumerate}
Then we have that 
\begin{align*}
    \abs{\cC_{{\mathsf{core}}}} \leq C d\bar\eta^{-1}\log(1+\bar\eta^{-1})
\end{align*}
for a universal constant $C$ when $\lambda=1$ and $\|\mc_t\|_2\leq1$.
\end{lemma}

The exploration trigger $w_t=C_w\|\mc_t\|_{\bA_t^{-1}}>\eta$ implies that 
\begin{align*}
    \mc_t^\top \bA_t^{-1}\mc_t>\eta^2/C_w^2.
\end{align*}
Hence, applying Lemma~\ref{lem:core-set-size} above with $\bar\eta=\eta^2/C_w^2$ gives that 
$\abs{\cT_T^{\mathsf{exp}}} = \widetilde \cO(d\eta^{-2})$ since $C_w^2$ is logarithmic in $eT$ up to fixed model constants. This completes the proof of Proposition~\ref{prop:adaptive-init-guarantee}.

\subsection{Proof of Corollary~\ref{cor:upper}}

The regret on the exploration rounds is at most $p_{\max}\abs{\cT_T^{\mathsf{exp}}}$, which is $\widetilde \cO(d\eta^{-2})$ by Proposition~\ref{prop:adaptive-init-guarantee}. On the high-probability pilot event from Proposition~\ref{prop:adaptive-init-guarantee}, let us re-index the $\ORBIT$ rounds as $t_1<\cdots<t_N$ with $N=\abs{\cT^\ORBIT}$, and define $\bar u_s:=\mc_{t_s}^\top\theta_\ast$, $\bar{\tilde u}_s:=\proj_{\cU}(\mc_{t_s}^\top\hat\theta_{t_s})$, and $\bar\xi_s:=\xi_{t_s}$. The exploration rule is context-measurable, and the scalar pilots on $\ORBIT$ rounds depend only on contexts, exploration prices, and exploration-round noises. Hence, $\{(\bar u_s,\bar{\tilde u}_s)\}_{s=1}^N$ is independent of $\{\bar\xi_s\}_{s=1}^N$ and 
$\ORBIT$'s internal refinement randomization, while the same event gives that $\abs{\bar{\tilde u}_s-\bar u_s}\leq\eta$ for all $s$. Consequently, the associated index--pilot sequence for the re-indexed $\ORBIT$ rounds satisfies Assumption~\ref{assump:orbit-interface}. Applying Theorem~\ref{thm:upper} with $H=T$ to this length-$N$ $\ORBIT$ run with bin width $h=T^{-1/(4\beta-3)}$, using the same upper pilot-input budget $T$ in the logarithmic coarse-sampling schedule $\log(eT)$, and using $N\leq T$, we can obtain 
\begin{align*}
    \widetilde \cO\left(T^{\frac{2\beta-1}{4\beta-3}}+T\eta^2\right)
\end{align*}
regret on $\ORBIT$ rounds. The complement of the pilot event has probability $\cO(T^{-3})$ and contributes at most $p_{\max}T\cdot \cO(T^{-3})=\mathfrak o(1)$. Thus, combining the three contributions yields that 
\begin{align*}
    \Regret(T)= \widetilde \cO\left(T^{\frac{2\beta-1}{4\beta-3}}+T\eta^2+d\eta^{-2}\right).
\end{align*}
Therefore, optimizing the last two terms gives that $\eta^2\asymp\sqrt{d/T}$ when this value satisfies the structural smallness conditions, and the displayed simplified bound follows. This concludes the proof of Corollary~\ref{cor:upper}.

\medskip
\section{Proofs for the lower bound in Section~\ref{sec:lower}}\label{app:sec5-proofs}

In this appendix, we aim to prove Theorem~\ref{thm:lower} and its supporting lemmas. We begin with fixing the hard family and the accompanying notation; 
the subsequent subsections state the four intermediate claims, prove Theorem~\ref{thm:lower}, and then establish the claims in turn. We index the hard family by
\begin{align*}
\omega\in\{-1,+1\}^M,
\end{align*}
where the bit $\omega_j$ controls the sign of a localized perturbation around the translated coordinate associated with cell $j$.

Under the normalization of Section~\ref{sec:lower}, the actual linear utility lies in a subinterval of $[0,1]$ and $p_{\max}=1$, so the structural H\"older domain in Assumption~\ref{assump:holder} may be taken as $[-2,2]$. The lower-bound geometry uses local coordinates of size at most $C_{\rm loc}=1/32$ and the auxiliary constant
\begin{align*}
B:=1-C_{\rm loc}=\frac{31}{32}.
\end{align*}
For notational convenience in the construction below, all auxiliary tail functions are defined on the real line by constant extension outside their support; in particular, their restrictions to any compact interval, including the shifted structural domain required by the centered instance, have the same smoothness constants up to fixed factors.
We next fix deterministic strips
\begin{align*}
I_{\mathrm{glob}}:=\left[\frac{1}{8},\frac{7}{8}\right], \quad
I_{\mathrm{lin}}:=\left[\frac{11}{32},\frac{5}{8}\right], \quad
I_{\mathrm{curv}}:=\left[\frac{3}{8},\frac{19}{32}\right], \quad
I_{\mathrm{sign}}:=\left[\frac{7}{16},\frac{17}{32}\right].
\end{align*}
Let $g_{\mathrm{uni}}:\mathbb R\to[0,1]$ be the truncated-linear tail
\begin{align*}
g_{\mathrm{uni}}(z):= \begin{cases} 1 & \text{ if } z\leq 0, \\
1-z/B & \text{ if } 0\leq z\leq B, \\
0 & \text{ if } z\geq B. \end{cases}
\end{align*}

\begin{lemma}[Smooth baseline tail]\label{lem:smooth-baseline-tail}
For each sufficiently small $\varepsilon_0>0$, there exists a function $g_0\in C^\infty(\mathbb R)$ such that
\begin{enumerate}[label=(\alph*)]
    \item[a)] $g_0$ is nonincreasing and takes values in $[0,1]$;
    \item[b)] $g_0(z)=1$ for all $z\leq 0$, and $g_0(z)=0$ for all $z\geq B$;
    \item[c)] $g_0(z)=g_{\mathrm{uni}}(z)=1-z/B$ for all $z\in I_{\mathrm{glob}}$;
    \item[d)] $\pnorm{g_0-g_{\mathrm{uni}}}{\infty}\leq \varepsilon_0$.
\end{enumerate}
\end{lemma}

\textit{Proof}. 
Let us choose $0<\delta<\min\{1/8,\ B-7/8,\ B\varepsilon_0/4\}$. We first construct a smooth density $f_0$ on $[0,B]$. Let $h_L\in C^\infty([0,\delta])$ be nonnegative, flat at both endpoints, equal to $0$ in a neighborhood of $0$, equal to $1/B$ in a neighborhood of $\delta$, and satisfying that 
\begin{align*}
    \int_0^\delta h_L(z)\d z=\frac{\delta}{B}.
\end{align*}
Such function can be obtained by first taking a smooth transition from $0$ to $1/B$ that is bounded by $1/B$, flat at the endpoints, and strictly below $1/B$ on a set of positive measure; its integral is then strictly smaller than $\delta/B$. Adding a nonnegative interior $C^\infty_0(0,\delta)$ bump with the unique positive coefficient that matches the displayed integral gives the desired $h_L$. The same construction after a change of variable $z\mapsto B-z$ yields $h_R\in C^\infty([B-\delta,B])$ nonnegative, flat at both endpoints, equal to $1/B$ in a neighborhood of $B-\delta$, equal to $0$ in a neighborhood of $B$, and satisfying that 
\begin{align*}
    \int_{B-\delta}^{B}h_R(z)\d z=\frac{\delta}{B}.
\end{align*}
Let us define
\begin{align*}
    f_0(z)=\begin{cases}
        h_L(z), & 0\leq z\leq \delta,\\
        1/B, & \delta\leq z\leq B-\delta,\\
        h_R(z), & B-\delta\leq z\leq B,\\
        0, & z\notin[0,B].
    \end{cases}
\end{align*}
The flatness and endpoint matching make $f_0\in C^\infty(\mathbb R)$, and $f_0\geq0$. Its total mass is
\begin{align*}
    \frac{\delta}{B}+\frac{B-2\delta}{B}+\frac{\delta}{B}=1.
\end{align*}

We now set
\begin{align*}
    g_0(z):=\int_z^\infty f_0(t)\d t .
\end{align*}
Then $g_0$ is a smooth nonincreasing tail, takes values in $[0,1]$, equals $1$ on $(-\infty,0]$, and equals $0$ on $[B,\infty)$. For $z\in[\delta,B-\delta]$, it holds that 
\begin{align*}
    g_0(z)=\int_z^{B-\delta}\frac{\d t}{B}+\int_{B-\delta}^B h_R(t)\d t =\frac{B-\delta-z}{B}+\frac{\delta}{B}=1-\frac{z}{B}.
\end{align*}
Since $I_{\rm glob}\subset[\delta,B-\delta]$, this establishes property (c). Finally, $g_0$ and $g_{\rm uni}$ differ only on $[0,\delta]\cup[B-\delta,B]$. On either boundary interval, we have that 
\begin{align*}
    \abs{g_0(z)-g_{\rm uni}(z)} \leq \int \abs{f_0(t)-B^{-1}\bm 1\{0\leq t\leq B\}}\d t \leq \frac{4\delta}{B} \leq \varepsilon_0,
\end{align*}
after the choice of $\delta$. This completes the proof of Lemma~\ref{lem:smooth-baseline-tail}.

We choose a sufficiently small constant $\varepsilon_0>0$ and fix a function $g_0$ satisfying Lemma~\ref{lem:smooth-baseline-tail} above. Let us fix a sufficiently small constant $\gamma>0$, and set
\begin{align*}
w:=\gamma T^{-1/(4\beta-3)}, \quad M:=\floor{\frac{1}{64w}}.
\end{align*}
We define the local grid coordinates by
\begin{align*}
c_j:=2jw, \quad j=1,\dots,M.
\end{align*}
The actual scalar contexts are a common translate of these local coordinates, as specified below after centering the auxiliary noise law.
Finally, let
\begin{align*}
\varphi\in C_0^\infty([-1/8,1/8])
\end{align*}
satisfy that 
\begin{align*}
\varphi(0)=0, \quad \varphi'(0)=1,
\end{align*}
and assume that $\varphi$ is odd.

With the above notation in place, for $u\in[-1/32,1/32]$, let us define the baseline oracle price
\begin{align*}
p_0^\ast(u):=\frac{B+u}{2}.
\end{align*}
At the grid contexts $c_j$, we write
\begin{align*}
p_j^0:=p_0^\ast(c_j)=\frac{B+c_j}{2}, \quad z_j:=p_j^0-c_j=\frac{B}{2}-jw, \quad J_j:=\left[z_j-\frac{w}{8},z_j+\frac{w}{8}\right].
\end{align*}
For each sign vector
\begin{align*}
\omega=(\omega_1,\dots,\omega_M)\in\{-1,+1\}^M,
\end{align*}
let us define the auxiliary perturbed tail
\begin{align*}
g_\omega(z):=g_0(z)+\kappa w^\beta \sum_{j\in[M]} \omega_j\varphi\Big(\frac{z-z_j}{w}\Big), \quad z\in\mathbb R,
\end{align*}
where $\kappa>0$ is a sufficiently small constant. Since $\varphi$ is odd and every bump support lies inside $[0,B]$, each perturbation integrates to zero
\begin{align*}
    \int_0^B \varphi\left(\frac{z-z_j}{w}\right)\d z=w\int\varphi=0.
\end{align*}
Thus, once the membership proof below verifies that $g_\omega$ is a valid tail function, all auxiliary laws with tails $g_\omega$ have the same mean
\begin{align*}
    \mu_0=\int_0^B g_\omega(z)\d z,
\end{align*}
which is independent of $\omega$. 

The actual centered tail used in the hard instance is given by 
\begin{align*}
    \bar g_\omega(z):=g_\omega(z+\mu_0),
\end{align*}
and the actual scalar contexts are $\mc_j:=\mu_0+c_j$ with $\theta_\ast=1$ and $\Prob(\mc_t=\mc_j)=1/M$. The actual noise is $X-\mu_0$, where $X$ has tail $g_\omega$, so it has mean zero. Moreover, the actual revenue at context $\mc_j$ is exactly $p g_\omega(p-c_j)$ in the local coordinate. In the rest of the proof, we will suppress such deterministic centering. Equivalently, we write $c_t:=\mc_t-\mu_0$ for the local coordinate of the observed context, and define
\begin{align*}
q_\omega(c,p):=g_\omega(p-c), \quad r_\omega(c,p):=pq_\omega(c,p), \quad p_\omega^\ast(u):=\argmax_{p\in[0,1]} r_\omega(u,p),
\end{align*}
where $c,u$ denote the local coordinates. Such reparametrization is common to all environments, so it does not reduce the information available to the learner. Let $P_\omega$ be the law of the full transcript under the corresponding centered instance. For $j\in[M]$, denote by $\omega^{(j)}$ the sign vector obtained from $\omega$ by flipping only its $j$th coordinate.

\subsection{Main lower-bound reduction}

We first state the four intermediate claims used to prove Theorem~\ref{thm:lower}. The lower-bound argument follows the same local decomposition suggested by the upper bound. Each cell carries one bit of information, and flipping that bit moves the local oracle price by order $w^{\beta-1}$ while leaving the rest of the environment unchanged. The four lemmas below quantify the geometry of this family, the way local regret aggregates across cells, and the amount of statistical information that one cell can reveal.

\begin{lemma}\label{lem:main-lower-family}
For all sufficiently small $\gamma,\kappa>0$ and all sufficiently large $T$, each corresponding centered hard instance is lower-bound normalized in the sense of Definition~\ref{def:lb-normalized-instance}. Moreover, if $p_j^0$ denotes the baseline oracle price at local coordinate $c_j$---equivalently, at actual scalar index $\mu_0+c_j$ in the centered instance---there exist constants $a_1,a_2>0$ such that for each $\omega\in\{-1,+1\}^M$ and each $j\in[M]$,
\begin{align*}
a_1 w^{\beta-1}\leq \omega_j\big(p_\omega^\ast(c_j)-p_j^0\big)\leq a_2 w^{\beta-1}.
\end{align*}
\end{lemma}

Lemma~\ref{lem:main-lower-family} above identifies the relevant local scale. Flipping one bit changes the local oracle price by order $w^{\beta-1}$. Since the revenue curve is curved around its maximizer, an incorrect local decision will incur a cost of order $w^{2\beta-2}$ each time the corresponding cell is visited.

For environment $\omega$, let us define the one-step regret as 
\begin{align*}
\Delta_t^\omega:=r_\omega(c_t,p_\omega^\ast(c_t))-r_\omega(c_t,p_t),
\end{align*}
and the local regret attached to cell $j$ as 
\begin{align*}
R_j^\omega:= \E_\omega \bigg[ \sum_{t \in [T]} \Delta_t^\omega \bm 1\{c_t=c_j \text{ or } p_t-c_t\in J_j\} \bigg],
\end{align*}
where $J_j :=[z_j-w/8, z_j+w/8]$ is the bump interval around $z_j$.

\begin{lemma}\label{lem:main-lower-agg}
For each environment $\omega$, it holds that 
\begin{align*}
\sum_{j \in [M]} R_j^\omega \leq 2\Regret^\omega(T).
\end{align*}
\end{lemma}

The aggregation characterized in Lemma \ref{lem:main-lower-agg} above is purely geometric. A single round can contribute to at most one context cell and at most one bump interval, so the sum of the local regret counters over $j$ can exceed the total regret by at most a factor of two.

\begin{lemma}\label{lem:main-lower-kl}
There exists a constant $C_{\mathrm{kl}}>0$ such that for each sign vector $\omega$ and each $j\in[M]$,
\begin{align*}
\KL(P_\omega,P_{\omega^{(j)}}) \leq C_{\mathrm{kl}} w^{2\beta-2}R_j^\omega + C_{\mathrm{kl}} T w^{4\beta-3},
\end{align*}
and symmetrically,
\begin{align*}
\KL(P_{\omega^{(j)}},P_\omega) \leq C_{\mathrm{kl}} w^{2\beta-2}R_j^{\omega^{(j)}} + C_{\mathrm{kl}} T w^{4\beta-3}.
\end{align*}
\end{lemma}

Lemma~\ref{lem:main-lower-kl} above captures the central tradeoff. If a policy pays little regret on cell $j$, the paired environments $\omega$ and $\omega^{(j)}$ remain statistically close. If it pays a lot of regret, the desired lower bound is already present. Either way, the policy cannot avoid a local cost.

\begin{lemma}\label{lem:main-lower-local}
There exists a constant $c_\star>0$ such that for each sign vector $\omega$ and each $j\in[M]$,
\begin{align*}
R_j^\omega + R_j^{\omega^{(j)}} \geq c_\star T w^{2\beta-1}.
\end{align*}
\end{lemma}

Lemma~\ref{lem:main-lower-local} above is the local two-point lower bound. It follows by combining Lemma~\ref{lem:main-lower-kl} with a decoder argument and the Bretagnolle--Huber inequality. We are now ready to prove Theorem~\ref{thm:lower} below.

\textit{Proof of Theorem~\ref{thm:lower}}. 
Let us fix an arbitrary policy $\pi$. For each $j\in[M]$, an application of Lemma~\ref{lem:main-lower-local} gives that 
\begin{align*}
R_j^\omega + R_j^{\omega^{(j)}} \geq c_\star T w^{2\beta-1}, \quad \forall \omega\in\{-1,+1\}^M.
\end{align*}
Summing over $\omega$ and using the fact that $\omega\mapsto \omega^{(j)}$ is a bijection of the hypercube, we can deduce that 
\begin{align*}
\frac{1}{2^M}\sum_{\omega} R_j^\omega \geq \frac{c_\star}{2} T w^{2\beta-1}.
\end{align*}
Further, summing this inequality over $j\in[M]$ leads to 
\begin{align*}
\frac{1}{2^M}\sum_{\omega}\sum_{j \in [M]} R_j^\omega \geq \frac{c_\star}{2} M T w^{2\beta-1}.
\end{align*}
Then with the aid of Lemma~\ref{lem:main-lower-agg}, we can show that 
\begin{align*}
\frac{1}{2^M}\sum_{\omega}\Regret^\omega(T) \geq \frac{c_\star}{4} M T w^{2\beta-1}.
\end{align*}

Note that since $M\asymp 1/w$, the right-hand side of the expression above is of order $T w^{2\beta-2}$. Consequently, substituting
\begin{align*}
w=\gamma T^{-1/(4\beta-3)}
\end{align*}
yields that 
\begin{align*}
T w^{2\beta-2}=\gamma^{2\beta-2}T^{\frac{2\beta-1}{4\beta-3}}.
\end{align*}
Therefore, for every policy $\pi$, at least one centered hard instance indexed by some $\omega\in\{-1,+1\}^M$ has regret at least $c_\beta T^{\frac{2\beta-1}{4\beta-3}}$. By Lemma~\ref{lem:main-lower-family}, that centered hard instance is lower-bound normalized in the sense of Definition~\ref{def:lb-normalized-instance}. Here, the actual centered tail is $\bar g_\omega$; the proof employs the auxiliary function $g_\omega$ only as a local-coordinate representative of the same revenues. Since $\pi$ is arbitrary, the desired conclusion is established. This concludes the proof of Theorem~\ref{thm:lower}.

\subsection{Proof of Lemma~\ref{lem:main-lower-family}}

We split the proof into four technical lemmas. The first lemma controls the size of the perturbations, the second one analyzes the geometry of the perturbed revenue curves, the third one identifies the sign-sensitive shift of the oracle price at each grid context, and the fourth one verifies the lower-bound normalization conditions.

\begin{lemma}\label{lem:lower-family-helper-derivs}
Let $s:=\floor{\beta}$ and $\alpha:=\beta-s\in[0,1)$.
Then there exist constants $C_k<\infty$ for $k=0,\dots,s$, and if $\alpha>0$, a constant
$C_{s,\alpha}<\infty$, depending only on $\varphi$, such that for each
$\omega\in\{-1,+1\}^M$, we have that 
\begin{align}
\pnorm{g_\omega-g_0}{\infty} &\leq C_0\kappa w^\beta, \label{eq:lower-deriv-0}\\
\pnorm{g_\omega'-g_0'}{\infty} &\leq C_1\kappa w^{\beta-1}, \label{eq:lower-deriv-1}\\
\pnorm{g_\omega''-g_0''}{\infty} &\leq C_2\kappa w^{\beta-2}, \label{eq:lower-deriv-2}\\
\pnorm{g_\omega^{(k)}-g_0^{(k)}}{\infty} &\leq C_k\kappa w^{\beta-k},
\quad k=0,\dots,s. \label{eq:lower-deriv-k}
\end{align}
If $\alpha>0$, with
\begin{align*}
[f]_{C^\alpha(K)} := \sup_{\substack{x,y\in K\\x\neq y}} \frac{\abs{f(x)-f(y)}}{\abs{x-y}^\alpha}
\end{align*}
we also have that 
\begin{align}
[g_\omega^{(s)}-g_0^{(s)}]_{C^\alpha(K)} \leq C_{s,\alpha}\kappa w^{\beta-s-\alpha} = C_{s,\alpha}\kappa
\quad\text{on every compact interval }K\subset\mathbb R.
\label{eq:lower-deriv-holder}
\end{align}
\end{lemma}

\textit{Proof.}
Let us define 
\begin{align*}
h_\omega(z):=g_\omega(z)-g_0(z) = \kappa w^\beta\sum_{j\in[M]}\omega_j \varphi\Big(\frac{z-z_j}{w}\Big).
\end{align*}
Since the supports of the translated bumps
\begin{align*}
\varphi\Big(\frac{\cdot-z_j}{w}\Big), \quad j\in[M],
\end{align*}
are pairwise disjoint, at each point of $\R$ at most one term in the sum is nonzero. Then it holds that for each integer $k\geq 0$,
\begin{align*}
h_\omega^{(k)}(z) = \kappa w^{\beta-k}\sum_{j\in[M]}\omega_j \varphi^{(k)}\Big(\frac{z-z_j}{w}\Big),
\end{align*}
and thus 
\begin{align*}
\pnorm{h_\omega^{(k)}}{\infty} \leq \kappa w^{\beta-k}\pnorm{\varphi^{(k)}}{\infty}.
\end{align*}
This establishes \eqref{eq:lower-deriv-0}--\eqref{eq:lower-deriv-k}.

We next assume that $\alpha>0$ and fix a compact interval $K\subset\mathbb R$. Fix $x\neq y$ in $K$.
Assume first that there do not exist two distinct indices $i\neq j$ such that
\begin{align*}
x\in[z_i-w/8,z_i+w/8], \quad y\in[z_j-w/8,z_j+w/8].
\end{align*}
Then either both points belong to the support of the same translated bump, or at most one of them lies in a bump support. For either case, $h_\omega^{(s)}(x)-h_\omega^{(s)}(y)$ is contributed by a single translated copy of $\varphi^{(s)}$, so we have that 
\begin{align*}
\frac{\abs{h_\omega^{(s)}(x)-h_\omega^{(s)}(y)}}{\abs{x-y}^\alpha} \leq \kappa w^{\beta-s-\alpha} [\varphi^{(s)}]_{C^\alpha(\R)}.
\end{align*}

If instead $x$ and $y$ lie in the supports of two different bumps, the support centers will be $w$ apart and each support has radius $w/8$, so it holds that 
\begin{align*}
\abs{x-y}\geq \frac{3w}{4}.
\end{align*}
Using the sup-norm bound already proved, we can obtain that 
\begin{align*}
\frac{\abs{h_\omega^{(s)}(x)-h_\omega^{(s)}(y)}}{\abs{x-y}^\alpha} &\leq \frac{2\kappa w^{\beta-s}\pnorm{\varphi^{(s)}}{\infty}}{(3w/4)^\alpha} \lesssim \kappa w^{\beta-s-\alpha}.
\end{align*}
Therefore, combining the two cases above leads to \eqref{eq:lower-deriv-holder}. This completes the proof of Lemma~\ref{lem:lower-family-helper-derivs}.

\begin{lemma}\label{lem:lower-family-helper-geometry}
For all sufficiently small $\gamma,\kappa>0$ and all sufficiently large $T$, the following statements hold uniformly over $\omega\in\{-1,+1\}^M$.

1) For each $u\in[-1/32,1/32]$, the revenue curve $p\mapsto r_\omega(u,p)$ has a unique global maximizer $p_\omega^\ast(u)$ in $I_{\mathrm{sign}}\subset I_{\mathrm{curv}}$, and there exist constants $\sigma_r,L_r>0$, depending only on the normalized problem parameters, such that
\begin{align*}
\frac{\sigma_r}{2}\abs{p-p_\omega^\ast(u)}^2 \leq r_\omega\big(u,p_\omega^\ast(u)\big)-r_\omega(u,p) \leq \frac{L_r}{2}\abs{p-p_\omega^\ast(u)}^2
\end{align*}
for each $p\in[0,1]$.

2) If
\begin{align*}
p_0^\ast(u):=\frac{B+u}{2},
\end{align*}
there exists a numerical constant $C_{\mathrm{sh}}<\infty$, independent of $\kappa,w,T$, such that
\begin{align*}
\sup_{u\in[-1/32,1/32]}\abs{p_\omega^\ast(u)-p_0^\ast(u)}\leq C_{\mathrm{sh}}\kappa w^{\beta-1}.
\end{align*}
\end{lemma}

\textit{Proof.}
Let us define
\begin{align*}
r_{\mathrm{uni}}(u,p):=p g_{\mathrm{uni}}(p-u), \quad u\in[-1/32,1/32], \quad p\in[0,1].
\end{align*}
A direct inspection of the three regions $p-u<0$, $0\leq p-u\leq B$, and $p-u>B$ shows that the unique global maximizer of $p\mapsto r_{\mathrm{uni}}(u,p)$ is exactly
\begin{align*}
p_0^\ast(u)=\frac{B+u}{2}.
\end{align*}
Since $p_0^\ast(u)\in[15/32,1/2]\subset\mathrm{int}(I_{\mathrm{curv}})$, it follows from the compactness that 
\begin{align}
\Delta_{\mathrm{out}}^{\mathrm{uni}} := \inf_{u\in[-1/32,1/32], p\in[0,1]\setminus I_{\mathrm{curv}}} \Big(r_{\mathrm{uni}}\big(u,p_0^\ast(u)\big)-r_{\mathrm{uni}}(u,p)\Big) > 0.
\label{eq:lower-outside-gap-uni}
\end{align}
We choose the smoothing tolerance in the construction of $g_0$ so that
\begin{align}
\varepsilon_0\leq \frac{1}{8}\Delta_{\mathrm{out}}^{\mathrm{uni}}.
\label{eq:lower-eps0-choice}
\end{align}

We now fix $u\in[-1/32,1/32]$ and $p\in I_{\mathrm{curv}}$. Then we have that $p-u\in I_{\mathrm{lin}}$, so from the definition of $g_0$, it holds that 
\begin{align*}
g_0'(p-u)=-\frac{1}{B}, \quad g_0''(p-u)=0, \quad r_{0,pp}(u,p)=-\frac{2}{B}.
\end{align*}
By resorting to Lemma~\ref{lem:lower-family-helper-derivs}, we can deduce that 
\begin{align*}
    \abs{r_{\omega,pp}(u,p)-r_{0,pp}(u,p)} \leq 2\pnorm{g_\omega'-g_0'}{\infty}+\pnorm{g_\omega''-g_0''}{\infty} \leq 2C_1\kappa w^{\beta-1}+C_2\kappa w^{\beta-2} \leq C\kappa.
\end{align*}
Then choosing $\kappa$ sufficiently small gives that 
\begin{align}
-\frac{4}{B} \leq r_{\omega,pp}(u,p) \leq -\frac{1}{B}, \quad \forall u\in[-1/32,1/32],  \ \forall p\in I_{\mathrm{curv}}.
\label{eq:lower-curv-window}
\end{align}

For the endpoint derivatives, note that if $p=7/16$ or $p=17/32$ and $u\in[-1/32,1/32]$, we have again that $p-u\in I_{\mathrm{lin}}$, so it holds that 
\begin{align*}
r_{0,p}(u,p)=g_0(p-u)+pg_0'(p-u)=1+\frac{u}{B}-\frac{2p}{B}.
\end{align*}
Consequently, we can obtain that 
\begin{align*}
\inf_{u\in[-1/32,1/32]}r_{0,p}\Big(u,\frac{7}{16}\Big)=\frac{2}{31}, \quad \sup_{u\in[-1/32,1/32]}r_{0,p}\Big(u,\frac{17}{32}\Big)=-\frac{2}{31}.
\end{align*}
An application of Lemma~\ref{lem:lower-family-helper-derivs} leads to 
\begin{align*}
\abs{r_{\omega,p}(u,p)-r_{0,p}(u,p)} \leq \pnorm{g_\omega-g_0}{\infty}+\pnorm{g_\omega'-g_0'}{\infty} \leq C'\kappa.
\end{align*}
After shrinking $\kappa$ if necessary, there exists some $c_{\partial}>0$ such that
\begin{align*}
r_{\omega,p}\Big(u,\frac{7}{16}\Big)\geq c_{\partial}, \quad r_{\omega,p}\Big(u,\frac{17}{32}\Big)\leq -c_{\partial}
\end{align*}
uniformly over $u\in[-1/32,1/32]$. Together with \eqref{eq:lower-curv-window}, this implies that $p\mapsto r_{\omega,p}(u,p)$ is strictly decreasing on $I_{\mathrm{curv}}$ and changes sign across $I_{\mathrm{sign}}\subset I_{\mathrm{curv}}$. Hence, there is a unique local maximizer
\begin{align*}
p_\omega^\ast(u)\in\Big(\frac{7}{16},\frac{17}{32}\Big)\subset I_{\mathrm{sign}}.
\end{align*}

We next show that this local maximizer is in fact global. It follows from $p_0^\ast(u)-u\in I_{\mathrm{lin}}$ that 
\begin{align*}
r_0\big(u,p_0^\ast(u)\big)=r_{\mathrm{uni}}\big(u,p_0^\ast(u)\big).
\end{align*}
Moreover, since $p\leq 1$ and $\pnorm{g_0-g_{\mathrm{uni}}}{\infty}\leq \varepsilon_0$, it holds that 
\begin{align*}
\sup_{u\in[-1/32,1/32], p\in[0,1]}\abs{r_0(u,p)-r_{\mathrm{uni}}(u,p)}\leq \varepsilon_0.
\end{align*}
Consequently, for each $p\in[0,1]\setminus I_{\mathrm{curv}}$, we have that 
\begin{align*}
    r_0\big(u,p_0^\ast(u)\big)-r_0(u,p) \geq r_{\mathrm{uni}}\big(u,p_0^\ast(u)\big)-r_{\mathrm{uni}}(u,p)-\varepsilon_0 \geq \Delta_{\mathrm{out}}^{\mathrm{uni}}-\varepsilon_0 \geq \frac{7}{8}\Delta_{\mathrm{out}}^{\mathrm{uni}}.
\end{align*}
Inside $I_{\mathrm{curv}}$, we have that $p-u\in I_{\mathrm{lin}}$, and thus $r_0(u,p)=r_{\mathrm{uni}}(u,p)$; on this interval, the truncated-linear revenue is strictly concave with unique maximizer $p_0^\ast(u)$. Together with the outside gap above, this shows that $p_0^\ast(u)$ is the unique global maximizer of the smooth baseline revenue $r_0(u,\cdot)$ on $[0,1]$.

Combining Lemma~\ref{lem:lower-family-helper-derivs} and the bound $p\leq 1$ yields that 
\begin{align*}
\sup_{u\in[-1/32,1/32], p\in[0,1]}\abs{r_\omega(u,p)-r_0(u,p)}\leq C_0\kappa w^\beta.
\end{align*}
Let us choose $\kappa$ sufficiently small and then $T$ sufficiently large so that
\begin{align*}
C_0\kappa w^\beta\leq \frac{1}{8}\Delta_{\mathrm{out}}^{\mathrm{uni}}.
\end{align*}
Since $p_\omega^\ast(u)$ maximizes $r_\omega(u,\cdot)$ over $I_{\mathrm{curv}}$ and $p_0^\ast(u)\in I_{\mathrm{sign}}\subset I_{\mathrm{curv}}$, it holds that 
\begin{align*}
r_\omega\big(u,p_\omega^\ast(u)\big)\geq r_\omega\big(u,p_0^\ast(u)\big).
\end{align*}
Then for each $p\in[0,1]\setminus I_{\mathrm{curv}}$, we have that 
\begin{align*}
    r_\omega\big(u,p_\omega^\ast(u)\big)-r_\omega(u,p) \geq r_\omega\big(u,p_0^\ast(u)\big)-r_\omega(u,p) \geq r_0\big(u,p_0^\ast(u)\big)-r_0(u,p)-2C_0\kappa w^\beta \geq \frac{5}{8}\Delta_{\mathrm{out}}^{\mathrm{uni}}.
\end{align*}
Hence, $p_\omega^\ast(u)$ is the unique global maximizer of $r_\omega(u,\cdot)$ on $[0,1]$. Denote by 
\begin{align*}
\Delta_{\mathrm{out}}:=\frac{1}{2}\Delta_{\mathrm{out}}^{\mathrm{uni}}.
\end{align*}
Then for all sufficiently large $T$, it holds that 
\begin{align}
\label{eq:lower-outside-gap-omega}
r_\omega\big(u,p_\omega^\ast(u)\big)-r_\omega(u,p) \geq \Delta_{\mathrm{out}} \quad \forall p\in[0,1]\setminus I_{\mathrm{curv}}.
\end{align}

For $p\in I_{\mathrm{curv}}$, the double-integral argument along with \eqref{eq:lower-curv-window} gives that 
\begin{align}
\frac{1}{2B}\abs{p-p_\omega^\ast(u)}^2 \leq r_\omega\big(u,p_\omega^\ast(u)\big)-r_\omega(u,p) \leq \frac{2}{B}\abs{p-p_\omega^\ast(u)}^2.
\label{eq:lower-local-quad}
\end{align}
If $p\notin I_{\mathrm{curv}}$, we have that  $p_\omega^\ast(u)\in I_{\mathrm{sign}}$ and
\begin{align*}
\dist{I_{\mathrm{sign}}}{[0,1]\setminus I_{\mathrm{curv}}}=\frac{1}{16},
\end{align*}
so it holds that $\abs{p-p_\omega^\ast(u)}\geq 1/16$. Using \eqref{eq:lower-outside-gap-omega} along with the trivial bounds
\begin{align*}
r_\omega\big(u,p_\omega^\ast(u)\big)-r_\omega(u,p)\leq 1, \quad \abs{p-p_\omega^\ast(u)}\leq 1,
\end{align*}
we can deduce that 
\begin{align*}
\Delta_{\mathrm{out}}\abs{p-p_\omega^\ast(u)}^2 \leq r_\omega\big(u,p_\omega^\ast(u)\big)-r_\omega(u,p) \leq 256\abs{p-p_\omega^\ast(u)}^2.
\end{align*}
Combining this with \eqref{eq:lower-local-quad} establishes the global quadratic growth bounds with
\begin{align*}
\sigma_r:=\min\left\{\frac{1}{B},2\Delta_{\mathrm{out}}\right\}, \quad L_r:=\max\left\{\frac{4}{B},512\right\}.
\end{align*}

Finally, let us define 
\begin{align*}
F_\omega(u,p):=r_{\omega,p}(u,p)=g_\omega(p-u)+pg_\omega'(p-u).
\end{align*}
Since $F_0\big(u,p_0^\ast(u)\big)=0$ and $p_0^\ast(u)-u\in I_{\mathrm{lin}}$, Lemma~\ref{lem:lower-family-helper-derivs} gives that 
\begin{align*}
\sup_{u\in[-1/32,1/32]}\abs{F_\omega\big(u,p_0^\ast(u)\big)}\leq C\kappa w^{\beta-1}.
\end{align*}
Applying the mean-value theorem between $p_0^\ast(u)$ and $p_\omega^\ast(u)$ and using \eqref{eq:lower-curv-window} then yield that 
\begin{align*}
\sup_{u\in[-1/32,1/32]}\abs{p_\omega^\ast(u)-p_0^\ast(u)}\leq C_{\mathrm{sh}}\kappa w^{\beta-1}
\end{align*}
for a numerical constant $C_{\mathrm{sh}}<\infty$.
This concludes the proof of Lemma~\ref{lem:lower-family-helper-geometry}.

\begin{lemma}\label{lem:lower-family-helper-shift}
For all sufficiently small $\gamma,\kappa>0$ and all sufficiently large $T$, there exist constants $a_1,a_2>0$ such that for each $\omega\in\{-1,+1\}^M$ and each $j\in[M]$,
\begin{align*}
a_1w^{\beta-1} \leq \omega_j\big(p_\omega^\ast(c_j)-p_j^0\big) \leq a_2w^{\beta-1}.
\end{align*}
\end{lemma}

\textit{Proof.}
Recall that
\begin{align*}
p_j^0=\frac{B+c_j}{2}, \quad z_j=p_j^0-c_j=\frac{B}{2}-jw.
\end{align*}
Note that at point $z_j$, only the $j$th bump is active, and $z_j\in I_{\mathrm{lin}}$. Since $\varphi(0)=0$ and $\varphi'(0)=1$, it holds that 
\begin{align*}
    F_\omega(c_j,p_j^0) = \big(g_\omega(z_j)-g_0(z_j)\big)+p_j^0\big(g_\omega'(z_j)-g_0'(z_j)\big) = \kappa w^\beta\omega_j\varphi(0)+p_j^0\kappa w^{\beta-1}\omega_j\varphi'(0) = \kappa p_j^0\omega_j w^{\beta-1}.
\end{align*}
It follows from $p_j^0\in[B/2,B/2+1/64]$ that the quantity above has sign $\omega_j$ and magnitude comparable to $w^{\beta-1}$. By invoking Lemma~\ref{lem:lower-family-helper-geometry}, we see that $p_\omega^\ast(c_j)\in I_{\mathrm{sign}}\subset I_{\mathrm{curv}}$ and $r_{\omega,pp}$ is bounded between $-4/B$ and $-1/B$ on $I_{\mathrm{curv}}$. Then applying the mean-value theorem to $p\mapsto F_\omega(c_j,p)$ between $p_j^0$ and $p_\omega^\ast(c_j)$, we can obtain that 
\begin{align*}
p_\omega^\ast(c_j)-p_j^0=-\frac{F_\omega(c_j,p_j^0)}{r_{\omega,pp}(c_j,\xi_j)}
\end{align*}
for some $\xi_j$ lying between these two points. Thus, the stated bounds follow immediately. This completes the proof of Lemma~\ref{lem:lower-family-helper-shift}.

\begin{lemma}\label{lem:lower-family-helper-membership}
For all sufficiently small $\gamma,\kappa>0$ and all sufficiently large $T$, each corresponding centered hard instance is lower-bound normalized in the sense of Definition~\ref{def:lb-normalized-instance}.
\end{lemma}

\textit{Proof.}
We work in the local coordinate $c_j=2jw$. Since $M=\floor{1/(64w)}$, it holds that 
\begin{align*}
0<c_j\leq 2Mw\leq C_{\rm loc}, \quad \forall j\in[M].
\end{align*}
The actual centered instance has scalar contexts $\mc_j=\mu_0+c_j$, parameter $\theta_\ast=1$, and centered noise tail $\bar g_\omega(z)=g_\omega(z+\mu_0)$. The translated scalar-index interval is
\begin{align*}
    \cU_{\rm lb}:=[\mu_0,
    \mu_0+C_{\rm loc}],
\end{align*}
so Assumption~\ref{assump:bounded} holds on this interval. Since the auxiliary noise law is supported on $[0,B]$, its mean satisfies that $0\leq \mu_0\leq B$, and thus $0\leq \mc_j\leq B+C_{\rm loc}=1$. Consequently, the actual context space is contained in $\mathbb{B}^1_2(1)$, and Assumption~\ref{assump:linear} holds with $d=1$, $C_\theta=1$, and $\theta_\ast=1$.
Each bump is supported inside $[z_j-w/8,z_j+w/8]$, and all bump supports lie strictly inside $I_{\mathrm{lin}}$ for all sufficiently large $T$. Since $g_0\in C^\infty(\mathbb R)$ is constant outside a compact interval and the translated bumps have disjoint supports, Lemma~\ref{lem:lower-family-helper-derivs} implies uniform H\"older bounds for the auxiliary tails $g_\omega$ on each compact interval relevant to the construction when $\beta$ is noninteger, and uniform bounds on $g_\omega^{(\beta)}$ when $\beta$ is integer. 

For the integer case, the Taylor-remainder formulation in Assumption~\ref{assump:holder} follows from the integral remainder formula and this uniform highest-derivative bound. The centered tails $\bar g_\omega$ are translates of $g_\omega$, so they inherit the same smoothness bounds on the compact price-index gap interval generated by $\cU_{\rm lb}$ and $[0,1]$. Hence, Assumption~\ref{assump:holder} holds uniformly over the hard family.
The functions $g_\omega$ are nonincreasing and take values in $[0,1]$. Outside the bump supports, this follows from the construction of $g_0$. On the bump supports, all points lie in $I_{\mathrm{lin}}$, where $g_0$ is bounded away from both $0$ and $1$ by a numerical margin, while Lemma~\ref{lem:lower-family-helper-derivs} gives that $\|g_\omega-g_0\|_\infty\leq C\kappa w^\beta$; choosing $\kappa$ small keeps $g_\omega$ in $[0,1]$ there. Also, on the bump supports $g_0'=-1/B$ and Lemma~\ref{lem:lower-family-helper-derivs} gives $\|g_\omega'-g_0'\|_\infty\leq C\kappa w^{\beta-1}$, which is smaller than $1/(2B)$ after shrinking $\kappa$, so $g_\omega$ remains nonincreasing. Thus, $F_\omega:=1-g_\omega$ is a distribution function of an auxiliary law $X$ supported on $[0,B]$.
The realized valuation in the actual centered instance can be written as
\begin{align*}
    v_t=(\mu_0+c_j)+(X_t-\mu_0)=c_j+X_t,
\end{align*}
and therefore $v_t\in[2w,C_{\rm loc}+B]=[2w,1]\subset[0,1]$ almost surely. This verifies the bounded-realized-valuation normalization used for uniform-price pilot identification. The centering shift does not change revenue geometry since $p\bar g_\omega(p-\mc_j)=p g_\omega(p-c_j)$.

It remains to verify the structural revenue geometry on $\cU_{\rm lb}$, not only at the grid points. For an actual scalar index $u=\mu_0+x\in\cU_{\rm lb}$ with $x\in[0,C_{\rm loc}]$, the centered-instance revenue is
\begin{align*}
    p\,\bar g_\omega(p-u)=p\,g_\omega(p-x)=r_\omega(x,p).
\end{align*}
Since $[0,C_{\rm loc}]\subset[-1/32,1/32]$, Lemma~\ref{lem:lower-family-helper-geometry} applies uniformly to each such $x$. It gives a unique interior maximizer and global quadratic-growth bounds over $p\in[0,1]$ with constants independent of $\omega,T$, so Assumption~\ref{assump:rho} holds on $\cU_{\rm lb}$. The endpoint-derivative margins in the proof of Lemma~\ref{lem:lower-family-helper-geometry} in fact place the maximizer in the open interval
\begin{align*}
    p_\omega^\ast(x)\in(7/16,17/32)\subset I_{\mathrm{sign}}.
\end{align*}
Consequently, if $\abs{p-p_\omega^\ast(x)}\leq1/16$, we have that $p\in I_{\mathrm{curv}}$. On this whole curvature window, \eqref{eq:lower-curv-window} gives that $-r_{\omega,pp}(x,p)\geq 1/B$, which is no smaller than the lower quadratic-growth constant used above. Consequently, the local-concavity radius in Lemma~\ref{lem:rho-exists} may be taken to be at least $1/16$ for every member of the hard family. Therefore, each centered hard instance satisfies Assumption~\ref{assump:rho} with a local-concavity radius at least $1/16$.
This concludes the proof of Lemma~\ref{lem:lower-family-helper-membership}.

We are now ready to prove Lemma~\ref{lem:main-lower-family} below.

\textit{Proof of Lemma~\ref{lem:main-lower-family}}. 
Note that Lemma~\ref{lem:lower-family-helper-membership} establishes that the perturbed family satisfies the lower-bound normalization conditions in Definition~\ref{def:lb-normalized-instance}, and Lemma~\ref{lem:lower-family-helper-shift} gives the stated sign-sensitive displacement of the oracle prices at the grid contexts $c_j$. This is exactly the desired conclusion, 
which completes the proof of Lemma~\ref{lem:main-lower-family}.

\subsection{Proof of Lemma~\ref{lem:main-lower-agg}}

By definition, it holds that 
\begin{align*}
R_j^\omega = \E_\omega \bigg[ \sum_{t \in [T]} \Delta_t^\omega \bm 1\{c_t=c_j \text{ or } p_t-c_t\in J_j\} \bigg].
\end{align*}
Summing over $j$ and using the linearity of expectation give that 
\begin{align*}
\sum_{j \in [M]} R_j^\omega = \E_\omega \bigg[ \sum_{t \in [T]} \Delta_t^\omega \sum_{j \in [M]} \bm 1\{c_t=c_j \text{ or } p_t-c_t\in J_j\} \bigg].
\end{align*}
Let us fix a round $t$. Exactly one of the events $\{c_t=c_j\}$ can occur since the contexts take values in the discrete set $\{c_1,\dots,c_M\}$. Further, the intervals $J_1,\dots,J_M$ are pairwise disjoint, because their centers are spaced by $w$ while their radius is only $w/8$. Hence, at most one of the events $\{p_t-c_t\in J_j\}$ can occur. Then it follows that 
\begin{align*}
\sum_{j \in [M]} \bm 1\{c_t=c_j \text{or} p_t-c_t\in J_j\}\leq 2.
\end{align*}
Therefore, substituting this pointwise inequality into the previous expression yields that 
\begin{align*}
\sum_{j \in [M]} R_j^\omega \leq 2\E_\omega \bigg[\sum_{t \in [T]} \Delta_t^\omega\bigg] = 2 \Regret^\omega(T).
\end{align*}
This concludes the proof of Lemma~\ref{lem:main-lower-agg}.

\subsection{Proof of Lemma~\ref{lem:main-lower-kl}}

Let us fix a sign vector $\omega$ and an index $j\in[M]$, and compare the environments $\omega$ and $\omega^{(j)}$. The context law is the same under both environments; only the Bernoulli purchase probabilities differ. We write
\begin{align*}
q_\nu(c,p):=g_\nu(p-c), \quad \nu\in\{\omega,\omega^{(j)}\}.
\end{align*}
Denote by 
\begin{align*}
\mathcal G_{t-1}:=\sigma(c_1,p_1,y_1,\dots,c_{t-1},p_{t-1},y_{t-1},c_t,p_t),
\end{align*}
where $c_t=\mc_t-\mu_0$ is the local coordinate. Since $\mu_0$ and this transformation are common to all paired environments, using $c_t$ instead of $\mc_t$ is an equivalent representation of the transcript. If the policy uses internal randomization, we augment the transcript by its random seeds, or equivalently condition on them. These seeds have the same law under $P_\omega$ and $P_{\omega^{(j)}}$, and contribute zero KL. After this harmless augmentation, the posted price $p_t$ is $\mathcal G_{t-1}$-measurable. An application of the chain rule for the Kullback--Leibler (KL) divergence in adaptive experiments leads to 
\begin{align}\label{eq:B6-new}
\KL(P_\omega,P_{\omega^{(j)}}) = \sum_{t \in [T]} \E_\omega \bigg[ \KL\big(\Ber(q_\omega(c_t,p_t)),\Ber(q_{\omega^{(j)}}(c_t,p_t))\big) \bigg].
\end{align}
Note that the two environments differ only through the $j$th bump. Hence, if $p_t-c_t\notin J_j$, we have that 
$q_\omega(c_t,p_t)=q_{\omega^{(j)}}(c_t,p_t),$
and the per-round KL contribution is zero. It therefore suffices to study rounds with
$p_t-c_t\in J_j.$

We first record a uniform quadratic bound for the Bernoulli KL on the relevant range. Since
\begin{align*}
z_j=\frac{B}{2}-jw, \quad 1\leq j\leq M\leq \frac{1}{64w},
\end{align*}
it holds that 
\begin{align*}
\frac{B}{2}-\frac{1}{64}\leq z_j\leq \frac{B}{2}-w.
\end{align*}
Since $J_j=\left[z_j-\frac{w}{8},z_j+\frac{w}{8}\right],$
all these intervals are contained in the fixed compact interval
\begin{align*}
I_{\mathrm{rel}}:=\left[\frac{B}{2}-\frac{1}{32},\frac{B}{2}+\frac{1}{32}\right]\subset I_{\mathrm{lin}}
\end{align*}
for all sufficiently large $T$. Thus, whenever $p_t-c_t\in J_j$, the baseline mean
$q_0(c_t,p_t)=g_0(p_t-c_t)$
belongs to a compact subinterval of $(0,1)$. Since the perturbation magnitude is $\cO(\kappa w^\beta)$ uniformly, by choosing $\kappa$ sufficiently small we can ensure that there exists a constant $\underline q_0>0$ such that
\begin{align*}
q_\omega(c,p), q_{\omega^{(j)}}(c,p)\in [\underline q_0,1-\underline q_0] \quad \text{whenever } p-c\in J_j.
\end{align*}
Consequently, there exists a constant $\bar C_{\mathrm{kl}}<\infty$ such that
\begin{align*}
\KL\big(\Ber(q),\Ber(q')\big)\leq \bar C_{\mathrm{kl}}(q-q')^2
\end{align*}
whenever $q,q'\in[\underline q_0,1-\underline q_0]$.

We now fix a round $t$ and examine two separate cases.

\noindent
\emph{Case A: $c_t=c_j$ and $p_t-c_t\in J_j$.} In this case, it holds that 
\begin{align*}
\abs{p_t-p_j^0}=\abs{(p_t-c_j)-z_j}\leq \frac{w}{8}.
\end{align*}
Since the two environments differ only in the sign of the $j$th bump, we have that 
\begin{align*}
q_\omega(c_j,p_t)-q_{\omega^{(j)}}(c_j,p_t) = 2\kappa w^\beta\varphi\Big(\frac{p_t-p_j^0}{w}\Big).
\end{align*}
Using $\varphi(0)=0$ and the mean-value theorem, we can deduce that 
\begin{align*}
\abs{\varphi\Big(\frac{p_t-p_j^0}{w}\Big)} \leq \pnorm{\varphi'}{\infty} \frac{\abs{p_t-p_j^0}}{w}.
\end{align*}
Then it follows that 
\begin{align*}
\abs{q_\omega(c_j,p_t)-q_{\omega^{(j)}}(c_j,p_t)} \lesssim w^{\beta-1}\abs{p_t-p_j^0},
\end{align*}
and thus
\begin{align}\label{eq:B7-new}
\kl\big(\Ber(q_\omega(c_j,p_t)),\Ber(q_{\omega^{(j)}}(c_j,p_t))\big) \lesssim w^{2\beta-2}(p_t-p_j^0)^2.
\end{align}

We next decompose
\begin{align*}
(p_t-p_j^0)^2 \leq 2\big(p_t-p_\omega^\ast(c_j)\big)^2 + 2\big(p_\omega^\ast(c_j)-p_j^0\big)^2.
\end{align*}
In view of Assumption~\ref{assump:rho}-(2), we have that 
\begin{align*}
\big(p_t-p_\omega^\ast(c_j)\big)^2 \leq \frac{2}{\sigma_r}\Delta_t^\omega,
\end{align*}
and thus by Lemma~\ref{lem:main-lower-family},
\begin{align*}
\big(p_\omega^\ast(c_j)-p_j^0\big)^2\leq a_2^2 w^{2\beta-2}.
\end{align*}
Hence, substituting these bounds into \eqref{eq:B7-new} yields that 
\begin{align}\label{eq:B8-new}
\KL\big(\Ber(q_\omega(c_j,p_t)),\Ber(q_{\omega^{(j)}}(c_j,p_t))\big) \leq C w^{2\beta-2}\Delta_t^\omega + C w^{4\beta-4}.
\end{align}

\noindent
\emph{Case B: $c_t=c_k\neq c_j$ and $p_t-c_t\in J_j$.}
In this case, the $j$th bump can still affect the Bernoulli mean, but the price is necessarily far from the oracle price at context $c_k$. Indeed, it holds that 
\begin{align*}
p_0^\ast(c_k)-c_k=z_k, \quad \abs{z_k-z_j}\geq w \quad (k\neq j).
\end{align*}
Since $p_t-c_k\in J_j$, we have that 
\begin{align*}
\abs{(p_t-c_k)-z_k}\geq \abs{z_j-z_k}-\frac{w}{8}\geq \frac{7w}{8}.
\end{align*}
By resorting to Lemma~\ref{lem:lower-family-helper-geometry}, we can show that 
\begin{align*}
\abs{\big(p_\omega^\ast(c_k)-c_k\big)-z_k} = \abs{p_\omega^\ast(c_k)-p_0^\ast(c_k)} \leq C_{\mathrm{sh}}\kappa w^{\beta-1}.
\end{align*}
After choosing $\kappa$ sufficiently small so that $C_{\mathrm{sh}}\kappa\leq 1/8$, and using $\beta\geq2$ so that $w^{\beta-1}\leq w$ for all large $T$, we can obtain that 
\begin{align*}
\abs{p_t-p_\omega^\ast(c_k)}\geq \frac{7w}{8}-\frac{w}{8}=\frac{3w}{4}.
\end{align*}

Assumption~\ref{assump:rho}-(2) further entails that 
\begin{align}\label{eq:B9-new}
\Delta_t^\omega\geq \frac{\sigma_r}{2}\left(\frac{3w}{4}\right)^2 \geq c w^2.
\end{align}
On the other hand, it holds that 
\begin{align*}
\abs{q_\omega(c_k,p_t)-q_{\omega^{(j)}}(c_k,p_t)} \leq 2\kappa w^\beta\pnorm{\varphi}{\infty},
\end{align*}
so we have that 
\begin{align}
\label{eq:B10-new}
\KL\big(\Ber(q_\omega(c_k,p_t)),\Ber(q_{\omega^{(j)}}(c_k,p_t))\big) \lesssim w^{2\beta} \lesssim w^{2\beta-2}\Delta_t^\omega,
\end{align}
where the last step above has utilized \eqref{eq:B9-new}.

Combining \eqref{eq:B8-new} and \eqref{eq:B10-new}, we can deduce the pointwise bound
\begin{align*}
\KL\big(\Ber(q_\omega(c_t,p_t)),\Ber(q_{\omega^{(j)}}(c_t,p_t))\big) \leq C w^{2\beta-2}\Delta_t^\omega\bm 1\{c_t=c_j \text{or} p_t-c_t\in J_j\}  + C w^{4\beta-4}\bm 1\{c_t=c_j\}.
\end{align*}
Substituting this inequality into \eqref{eq:B6-new} yields that 
\begin{align*}
\KL(P_\omega,P_{\omega^{(j)}}) \leq C w^{2\beta-2}R_j^\omega + C w^{4\beta-4}\E_\omega[N_j].
\end{align*}
Since $\E_\omega[N_j]=T/M\leq 128Tw$ for all sufficiently large $T$, we can obtain that
\begin{align*}
\KL(P_\omega,P_{\omega^{(j)}}) \leq C_{\mathrm{kl}} w^{2\beta-2}R_j^\omega + C_{\mathrm{kl}} T w^{4\beta-3}.
\end{align*}
The proof of the symmetric bound for $\KL(P_{\omega^{(j)}},P_\omega)$ is identical. This completes the proof of Lemma~\ref{lem:main-lower-kl}.

\subsection{Proof of Lemma~\ref{lem:main-lower-local}}

We start with establishing the decoder bound used in the two-point argument.

\paragraph{Step 1: a decoder for the local bit.}
Let us fix $j\in[M]$ and define
$N_j:=\sum_{t \in [T]} \bm 1\{c_t=c_j\}.$
Since the contexts are i.i.d. and uniform on $\{c_1,\dots,c_M\}$, it holds that 
\begin{align*}
N_j\sim \mathrm{Bin}\Big(T,\frac{1}{M}\Big), \quad \E[N_j]=\frac{T}{M}.
\end{align*}
In light of $M=\floor{1/(64w)}$, for all sufficiently large $T$ we have that $M\leq 1/(32w)$ and thus 
\begin{align*}
\E[N_j]=\frac{T}{M}\geq 32Tw.
\end{align*}
A standard Chernoff bound therefore implies that there exist some constants $c_0,c_1>0$ such that
\begin{align}\label{eq:B5-new}
\Prob(N_j\leq c_0Tw)\leq e^{-c_1Tw}.
\end{align}

We now define the \textit{decoder}. Consider only rounds with local coordinate $c_t=c_j$. Among those rounds, count how often the posted price lies to the right of the baseline threshold $p_j^0$. We set
\begin{align*}
\hat\omega_j:= \begin{cases}
+1 & \text{ if at least half of these prices satisfy that } p_t\geq p_j^0, \\
-1 & \text{ otherwise.}
\end{cases}
\end{align*}
When $N_j=0$, the ``at least half'' convention decodes $+1$; this tie convention is irrelevant on the event $N_j\geq c_0Tw$. This decoder uses \textit{only} the observed transcript and the common centering constant $\mu_0$, which is the same for all environments in the hard family.

First assume that the true sign is $\omega_j=+1$. If the decoder errs, among the $N_j$ rounds with $c_t=c_j$, at least $N_j/2$ rounds must satisfy that $p_t<p_j^0$. An application of Lemma~\ref{lem:main-lower-family} gives that 
$p_\omega^\ast(c_j)\geq p_j^0+a_1w^{\beta-1},$
so on each such round, we have that 
\begin{align*}
\abs{p_t-p_\omega^\ast(c_j)}\geq a_1 w^{\beta-1}.
\end{align*}
Assumption~\ref{assump:rho}-(2), which holds on the hard family in view of Lemma~\ref{lem:main-lower-family}, further yields the pointwise regret lower bound
\begin{align*}
\Delta_t^\omega \geq \frac{\sigma_r}{2}a_1^2 w^{2\beta-2} \quad \text{whenever } c_t=c_j, \,  p_t<p_j^0.
\end{align*}
Consequently, on event
\begin{align*}
\Big\{\hat\omega_j\neq \omega_j\Big\}\cap\Big\{N_j\geq c_0Tw\Big\},
\end{align*}
we have that 
\begin{align*}
\sum_{t \in [T]} \Delta_t^\omega\bm 1\{c_t=c_j\} \geq \frac{N_j}{2}\cdot \frac{\sigma_r}{2}a_1^2 w^{2\beta-2} \geq c T w^{2\beta-1}
\end{align*}
for some constant $c>0$.

The same argument is applicable when the true sign is $\omega_j=-1$. Hence, regardless of the sign of $\omega_j$, it holds that 
\begin{align*}
\sum_{t \in [T]} \Delta_t^\omega\bm 1\{c_t=c_j\} \geq c T w^{2\beta-1} \bm 1\Big\{\hat\omega_j\neq \omega_j, N_j\geq c_0Tw\Big\}.
\end{align*}
Taking expectations and using the definition of $R_j^\omega$, we can deduce that 
\begin{align*}
    R_j^\omega \geq \E_\omega\bigg[\sum_{t \in [T]} \Delta_t^\omega\bm 1\{c_t=c_j\}\bigg] \geq c T w^{2\beta-1}\Prob_\omega\Big(\hat\omega_j\neq \omega_j, N_j\geq c_0Tw\Big).
\end{align*}
Rearranging and using \eqref{eq:B5-new} yield that 
\begin{align*}
    \Prob_\omega(\hat\omega_j\neq \omega_j) \leq \Prob_\omega(N_j<c_0Tw)+\Prob_\omega\Big(\hat\omega_j\neq \omega_j, N_j\geq c_0Tw\Big) \leq e^{-c_1Tw}+C_{\mathrm{dec}}\frac{R_j^\omega}{Tw^{2\beta-1}}
\end{align*}
for a suitable constant $C_{\mathrm{dec}}>0$. The same argument with $\omega$ replaced by $\omega^{(j)}$ establishes the second claim.

\paragraph{Step 2: the Bretagnolle--Huber  inequality and the local lower bound.}
Let us fix $\omega$ and $j$. For brevity, denote by $P:=P_\omega$ and $Q:=P_{\omega^{(j)}}$. Let $E:=\{\hat\omega_j\neq \omega_j\}$. Then it holds that 
$P(E)=\Prob_\omega(\hat\omega_j\neq \omega_j)=:E_1.$
Under $Q$, the true bit is $\omega_j^{(j)}=-\omega_j$, so the complement event $E^c=\{\hat\omega_j=\omega_j\}$ is exactly the event that the decoder makes an error under $Q$. Consequently, we have that 
$Q(E^c)=\Prob_{\omega^{(j)}}(\hat\omega_j\neq \omega_j^{(j)})=:E_2.$

The decoder bound established in Step~1 above yields that 
\begin{align*}
E_1 \leq e^{-c_1Tw} + C_{\mathrm{dec}}\frac{R_j^\omega}{Tw^{2\beta-1}}, \quad E_2 \leq e^{-c_1Tw}+ C_{\mathrm{dec}}\frac{R_j^{\omega^{(j)}}}{Tw^{2\beta-1}}.
\end{align*}
Summing them gives that 
\begin{align}\label{eq:B11-new}
E_1+E_2 \leq 2e^{-c_1Tw} + C_{\mathrm{dec}}\frac{R_j^\omega+R_j^{\omega^{(j)}}}{Tw^{2\beta-1}}.
\end{align}

The Bretagnolle--Huber inequality (see, e.g., \citet[Theorem 14.2]{lattimore2020bandit}) states that for any event $A$,
\begin{align*}
P(A)+Q(A^c)\geq \frac{1}{2} e^{-\KL(P,Q)}.
\end{align*}
Applying this inequality with $A=E$ leads to 
\begin{align*}
E_1+E_2\geq \frac{1}{2} e^{-\KL(P,Q)}.
\end{align*}
An application of the same inequality with the roles of $P,Q$ reversed and with event $A=E^c$ gives that 
\begin{align*}
E_1+E_2\geq \frac{1}{2} e^{-\KL(Q,P)}.
\end{align*}
Taking the geometric mean of the two lower bounds above, we can deduce that 
\begin{align}\label{eq:B12-new}
E_1+E_2 \geq \frac{1}{2} \exp\Big(-\frac{\KL(P,Q)+\KL(Q,P)}{2}\Big).
\end{align}

Assume, for contradiction, that
$R_j^\omega+R_j^{\omega^{(j)}}\leq \alpha T w^{2\beta-1}$
for some constant $\alpha>0$. Then an application of Lemma~\ref{lem:main-lower-kl} shows that 
\begin{align}
    \label{eq:B13-new} \KL(P,Q)+\KL(Q,P) \leq C_{\mathrm{kl}} w^{2\beta-2}\big(R_j^\omega+R_j^{\omega^{(j)}}\big)+2C_{\mathrm{kl}}T w^{4\beta-3} \leq C_{\mathrm{kl}}(\alpha+2)T w^{4\beta-3} = C_{\mathrm{kl}}(\alpha+2)\gamma^{4\beta-3},
\end{align}
where the last identity above has used $w=\gamma T^{-1/(4\beta-3)}$. We now choose 
\begin{align*}
\alpha_0:=\frac{1}{16C_{\mathrm{dec}}}
\end{align*}
and $\gamma>0$ sufficiently small so that
\begin{align*}
C_{\mathrm{kl}}(\alpha_0+2)\gamma^{4\beta-3}\leq \log 4.
\end{align*}
Under the contradictory assumption with $\alpha=\alpha_0$, inequality \eqref{eq:B13-new} and the lower bound \eqref{eq:B12-new} entail that 
\begin{align}\label{eq:B14-new}
E_1+E_2\geq \frac{1}{2} e^{-\log 4/2}=\frac{1}{4}.
\end{align}

On the other hand, it follows from \eqref{eq:B11-new} that
\begin{align*}
E_1+E_2 \leq 2e^{-c_1Tw}+C_{\mathrm{dec}}\alpha_0 = 2e^{-c_1Tw}+\frac{1}{16}.
\end{align*}
Since 
$Tw=\gamma T^{\frac{4\beta-4}{4\beta-3}}\to\infty,$
for all sufficiently large $T$ we have that 
\begin{align*}
2e^{-c_1Tw}\leq 1/16.
\end{align*}
Consequently, we can deduce that 
\begin{align*}
E_1+E_2\leq \frac{1}{8},
\end{align*}
which contradicts \eqref{eq:B14-new}. Therefore, the contradictory assumption made above is in fact impossible, and thus we can obtain that
\begin{align*}
R_j^\omega+R_j^{\omega^{(j)}}>\alpha_0 T w^{2\beta-1}, \quad \forall \omega, \, \forall j.
\end{align*}
This establishes the desired conclusion with $c_\star:=\alpha_0$, which concludes the proof of Lemma~\ref{lem:main-lower-local}.
\end{APPENDICES}

\end{document}